\documentclass[11pt]{article}
\usepackage[utf8]{inputenc} 
\usepackage[T1]{fontenc}    

\usepackage{geometry}
\usepackage{setspace}
\usepackage{amsmath, amssymb, amsfonts, bm, mathtools, mathrsfs}
\usepackage{amsthm}
\usepackage[dvipsnames]{xcolor}
\definecolor{darkblue}{rgb}{0,0,.5}
\usepackage{graphicx}
\usepackage{subfigure}
\usepackage[numbers]{natbib}
\usepackage[colorlinks=true,allcolors=darkblue]{hyperref}       
\usepackage{url}            
\usepackage{booktabs}       
\usepackage{amsfonts}       
\usepackage{nicefrac}       
\usepackage{microtype}      

\allowdisplaybreaks

\usepackage{float}
\usepackage{multirow}
\usepackage{footnote}
\usepackage{dsfont}
\usepackage{mathabx}

\usepackage{algorithm}
\usepackage{algorithmic}
\usepackage{nicefrac}
\usepackage{tikz}
\usepackage{overpic}

\usepackage{dsfont}
\usepackage{hyperref}
\usepackage[capitalize]{cleveref}
\usepackage{crossreftools}

\makeatletter
\newcommand*{\rom}[1]{\expandafter\@slowromancap\romannumeral #1@}
\makeatother

\newcommand{\ind}{\mathds{1}}
\newcommand{\var}{\mathsf{Var}}

\newcommand{\brc}[1]{\left\{{#1}\right\}}
\newcommand{\prn}[1]{\left({#1}\right)} 
\newcommand{\brk}[1]{\left[{#1}\right]} 
\newcommand{\norm}[1]{\left\|{#1}\right\|} 
\newcommand{\abs}[1]{\left|{#1}\right|} 
\newcommand{\what}[1]{\widehat{#1}}
\newcommand{\wtilde}[1]{\widetilde{#1}}

\newcommand{\rd}{\mathrm{d}}

\newcommand{\<}{\langle} 
\renewcommand{\>}{\rangle}

\newcommand{\simiid}{\stackrel{\textup{iid}}{\sim}}

\newcommand{\cd}{\stackrel{d}{\rightarrow}}

\def\gA{{\mathcal{A}}}

\def\gE{{\mathcal{E}}}
\def\gF{{\mathcal{F}}}

\def\gH{{\mathcal{H}}}
\def\gI{{\mathcal{I}}}

\def\gK{{{\mathcal{K}}}}
\def\gL{{\mathcal{L}}}
\def\gM{{\mathcal{M}}}
\def\gN{{\mathcal{N}}}

\def\gP{{\mathcal{P}}}
\def\gQ{{\mathcal{Q}}}
\def\gR{{\mathcal{R}}}
\def\gS{{\mathcal{S}}}
\def\gT{{\mathcal{T}}}

\def\gW{{\mathcal{W}}}

\def\sP{{\mathscr{P}}}

\def\bR{{\bm{R}}}

\def\bD{{\bm{D}}}
\def\bh{{\bm{h}}}
\def\bH{{\bm{H}}}
\def\bK{{\bm{K}}}
\def\bw{{\bm{w}}}

\def\bW{{\bm{W}}}
\def\bx{{\bm{x}}}
\def\bI{{\bm{I}}}

\def\bG{{\bm{G}}}

\def\bPi{{\bm{\Pi}}}

\def\bSigma{{\bm{\Sigma}}}
\def\btheta{{\bm{\theta}}}

\def\bu{{\bm{u}}}
\def\bv{{\bm{v}}}

\def\btheta{{\bm{\theta}}}

\def\bZ{{\bm{Z}}}

\def\bD{{\bm{D}}}
\def\bz{{\bm{z}}}
\def\bT{{\bm{T}}}

\def\diag{{\operatorname{diag}}}

\def\RB{{\mathbb R}}
\def\EB{{\mathbb E}}

\def\PB{{\mathbb P}}

\def\NB{{\mathbb N}}
\def\GB{{\mathbb G}}

\def\ie{{\em i.e.\/}}

\newcommand{\KS}{\textup{KS}}

\makeatletter
\long\def\@makecaption#1#2{
  \vskip 0.8ex
  \setbox\@tempboxa\hbox{\small {\bf #1:} #2}
  \parindent 1.5em  
  \dimen0=\hsize
  \advance\dimen0 by -3em
  \ifdim \wd\@tempboxa >\dimen0
  \hbox to \hsize{
    \parindent 0em
    \hfil 
    \parbox{\dimen0}{\def\baselinestretch{0.96}\small
      {\bf #1.} #2
    } 
    \hfil}
  \else \hbox to \hsize{\hfil \box\@tempboxa \hfil}
  \fi
}
\makeatother

\newcommand{\cov}{\mathsf{Cov}}

\usepackage{enumerate}

\newtheorem{claim}{Claim}[section]
\newtheorem{lemma}[claim]{Lemma}
\newtheorem{assumption}{Assumption}

\newtheorem{theorem}{Theorem}[section]

\newtheorem{proposition}{Proposition}[section]
\newtheorem{remark}{Remark}[section]
\newtheorem{corollary}{Corollary}[section]

\usepackage{xr}

\title{Statistical Efficiency and Inference of Quantile Distributional Reinforcement Learning}

\author{
Zijie Cheng\thanks{School of Mathematical Sciences, Peking University; email: \texttt{xmwsbsj@stu.pku.edu.cn}.} \and
Yang Peng\thanks{Yau Mathematical Sciences Center, Tsinghua University; email: \texttt{yang-peng@mail.tsinghua.edu.cn}.} \and
Zhihua Zhang\thanks{School of Mathematical Sciences, Peking University; email: \texttt{zhzhang@math.pku.edu.cn}.}
}

\begin{document}
\maketitle
\begin{abstract}
In this paper, we study quantile-based distributional reinforcement learning from the perspective of statistical efficiency. 
We focus on distributional policy evaluation, whose goal is to characterize the return distribution, namely the distribution of discounted cumulative rewards under a given policy. 
To obtain a finite-dimensional representation of the return distribution, we consider the quantile fixed point $\bm\eta_m$ induced by the quantile-projected distributional Bellman equation. 
Assuming access to a generative model, we construct a certainty-equivalence estimator $\bm\eta_m^{(n)}$ based on an empirical Markov decision process.
For a fixed number of quantiles $m$, we 
establish a non-asymptotic error bound for $\bm\eta_m^{(n)}$ and $\bm\eta_m$ under the supremum $W_\infty$ metric, showing that the estimation error scales as $\wtilde{O}(\sqrt{m/n})$ with respect to the number of quantiles $m$ and the sample size $n$. 
This implies that the quantile-based distributional policy evaluation problem can be solved with sample efficiency, achieving the optimal parametric $\sqrt{n}$ convergence rate.
We further derive the asymptotic distribution of the standardized quantile parameters $\sqrt{n}(\bm\theta_m^{(n)}-\bm\theta_m)$ and characterize the semiparametric efficiency bound, which is attained by our estimator. 
Beyond the fixed-dimensional setting, we investigate the asymptotic regime in which the number of quantiles diverges. 
We characterize the limit covariance structure and show that it matches the semiparametric efficiency bound of the underlying nonparametric model for distributional policy evaluation, showing that quantile-based estimators remain asymptotically efficient in the infinite-dimensional limit. Finally, we establish a Berry--Esseen theorem for smooth functionals $\sqrt{n}(\eta_m^{(n)}(s)-\eta_m(s))f$, thereby providing a foundation for statistically valid inference on a broad class of functionals of the quantile-projected return distribution.

\end{abstract}
\tableofcontents
\section{Introduction}\label{Section:intro}
Reinforcement learning has achieved substantial progress across a wide range of domains, including game-playing \citep{silver2018general,vinyals2019grandmaster}, robotics \citep{kober2013reinforcement}, and large language models \citep{ouyang2022training, openai2023gpt4}. In the classical formulation of reinforcement learning based on the reward hypothesis \citep{sutton2004,sutton2018reinforcement}, policy performance is evaluated through expected returns (\ie, the expected discounted cumulative reward). 
While this criterion provides a principled objective for sequential decision-making, it ignores the distributional characteristics of returns, such as variability, tail risk, and higher-order distributional information. 
Such considerations are important in many real-world applications. 
In healthcare, for instance, one is concerned not only with average outcomes but also with adverse or long-tail effects of dynamic treatment regimes, as poor outcomes may have severe consequences for patients \citep{lavori2004dynamic}. 
In financial decision-making, investors must balance return and risk, since higher expected returns are typically associated with greater uncertainty \citep{ghysels2005there}.

Distributional reinforcement learning \citep{morimura2010nonparametric,bellemare2017distributional} generalizes the classical expected-return framework by characterizing the full distribution of return rather than only its expectation. 
This distributional perspective provides a richer description of uncertainty in the performance of learning agents, capturing variability arising from both the stochastic nature of environments and the randomness induced by the agent's policy. 
However, return distributions are generally infinite-dimensional objects, and hence cannot be represented or learned exactly. This motivates finite-dimensional approximations of return distributions, among which categorical and quantile-based representations are two prominent choices. In particular, a widely-used family of algorithms for distributional reinforcement learning is based on the notion of quantile representations, an approach that originated with \citet{dabney2018distributional}.
By approximating the return distribution through a finite collection of quantiles, quantile representations provide a flexible and computationally tractable approximation while preserving key distributional information. As a result, quantile-based approaches have been widely adopted in practice and have become a standard paradigm in modern distributional reinforcement learning, underpinning methods such as quantile regression deep Q-network (QR-DQN) \cite{dabney2018distributional}, implicit quantile network (IQN) \cite{dabney2018implicit}, and non-crossing quantile regression method \cite{zhou2020noncrossing}. 
This paradigm has also been successfully applied in a range of settings, including high-altitude balloon guidance \citep{bellemare2020autonomous}, robotic manipulation \citep{bodnar2019quantile}, and benchmark simulated domains such as the Arcade Learning Environment \citep{dabney2018distributional,dabney2018implicit}.

Despite the empirical success of quantile-based distributional reinforcement learning, its statistical foundations remain relatively underdeveloped, with limited understanding of the sampling variability and asymptotic behavior of learned return quantiles.
In practical reinforcement learning problems, the return distribution is determined by unknown transition dynamicsand reward distributions, and therefore should be inferred from finite data. 
This naturally raises a fundamental statistical question: can quantile-based representations support statistically efficient estimation and inference for return distributions? In particular, while quantile approximations provide a tractable finite-dimensional parametrization of distributional reinforcement learning, it remains unclear whether the resulting estimators admit finite-sample guarantees, asymptotically valid inferential procedures, and, more fundamentally, whether the statistical efficiency of the original infinite-dimensional problem is preserved under finite-dimensional parametrization. In this paper, we address these questions through a systematic study of estimation and inference for quantile distributional reinforcement learning.

From a theoretical perspective, the quantile setting is fundamentally different from the categorical setting. 
In categorical distributional reinforcement learning, the support locations are fixed and the parametrization is described by a finite-dimensional probability vector. 
As a result, the corresponding categorical-projected Bellman update acts linearly on the probability vectors, which makes it amenable to tools developed for linear stochastic approximation. 
In contrast, quantile-based methods have a rather different structure. They parameterize return distributions through inverse distribution functions, incurring an intrinsically nonlinear dependence between the parameters and the underlying distribution. 
Moreover, the quantile Bellman equation involves non-smooth and asymmetric operations arising from quantile projection and quantile loss. 
Perturbations of the quantile parameters can affect the induced distribution through indicator-type terms, and errors across different quantile levels may interact nonlinearly.
Consequently, classical techniques developed for linear projected operators, as well as many existing results for smooth nonlinear stochastic approximation, do not directly apply. 
These structural differences give rise to a new class of statistical and analytical challenges. 
To address them, we develop a unified theoretical framework that characterizes both the finite-sample and asymptotic behavior of quantile-based distributional reinforcement learning estimators.

\subsection{Our Contributions}

In this paper, we focus on the problem of distributional policy evaluation with quantile parametrization, where the goal is to estimate the return distribution of a given policy using a finite-dimensional quantile representation. 
We consider a $\gamma$-discounted infinite-horizon Markov decision process (MDP) in the tabular setting, where the state space $\gS$ and the action space $\gA$ are finite.
Given a policy $\pi$ and quantile level $m$, let $\bm\eta_m$ denote the unique fixed point of the projected distributional Bellman equation $\bm\eta=\bPi_m\gT^\pi \bm\eta$. Here, $\bPi_m$ represents the quantile projection operator and $\gT^\pi$ represents the distributional Bellman operator.

Our goal is to estimate $\bm\eta_m$ when the underlying MDP is unknown. 
We assume that both the distribution of the random reward and the transition probability of the MDP are unknown. 
Given an offline dataset consisting of $n|\gS||\gA|$ samples collected from a generative model, we construct an empirical MDP $\what{\gM}^{(n)}$ with empirical reward distribution $\what{\gP}_R^{(n)}$ and empirical transition kernel $\what P^{(n)}$. Following the certainty-equivalence principle~\citep{simon1956dynamic,theil1957note}, we estimate $\bm\eta_m$ by the fixed point $\bm\eta_m^{(n)}$ of the empirical projected distributional Bellman operator $\bPi_m \what{\gT}_n^\pi$, where $\what{\gT}_n^\pi$ is the distributional Bellman operator associated with $\what{\gM}^{(n)}$. We refer to $\bm\eta_m^{(n)}$ as the quantile fixed point estimator.

In this paper, we analyze the statistical properties of the quantile fixed point estimator $\bm\eta_m^{(n)}$. 
Our contributions are summarized as follows, including non-asymptotic guarantees, asymptotic inference, and non-asymptotic Gaussian approximation results, all of which are new to the best of our knowledge.

We first establish non-asymptotic estimation guarantees for quantile distributional policy evaluation in Theorem~\ref{thm:non_asymp_improved}. Under mild regularity conditions, we prove that the estimation error $\sup_{s\in\gS}W_\infty(\eta_m^{(n)}(s),\eta_m(s))$ scales as $\wtilde{O}(\sqrt{m/n})$ with respect to $m$ and $n$. This implies a sample complexity of $\wtilde O(m\varepsilon^{-2})$ up to problem-dependent constants. In addition, we provide a non-asymptotic upper bound on the approximation error induced by quantile discretization in Theorem~\ref{thm:approx_error}, which explicitly characterizes the dependence on the number of quantiles $m$ through the modulus of continuity of the return distribution.

We next develop an asymptotic inference theory for quantile distributional policy evaluation. Denoting the quantile parameters of $\bm\eta_m^{(n)}$ and $\bm\eta_m$ by $\bm\theta_m^{(n)}$ and $\bm\theta_m$, respectively, we prove in Theorem~\ref{thm:CLT} that $\sqrt n(\bm\theta_m^{(n)}-\bm\theta_m)$ converges weakly to a Gaussian distribution and that the resulting asymptotic covariance matrix attains the semiparametric efficiency bound. As a consequence, for any smooth test function $f$ and state $s$, we establish the asymptotic normality of the functional $\sqrt n(\eta_m^{(n)}(s)-\eta_m(s))f$ with asymptotic variance $\sigma_{m,f,s}^2$, and construct asymptotically valid confidence intervals for these functionals. 

Beyond asymptotic normality, we investigate the behavior of the covariance structure as the quantization level increases. 
In Theorem~\ref{thm:variance_convergence}, we prove that $\sigma_{m,f,s}^2$ converges to a well-defined infinite-dimensional limit $\sigma_{f,s}^2$ as $m\to\infty$ and characterize this limit through a pair of operators associated with the quantile Bellman equation. 
More importantly, we show that $\sigma_{f,s}^2$ coincides exactly with the semiparametric efficiency bound for estimating the functional $\eta^\pi(s)f$ of the infinite-dimensional return distribution $\eta^\pi(s)$. 
Consequently, the sequence of efficient finite-dimensional quantile estimators preserves asymptotic efficiency in the infinite-dimensional limit. 
This establishes a rigorous connection between quantile-based approximation and statistically optimal inference, showing that quantile parametrization introduces no asymptotic efficiency loss while providing a tractable finite-dimensional representation.

Finally, we establish a Berry--Esseen bound for the scaled functional $\sigma_{m,f,s}^{-1}\sqrt{n}(\eta_m^{(n)}(s)-\eta_m(s))f$ in Theorem~\ref{thm:berry-esseen}, quantifying the rate of convergence to the Gaussian distribution. Specifically, under some regularity conditions, we prove that 
\[
\sup_{t\in\RB}\left\vert\PB\left[\frac{\sqrt{n}}{\sigma_{m,f,s}}\left(\eta_m^{(n)}(s)-\eta_m(s)\right)f\leq t\right]-\Phi(t)\right\vert=\wtilde{O}(m^{13/4}n^{-1/4}), 
\]
where $\Phi(\cdot)$ is the cumulative distribution function of the standard Gaussian distribution. The $n^{-1/4}$ rate reflects the intrinsic non-smoothness of the quantile projection operator, which induces indicator functions in the fixed-point equation. 

From a technical perspective, our analysis develops a unified statistical treatment of quantile-based distributional reinforcement learning. The main difficulty stems from the fact that the quantile projected Bellman equation defines an implicitly specified, non-smooth fixed point system. 
We overcome this challenge by developing a $Z$-estimation formulation of the quantile Bellman fixed point, which allows us to derive non-asymptotic error bounds and asymptotic properties for the resulting estimator. 
A second technical challenge comes from the diverging-dimensional nature of the quantile parametrization. To address this issue, we introduce an operator embedding framework based on averaging and interpolation maps between finite-dimensional quantile vectors and functions in a Hilbert space. This allows us to establish operator-level convergence as $m\to\infty$ and to connect the resulting limit with the semiparametric efficiency bound of the original distributional policy evaluation problem. 
Finally, we develop a nonlinear Berry--Esseen expansion for smooth functionals of the implicit fixed-point estimator. Since the estimator does not admit a Bahadur representation, classical quantile process techniques are not applicable. We instead construct a stochastic expansion around the fixed-point equation and control the resulting nonlinear remainder terms induced by the Bellman operator, yielding an explicit $\wtilde{O}(n^{-1/4})$ Gaussian approximation rate.

\subsection{Related Works}
\paragraph{Distributional Reinforcement Learning.}
Distributional reinforcement learning has achieved remarkable success in fields such as communications \cite{hua2019gan}, transportation systems \cite{naeem2020generative}, and algorithm discovery \cite{fawzi2022discovering}.
A variety of approaches have been proposed to represent return distributions, including categorical representations \cite{bellemare2017distributional}, quantile representations \cite{dabney2018distributional}, and generative-model-based representations \cite{freirich2019distributional,doan2018gan}. 
Among these approaches, quantile representations have become one of the most widely studied paradigms. Starting from the quantile temporal-difference algorithm \cite{dabney2018distributional}, subsequent works developed more flexible quantile parametrizations through implicit quantile networks \cite{dabney2018implicit} and fully parameterized quantile functions \cite{yang2019fully}, among many other variants. 

Despite the empirical success, theoretical understanding of distributional reinforcement learning remains limited. 
Early theoretical studies mainly focused on the convergence properties of distributional Bellman operators and distributional temporal-difference learning algorithms. 
For categorical representations, \citet{rowland2018analysis} provided the first asymptotic convergence analysis for categorical temporal-difference learning and established its consistency. 
Based on this line of work, \citet{peng2024statistical} derived non-asymptotic convergence rates and sample complexity guarantees for categorical distributional reinforcement learning, while \citet{peng2025finite} further characterized its finite-sample behavior under linear function approximation. More recently, \citet{wu2023distributional} proposed a likelihood-based framework for offline distributional reinforcement learning and established non-asymptotic estimation guarantees. 
For quantile-based methods, \citet{rowland2023analysis} extended the asymptotic convergence analysis to quantile temporal-difference learning and established corresponding consistency results. However, their convergence analysis is asymptotic and does not consider sample complexities as well as asymptotic distributions. Separately, \citet{rowland2023statistical} decompose the estimation error of quantile temporal-difference learning into fixed-point bias and finite-sample variance, showing that it can outperform classical temporal-difference learning for expected return estimation.


\paragraph{Statistical Analysis of Reinforcement Learning.}
Statistical analysis in the context of reinforcement learning has drawn growing interest in the community. 
Existing works have developed a rich set of tools for uncertainty quantification and statistical inference for value functions. 
\citet{thomas2015high} and \citet{jiang2016doubly} proposed high-confidence bounds for value functions in the setting of off-policy evaluation.
\citet{hao2021bootstrapping} devised a bootstrapping procedure to perform statistical inference in off-policy evaluation. 
\citet{yang2022toward} investigated the asymptotic behavior of distributionally robust value functions and constructed asymptotically valid confidence bounds. 
\citet{shi2022statistical} modeled the value function with the series/sieve methods and devised confidence intervals for value functions in both the settings of policy evaluation and policy learning.
\citet{zhu2023uncertainty} also constructed asymptotically tight confidence intervals for learned (optimal) value functions.
\citet{li2023statistical} and \citet{li2023online} considered online statistical inference for value functions in an online reinforcement learning setting.

Comparatively fewer works study statistical properties of return distributions and other distributional functionals. \citet{chandak2021universal} and \citet{huang2022off} proposed procedures for estimating cumulative distribution functions and constructing confidence bands. More recently, \citet{Qi03072025} investigated distributional off-policy evaluation and established finite-sample guarantees under offline data. \citet{zhang2025estimation} studied model-based estimation and statistical inference for return distributions under a generative model setting, establishing non-asymptotic convergence rates and asymptotic normality results. Compared to their work, we focus on quantile parametrizations of return distributions, where the nonlinear quantile projection operator introduces additional analytical challenges absent in nonparametric distribution estimation.

\paragraph{Berry--Esseen Bounds for Quantile Estimators.}
Berry--Esseen bounds for quantile estimators have been extensively studied in the statistics literature. Bahadur--Kiefer representation \cite{bahadur1966note,kiefer1967bahadur} provided a linear expansion of sample quantiles with a stochastic remainder of order $\wtilde{O}(n^{-3/4})$, which is fundamental in developing refined probabilistic approximations for quantile estimators. 
\cite{lahiri2009berry} obtained a $O(n^{-1/2})$ Berry--Esseen bound for sample quantiles under weakly dependent conditions, relying on the expression of sample quantiles in terms of the empirical distribution function. 
For quantile regression estimators, \citet{koenker2005quantile} developed a Bahadur representation. 
Their result was later sharpened to $O((\log n)^{3/2}n^{-1/2})$ Berry--Esseen bounds for regression quantile processes in \cite{Portnoy2012NearlyRA}, based on density-level approximation of the estimator whose distribution is generated by the empirical distribution. 
More recently, Bahadur-type representations have also been developed for algorithmic quantile estimators arising from stochastic optimization procedures. 
For example, \citet{chen2025smoothedsgdquantilesbahadur} established a Bahadur representation for stochastic gradient-based quantile estimation via a smoothed loss formulation. 
However, these results rely heavily on the special structure of classical quantile estimators, where the estimator admits an explicit representation in terms of the empirical distribution function. 
In contrast, the quantile parameters considered in this paper are defined implicitly through a projected Bellman fixed-point equation, and standard Bahadur-type representations are not available.


The remainder of this paper is organized as follows. 
In Section~\ref{Section:preliminary}, we introduce some basic concepts of distributional reinforcement learning.
In Section~\ref{Section:analysis}, we present our statistical analysis of quantile distributional reinforcement learning.
In Section~\ref{Section:proof_outline}, we provide an outlined proof of results in Section~\ref{Section:analysis}.
In Section~\ref{Section:numerical_experiments}, we verify our theoretical findings through numerical simulations. 
We conclude our work in Section~\ref{Section:conclusion}. 
Details of the proof are given in the appendices.
\section{Preliminaries}\label{Section:preliminary}
In this section, we introduce the necessary background for our work. We review the Markov decision process in Section~\ref{subsec:problem_setting}, followed by metrics on the space of measures in Section~\ref{subsec:measure_metric}. We introduce distributional Bellman operator and quantile projection operator in Section~\ref{subsec:dist_bellman_quantile_projection}. Finally, in Section~\ref{subsec:assumptions}, we present the assumptions under which our main theoretical results are established.

\subsection{Problem Setup}\label{subsec:problem_setting}

We consider a discounted Markov decision process (MDP) specified by 
the tuple $\gM=\<\gS,\gA,\gP_R,P,\gamma\>$, where $\gS$ and $\gA$ are finite state space and finite 
action space respectively, $\gP_R\colon\gS\times\gA\to\Delta([0,1])$ is the distribution of rewards, $P\colon\gS\times\gA\to\Delta(\gS)$ is the transition probability, and $\gamma\in(0,1)$ is the discount factor.
Here $\Delta(\cdot)$ denotes the set of probability distributions over some set.

For a fixed policy $\pi\colon\gS\to\Delta(\gA)$ and an initial state $S_0= s\in\gS$, a random trajectory $\{(S_t,A_t,R_t)\}_{t=0}^\infty$ can be sampled from the MDP using the following procedure: 
\begin{equation*}
    \begin{aligned}
        A_t\mid S_t&\sim\pi(\cdot\mid S_t),\\
        R_t\mid (S_t,A_t)&\sim \gP_R({\cdot}\mid S_t,A_t),\\
        {S_{t+1}}\mid{(S_t,A_t)}&\sim P({\cdot}\mid{S_t,A_t}).\\
    \end{aligned}
\end{equation*}

The return of such a trajectory starting from state $s$ is 
defined as the random variable
\begin{equation*}
    G^\pi(s)\coloneq \sum_{t=0}^\infty \gamma^t R_t,
\end{equation*}
which is bounded almost surely by $[0, (1-\gamma)^{-1}]$. The value function $V^\pi(s)$ is defined by the expected return $\EB[G^\pi(s)]$. 
We further denote by $\eta^\pi(s) \in \Delta([0, (1-\gamma)^{-1}])$ 
the distribution of $G^\pi(s)$, and denote $\bm\eta^\pi = 
(\eta^\pi(s))_{s \in \gS}$ for the collection of 
return distributions across all states.


\subsection{Metrics on the Space of Measures}\label{subsec:measure_metric}
Denote the space of all probability distributions on $\RB$ as $\sP$. For $\mu\in\sP$, the cumulative distribution function is defined as $F_\mu(x)=\mu(-\infty,x]$, and the quantile function is defined as $F^{-1}_{\mu}(\tau)=\inf\{x\colon F_\mu(x)\geq \tau\}$. For $1\leq p<\infty$ and $\mu,\nu\in\sP$, the $p$-Wasserstein metric between $\mu$ and $\nu$ is defined as 
\begin{equation*}
    W_p(\mu, \nu)=\left( \int_0^1 \left|F^{-1}_\mu(t)-F^{-1}_\nu(t)\right|^p\mathrm{d}t\right)^{1 / p}, 
\end{equation*}
and the $\infty$-Wasserstein metric is defined as
\begin{equation*}
    W_\infty(\mu, \nu)=\sup_{t\in[0,1]} \left|F^{-1}_\mu(t)-F^{-1}_\nu(t)\right|. 
\end{equation*}

Suppose $\mu$ and $\nu$ have cumulative distribution functions $F_\mu$ and $F_\nu$, respectively. 
In the case of $p=1$ we have
\begin{equation*}
    W_1(\mu, \nu)=\int_\RB |F_\mu(x)-F_\nu(x)| \mathrm{d}x.
\end{equation*}
The Kolmogorov--Smirnov metric (KS metric) is defined as
\begin{equation*}
    {\KS}(\mu,\nu)=\sup_{x\in\RB}\abs{F_\mu(x)-F_\nu(x)}.
\end{equation*}

Moreover, for any extended metric $d\colon\sP\times\sP\to[0,\infty]$, we can define its supremum extension $\Bar{d}\colon \sP^\gS\times\sP^\gS\to[0,\infty]$ as
\[
\Bar{d}(\bm\eta,\bm\eta^\prime)=\sup_{s\in\gS}d(\eta(s),\eta^\prime(s)), 
\]
which is an extended metric on $\sP^\gS$. 


\subsection{Distributional Bellman Operator and Quantile Projection Operator}\label{subsec:dist_bellman_quantile_projection}

A fundamental property of the value function is that it satisfies 
the Bellman equation. Letting $\bm V^\pi \coloneq (V^\pi(s))_{s\in\gS}$, we have for any $s\in\gS$, 
\begin{equation}\label{eq:Bellman_equation}
    \begin{aligned}
            V^\pi(s)&=\brk{T^\pi\bm V^\pi}(s)\\
    &\coloneq \EB_{A\sim\pi(\cdot\mid s), R\sim\gP(\cdot\mid s,A)}[R]+\EB_{A\sim\pi(\cdot\mid s),S^\prime\sim P(\cdot\mid s,A)} [V^\pi(S^\prime)]\\
    &=\sum_{a\in\gA}\pi(a\mid s)\int_0^1 r \gP_R(\mathrm{d}r\mid s,a)+\sum_{a\in\gA,s^\prime\in\gS} \pi(a\mid s)P(s^\prime\mid s,a)V^\pi(s^\prime).
    \end{aligned}
\end{equation}
The operator $T^\pi\colon \RB^{\gS}\to \RB^{\gS}$ is referred to as 
the Bellman operator, and Equation~\eqref{eq:Bellman_equation} 
characterizes $\bm V^\pi$ as its unique fixed point.

An analogous relationship holds for the return distributions $\bm\eta^\pi$, known as the distributional Bellman equation. That is, for each $s\in\gS$, 
\begin{equation*}\label{Equation_distributional_Bellman_equation}
\begin{aligned}
        \eta^\pi(s)&=\brk{\gT^\pi\bm\eta^\pi}(s)\\
    &\coloneq \EB_{A\sim\pi(\cdot\mid s), R\sim\gP_R(\cdot \mid s,A),S^\prime\sim P(\cdot\mid s,A)}\brk{\prn{b_{R,\gamma}}_\#\eta^\pi(S^\prime)}\\
    &=\sum_{a\in\gA,s^\prime\in\gS}\pi(a\mid s)P(s^\prime\mid s,a)\int_0^1 \prn{b_{r,\gamma}}_\#\eta^\pi(s^\prime)\gP_R(\mathrm{d}r\mid s,a).
\end{aligned}
\end{equation*}
Here $b_{r,\gamma}\colon \RB\to\RB$ denotes the affine map $b_{r,\gamma}(x)=r+\gamma x$, and $g_\#\mu$ is the pushforward of a measure $\mu$ under $g$, defined by $g_\#\mu(B) = \mu(g^{-1}(B))$ for all Borel sets $B$. 
The integral ${\int_0^1 \prn{b_{r,\gamma}}_\#\eta^\pi(s^\prime)\gP_R(\mathrm{d}r\mid s,a)}$ is defined in the sense that for any Borel set $B$,
\begin{equation*}
    \brk{\int_0^1 \prn{b_{r,\gamma}}_\#\eta^\pi(s^\prime)\gP_R(\mathrm{d}r\mid s,a)}(B)=\int_0^1 \brk{\prn{b_{r,\gamma}}_\#\eta^\pi(s^\prime)}(B)\gP_R(\mathrm{d}r\mid s,a). 
\end{equation*}
Recall that $\sP$ is the space of all probability measures on $\RB$, and $\gT^\pi\colon \sP\to \sP$ is referred to as the distributional Bellman operator, of which the fixed point is the return distribution $\bm\eta^\pi$.

Because the exact distribution $\bm\eta^\pi$ is infinite-dimensional and cannot be computed exactly, we approximate it using a quantile-parameterized distribution. The space of all quantile-parametrized probability distributions is defined as
\begin{equation*}\label{eq:quantile_param}
    \sP_{m} \coloneq  \left\{ \nu_{\bx}=\frac{1}{m}\sum_{i=1}^m \delta_{x_i} \mid  \bx=\prn{x_1, \ldots, x_m}^{\top}\in \RB^{m}, x_1\leq \ldots\leq x_m \right\}, 
\end{equation*}
which is a mixture of Dirac measures and $m\in\NB$. 
We define the quantile
projection operator $\Pi_{m}\colon\sP\to\sP_m$ as
\begin{equation*}\label{eq:quantile_projection}
   \Pi_{m}\nu=\frac{1}{m}\sum_{i=1}^m\delta_{F^{-1}_{\nu}(\tau_i)}, 
\end{equation*}
where $\tau_i=\frac{2i-1}{2m}$. We lift $\Pi_m$ to the product space $\sP^\gS$ by defining
$(\bPi_m{\bm\eta})(s) \coloneq  \Pi_m\eta(s)$ for any $\bm\eta=(\eta(s))_{s\in\gS}\in\sP^\gS$. 
The properties of $\gT^\pi$ and $\bPi_m$ are summarized in the following proposition. 
\begin{proposition}\emph{\cite[Proposition 4.15, Lemma 5.25]{bdr2022}}\label{prop:basic_properties}
The following statements hold:
\begin{itemize}
    \item $\gT^\pi$ is  $\gamma$-contractive under the $\bar{W}_p$ metric for every $p\in[1,\infty]$, namely for every $\bm\eta,\bm\eta^\prime\in\sP^\gS$, 
    \[
    \bar{W}_p(\gT^\pi\bm\eta,\gT^\pi\bm\eta^\prime)\leq\gamma\bar{W}_p(\bm\eta,\bm\eta^\prime); 
    \]
    \item $\bPi_m$ is non-expansive under the $\bar{W}_\infty$ metric, namely $\bar{W}_\infty(\bPi_m\bm\eta,\bPi_m\bm\eta^\prime)\leq \bar{W}_\infty(\bm\eta,\bm\eta^\prime)$ for every $\bm\eta,\bm\eta^\prime\in\sP^\gS$.  
\end{itemize}
\end{proposition}

It immediately follows from Proposition~\ref{prop:basic_properties} that the quantile projected Bellman operator $\bPi_{m}{\gT}^{\pi}$ 
is a $\gamma$-contraction in the Polish space $(\sP^\gS,\bar{W}_{\infty})$.
Hence, the quantile projected Bellman equation $\bm\eta=\bPi_{m}\gT^{\pi}\bm\eta$ admits a unique solution $\bm\eta_{m}$. 
We define the quantile parameter $\bm\theta_m\in\RB^{\gS\times [m]}$ by 
\begin{equation*}
    \eta_m(s)=\frac{1}{m}\sum_{i=1}^m\delta_{\theta_m(s,i)},
\end{equation*}
with $\theta_m(s,1)\leq \cdots \leq\theta_m(s,m)$ for every $s\in\gS$.

\subsection{Main Assumptions}\label{subsec:assumptions}
For every $(s,a)\in\gS\times\gA$, we make the following assumption about $\gP_R(\cdot\mid s,a)$. 
\begin{assumption}\label{assump:density}
    For any $s\in\gS$, $a\in\gA$, $\gP_R(\cdot\mid s,a)$ is supported on $[0,1]$ and has a Lebesgue density $p_{s,a}$. Moreover, $p_{s,a}$ is continuous on $[0,1]$ and there exists a positive constant $C_0$ such that $0< p_{s,a}(x)\leq C_0$ for any $x\in(0,1)$. 
\end{assumption}
The following proposition shows that this assumption leads to a well-behaved return density, and its proof is presented in Appendix~\ref{Appendix_prop_return_density}. 
\begin{proposition}\label{prop:return_density}
    For every $s\in\gS$, $\eta^\pi(s)$ has a continuous Lebesgue density $p_{\eta^\pi(s)}$ upper bounded by $C_0$ on $[0,(1-\gamma)^{-1}]$. Moreover, $p_{\eta^\pi(s)}(x)>0$ for any $x\in(0,(1-\gamma)^{-1})$ and $p_{\eta^\pi(s)}(0)=p_{\eta^\pi(s)}((1-\gamma)^{-1})=0$. 
\end{proposition}

Moreover, we make an assumption on the locations of $\bm\theta_m$. 
\begin{assumption}\label{Assumption_no_boundary_quantile}
For every $m$ and $(s,i),(s^\prime,j)\in\gS\times [m]$, we have
\[
\theta_m(s,i)-\gamma\theta_m(s^\prime,j)\notin\{0,1\}. 
\]
\end{assumption}
This is a technical condition that excludes boundary-degenerate configurations in which Bellman-shifted quantile locations coincide with the endpoints of the reward support. In particular, it guarantees that the projected Bellman equation is continuously differentiable around the true parameter, thereby ensuring the regularity conditions required by the $Z$-estimation theory used in subsequent analysis.

To carry out non-asymptotic analysis, we impose one of the following two alternative regularity conditions on the reward densities. The first condition assumes that the reward densities are uniformly bounded away from zero, a common assumption in the quantile estimation. 
\begin{assumption}\label{Assumption_reward_lower_bounded}
For any $s\in\gS$, $a\in\gA$, $p_{s,a}$ is continuous on $[0,1]$ and there exists a positive constant $c_0$ such that $p_{s,a}(x)\geq c_0$ for any $x\in[0,1]$. 
\end{assumption}
The second condition allows reward densities to vanish at the boundary of the support, at the expense of additional smoothness and shape constraints near the endpoints. 
\begin{assumption}\label{Assumption_reward_bounded_density}
For any $s\in\gS$, $a\in\gA$, $p_{s,a}$ is Lipschitz continuous on $[0,1]$ and $p_{s,a}(0)=p_{s,a}(1)=0$, and there exists $\kappa>0$ such that $p_{s,a}$ is increasing on $[0,\kappa]$ and decreasing on $[1-\kappa,1]$. 
\end{assumption}
The monotonicity conditions near $0$ and $1$ exclude highly oscillatory boundary behavior and ensure that the density approaches zero in a regular manner. Typical examples include Beta distribution with shape parameters $\alpha,\beta>1$.

Under Assumption~\ref{Assumption_reward_lower_bounded}, we set $\kappa=\frac12$. Under Assumption~\ref{Assumption_reward_bounded_density}, $\kappa$ is the constant appearing in the assumption. In either case, the parameter $\kappa$ will be used in the subsequent analysis.



Finally, we make the following assumption on the quantile function of the return distribution $F^{-1}_{\eta^\pi(s)}$. 
\begin{assumption}\label{Assumption_return_density}
    For any $s\in\gS$, it holds that
    \begin{itemize}
        \item 
        \[
        \int_0^1 \frac{F^{-1}_{\eta^\pi(s)}(t)}{t}\mathrm{d}t<+\infty,\quad \int_{0}^1\frac{\frac{1}{1-\gamma}-F^{-1}_{\eta^\pi(s)}(t)}{1-t}\mathrm{d}t<+\infty;
        \]
        \item $\omega_s\left(\varrho\right)\coloneq \sup_{|x-y|\leq\varrho}|F_{\eta^\pi(s)}^{-1}(x)-F_{\eta^\pi(s)}^{-1}(y)|=o\left(1/\log\frac{1}{\varrho}\right)$ as $\varrho\rightarrow0$.
    \end{itemize}
\end{assumption}
This assumption requires that $F_{\eta^\pi(s)}$ does not approach $0$ or  $1$ too rapidly near the boundary, and it is rather mild. 
For example, $F_{\eta^\pi(s)}(x)\sim x^c$ with $c>0$ or even $F_{\eta^\pi(s)}(x)\sim\exp(-x^{-b})$ with $0<b<1$ near $0$ satisfies this assumption. 


\section{Main Results}\label{Section:analysis}
In this section, we analyze quantile-projected distributional reinforcement learning from both non-asymptotic and asymptotic perspectives. 
We first introduce the quantile fixed point estimator $\bm\eta_m^{(n)}$ and establish non-asymptotic convergence rates for $\Bar{W}_\infty(\bm\eta_m^{(n)},\bm\eta_m)$. We also study the asymptotic behavior of $\sqrt n(\eta_m^{(n)}(s)-\eta_m(s))$ and the associated linear functional $\sqrt n(\eta_m^{(n)}(s)-\eta_m(s))f$, showing that quantile-projected distributional policy evaluation is sample-efficient when a generative model is available. 
We show that these quantities are asymptotically Gaussian and derive explicit expressions for their asymptotic variances. We further characterize the limiting behavior of the asymptotic variance as the quantization level tends to infinity through an operator-theoretic representation. Finally, we establish Berry--Esseen bounds for the Gaussian approximation, providing quantitative guarantees on the accuracy of the asymptotic Gaussian approximation.

\subsection{The Quantile Fixed Point Estimator}
In this paper, we study the distributional policy evaluation problem, where the goal is to estimate the return distribution $\bm\eta^\pi$ associated with a fixed policy $\pi$. 
We consider the setting where the underlying MDP is unknown and can only be accessed through a finite dataset. 
We assume the dataset is sampled from a generative model, which generates a sample of the next state $s^\prime$ following $P(\cdot\mid s,a)$ and a reward sample $r$ following $\gP_R(\cdot\mid s,a)$ for any given pair $(s,a)\in\gS\times\gA$.
For each pair $(s,a)\in\gS\times\gA$, we query the generative model $n$ independent times
and produces an array  
\[
\{X_1^{(s,a)}, \ldots, X_n^{(s,a)}\}\coloneq\{(r^{(s,a)}_1,s^{\prime(s,a)}_1), \ldots, (r^{(s,a)}_n,s^{\prime(s,a)}_n)\}\simiid \gP_R(\cdot\mid s,a)\otimes P(\cdot\mid s,a). 
\]

Given the dataset, we may obtain the empirical estimates of the transition probability and reward distribution as
\begin{equation*}
    \begin{aligned}
        \what{P}^{(n)}(s^\prime\mid s,a)&=\frac{1}{n}\sum_{i=1}^n \ind\brc{s_i^{\prime(s,a)}=s^\prime},\\
            \what{\gP}_R^{(n)}(\cdot\mid s,a)&=\frac{1}{n}\sum_{i=1}^n \delta_{r_i^{(s,a)}}(\cdot).
    \end{aligned}
\end{equation*}
Thus, $\what{P}^{(n)}$ and $\what{\gP}_R^{(n)}$ define an empirical MDP $\what{\gM}^{(n)}=\<\gS,\gA,\what{\gP}_R^{(n)},\what{P}^{(n)} ,\gamma\>$. 
Denote the empirical distributional Bellman operator of $\what{\gM}^{(n)}$ as $\what{\gT}^\pi_n$. Then the empirical quantile projected Bellman equation $\bm\eta=\bPi_{m}\what{\gT}^\pi_n\bm\eta$ admits a unique solution $\bm\eta^{(n)}_{m}$. We also define $\bm\theta^{(n)}_m\in\RB^{\gS\times [m]}$ by 
\begin{equation*}
    \eta^{(n)}_m(s)=\frac{1}{m}\sum_{i=1}^m\delta_{\theta^{(n)}_m(s,i)},
\end{equation*}
with $\theta^{(n)}_m(s,1)\leq\ldots\leq\theta^{(n)}_m(s,m)$ for every $s\in\gS$. We call $\bm\eta_m^{(n)}$ the quantile fixed point estimator throughout the paper. 

The quantile fixed point estimator can be computed by quantile dynamic programming (QDP), namely iterating $\bm\eta_{m,k+1}^{(n)}=\bPi_{m}\what{\gT}^\pi_n\bm\eta_{m,k}^{(n)}$ until convergence. 
Because we need to sort all $r_i^{(s,a)}+\gamma\theta_{m,k}^{(n)}(s^\prime,j)$ for the computation of quantiles, the computational cost of $\bPi_{m}\what{\gT}^\pi_n\bm\eta_{m,k}^{(n)}$ is $O(|\gS|^2|\gA|mn)$. Combining this with Proposition~\ref{prop:basic_properties}, we know that in order to achieve $\Bar{W}_\infty(\bm\eta_{m}^{(n)},\bm\eta_{m,k}^{(n)})\le\varepsilon$, the total computational cost is 
 $O(|\gS|^2|\gA|(1-\gamma)^{-1}mn\log\frac{1}{\varepsilon})$. 

\subsection{Non-asymptotic Analysis}
We first establish a non-asymptotic convergence rate for the quantile fixed point estimator $\bm\eta_m^{(n)}$.

\begin{theorem}\label{thm:non_asymp_improved}
    Suppose that either Assumption~\ref{Assumption_reward_lower_bounded} or~\ref{Assumption_reward_bounded_density} holds. Given any $\delta\in(0,1)$, with probability at least $1-\delta$, we have
    \begin{equation*}
        \Bar{W}_\infty(\bm\eta^{(n)}_m,\bm\eta_m)\leq C_1(\gM) \sqrt{\frac{m\log(6|\gS|m/\delta)}{n}}. 
    \end{equation*}
    Moreover, we have
    \begin{equation*}
        \EB\brk{\Bar{W}_\infty(\bm\eta^{(n)}_m,\bm\eta_m)}\leq C_1(\gM)\sqrt{\frac{m\log(6|\gS|m)}{n}},
    \end{equation*}
    where $C_1(\gM)$ is a constant independent of $m$ and $n$ and only depends on $\gM$. 
\end{theorem}
The proof of Theorem~\ref{thm:non_asymp_improved} is provided in Section~\ref{subsec:proof_outline_nonasymp}. 
\begin{remark}
    The proof yields the following characterization of the constant $C_1(\gM)$ in Theorem~\ref{thm:non_asymp_improved}. For sufficiently large $m$, we have
    \begin{equation*}
        C_1(\gM)\leq \frac{C}{1-\gamma}\max\brc{\frac{1}{1-\gamma},\brk{\inf_{\substack{\tau\in[\kappa_0/2,1-\kappa_0/2]\\ s\in\gS}}p_{\eta^\pi(s)}\prn{F^{-1}_{\eta^\pi(s)}(\tau)}}^{-1}}, 
    \end{equation*}
    where $C$ is a universal constant and 
    \begin{equation*}
        \kappa_0=\frac{1}{2}\min_{s\in\gS}\brc{\min\brc{F_{\eta^\pi(s)}\prn{\frac{\kappa}{2}},1-F_{\eta^\pi(s)}\prn{(1-\gamma)^{-1}-\frac{\kappa}{2}}}}.
    \end{equation*}
    In particular, $\kappa_0$ depends only on the underlying MDP $\gM$.
\end{remark}

Theorem~\ref{thm:non_asymp_improved} indicates that $n=\wtilde{O}(m\varepsilon^{-2})$ up to problem-dependent constants suffices to ensure 
$\Bar{W}_\infty(\bm\eta_m^{(n)},\bm\eta_m)\leq \varepsilon$ with high probability, which implies distributional policy evaluation with quantile parametrization is sample-efficient. 

The key idea in our proof is that we first analyze the concentration behaviors of the quantiles of $\what{\gT}^\pi_n\bm\eta_m$ around those of $\gT^\pi\bm\eta_m$. We relate the quantile estimation problem to the analysis of the corresponding empirical distribution functions and perform a variance-dependent analysis via Bernstein-type inequalities. This yields a refined concentration bound that explicitly captures the dependence of the estimation variability on the quantile level. Then we combine this concentration result with the contraction property of the projected Bellman operator and establish the theorem.

Moreover, the approximation error of $\bm\eta_m$ is characterized as follows, the proof of which is provided in Section~\ref{subsec:proof_outline_nonasymp}. Combined with Proposition~\ref{prop:return_density}, this theorem implies that the approximation error converges to zero as $m\to\infty$.  
\begin{theorem}\label{thm:approx_error}
    For every $m\in\NB$, $\Bar{W}_\infty(\bm\eta_m,\bm\eta^\pi)\leq(1-\gamma)^{-1}\sup_{s\in\gS}\omega_s((2m)^{-1})$. 
\end{theorem}

\subsection{Asymptotic Analysis}
We now investigate the asymptotic behavior of $\bm\eta_m^{(n)}$, or $\bm\theta_m^{(n)}$ equivalently. 
To characterize the asymptotic distribution, we introduce two matrices. 
The matrix $\bSigma_m$ captures the sampling variability, which is defined as
\begin{equation*}
    (\bSigma_m)_{(s,i),(s^\prime,j)}=\ind\{s=s^\prime\}\sum_{a\in\gA}\pi(a\mid s)^2\cov_{(R,S)\sim \mathcal{P}_R(\cdot\mid s,a)\otimes P(\cdot\mid s,a)}(y^{\bm\theta_m}_{s,i}(R,S),y^{\bm\theta_m}_{s^\prime,j}(R,S)),
\end{equation*}
where $\bm y^\btheta\colon[0,1]\times\gS\to\RB^{\gS\times m}$ is defined as 
\begin{equation*}
    y^\btheta_{s,i}(r,s^\prime)=\frac{1}{m}\sum_{k=1}^{m}\ind\{r+\gamma\theta(s^\prime,k)<\theta(s,i)\}.
\end{equation*}
We can also formulate as
\begin{equation*}
    \bSigma_m=\diag\left\{\sum_{a\in\gA}\pi(a\mid s)^2\var_{(R,S)\sim \mathcal{P}_R(\cdot\mid s,a)\otimes P(\cdot\mid s,a)}\left(\bm y^{\bm\theta_m}_s(R,S)\right)\right\}_{s\in\gS}.
\end{equation*}

The matrix $\bG_m$ characterizes the local sensitivity
of the projected Bellman fixed-point equation with respect
to perturbations of the quantile parameters. The entries of $\bG_m$ are given by
\begin{align*}
        (\bG_m)_{(s,i),(s^\prime,j)}=&-\frac{\gamma }{m}\sum_{a\in\gA}\pi(a\mid s)P(s^\prime\mid s,a)p_{s,a}(\theta_m(s,i)-\gamma\theta_m(s^\prime,j))\\
        &+\frac{\ind\{(s,i)=(s',j)\}}{m}\sum_{a\in\gA,\tilde{s}\in\gS}\sum_{k=1}^m \pi(a\mid s)P(\tilde{s}\mid s,a)p_{s,a}\left(\theta_m(s,i)-\gamma\theta_m(\tilde{s},k)\right).
\end{align*}
With the matrices $\bSigma_m$ and $\bG_m$, we are now ready to characterize the asymptotic distribution of the estimated quantile parameters $\btheta_m^{(n)}$ and the proof is deferred to Section~\ref{subsec:proof_outline_asymp}. 


\begin{theorem}\label{thm:CLT}
$\bG_m$ is invertible, and
\begin{equation*}
    \sqrt{n}(\btheta_m^{(n)}-\btheta_m)\cd\gN\left(\bm{0}_m,\bG_m^{-1}\bSigma_m\bG_m^{-\top}\right).
\end{equation*}
Moreover, the asymptotic covariance matrix coincides with the semiparametric efficiency bound for estimating $\btheta_m$ in the model class
\begin{equation*}
    \gQ=\{(Q_{s,a})_{(s,a)\in\gS\times\gA}\colon Q_{s,a}\ll \lambda\otimes\#_{\gS}\},
\end{equation*}
where $\lambda$ is the Lebesgue measure on $[0,1]$ and $\#_{\gS}$ is the counting measure on $\gS$. Therefore, $\btheta_m^{(n)}$ is semiparametrically efficient.
\end{theorem}


Theorem~\ref{thm:CLT} shows that, for any fixed $m$, the estimator $\btheta_m^{(n)}$ admits a Gaussian limit after normalization by $\sqrt n$. A key step in the proof is to reformulate the projected Bellman equation as a finite-dimensional estimating equation and then apply tools from $Z$-estimation theory.
Correspondingly, the asymptotic covariance takes the form $\bG_m^{-1}\bSigma_m\bG_m^{-\top}$, reflecting the interaction between the sampling variability captured by $\bSigma_m$ and the local sensitivity of the projected Bellman fixed-point equation encoded by $\bG_m$.

Beyond asymptotic normality, Theorem~\ref{thm:CLT} also establishes the semiparametric efficiency of $\btheta_m^{(n)}$. 
This implies that the asymptotic covariance matrix $\bG_m^{-1}\bSigma_m\bG_m^{-\top}$ coincides with the semiparametric efficiency bound, and therefore no regular estimator can achieve a uniformly smaller asymptotic covariance matrix. The proof is based on a pathwise differentiability analysis of the projected Bellman fixed-point mapping. By characterizing the response of the fixed point to smooth perturbations of the underlying data-generating distribution, we identify the efficient influence function and show that $\btheta_m^{(n)}$ attains the corresponding efficiency bound.



Integral functionals induced by smooth test functions offer a more natural way to compare return distributions. 
Applying the delta method, we have the following corollary for these functionals. 
\begin{corollary}\label{cor:CLT_test_functional}
    Suppose that $f\in C^1[0,(1-\gamma)^{-1}]$. 
    For $s\in\gS$, define $(\bm\varphi_{m,f,s})_{s^\prime,j}=\ind\{s^\prime=s\}\frac{f^\prime(\theta_m(s^\prime,j))}{m}$ and we have
    \begin{equation*}
        \sqrt{n}\left(\eta_m^{(n)}(s)-\eta_m(s)\right)f\cd\gN(0,\sigma_{m,f,s}^2)
    \end{equation*}
    where
    \begin{equation*}
        \sigma_{m,f,s}^2=\bm\varphi_{m,f,s}^\top \bG_m^{-1}\bSigma_m\bG_m^{-\top}\bm\varphi_{m,f,s}
    \end{equation*}
\end{corollary}
\begin{proof}[Proof of Corollary~\ref{cor:CLT_test_functional}]
    For test function $f$ and $s\in\gS$, define $F\colon\RB^{\gS\times[m]}\to\RB$ by $F(\btheta)=\frac{1}{m}\sum_{i=1}^mf(\theta(s,i))$. Then 
    $\eta_m^{(n)}(s)f=F(\btheta_m^{(n)})$, $\eta_m(s)f=F(\btheta_m)$ and $\nabla F=\bm\varphi_{m,f,s}$. The conclusion follows from Theorem~\ref{thm:CLT} and the multivariate delta method.
\end{proof}

The asymptotic normality established above further enables statistical inference for these functionals. By estimating the asymptotic variance with plug-in estimators, we obtain an asymptotically valid confidence interval for $\eta_m(s)f$.
\begin{corollary}\label{cor:inference}
    Suppose that $f\in C^1[0,(1-\gamma)^{-1}]$, and for each pair $(s,a)$, we have a density estimator $\what{p}_{s,a}$ satisfying
    \begin{equation}\label{eq:density_estimator_property}
        \sup_{x\in[0,1]}|\what{p}_{s,a}(x)-p_{s,a}(x)|\to0
    \end{equation}
    in probability as $n\to\infty$. We construct $\what{\bG}_{m}$, $\what{\bSigma}_m$ and $\what{\bm\varphi}_{m,f,s}$ by replacing $p_{s,a}$, $\gP_R(\cdot\mid s,a)$, $P(\cdot\mid s,a)$, and $\btheta_m$ in the definitions of $\bG_m$, $\bSigma_m$, and $\bm\varphi_{m,f,s}$ with $\what{p}_{s,a}$, $\what{\gP}_R^{(n)}(\cdot\mid s,a)$, $\what{P}^{(n)}(\cdot\mid s,a)$ and $\btheta_m^{(n)}$, respectively, and $\what{\sigma}_{m,f,s}^2=\what{\bm\varphi}_{m,f,s}^\top \what{\bG}_m^{-1}\what{\bSigma}_m\what{\bG}_m^{-\top}\what{\bm\varphi}_{m,f,s}$. Let the confidence interval
    \begin{equation*}
        \mathbf{CI}(\alpha)=[\eta_m^{(n)}(s)f-z_{\frac{\alpha}{2}}\frac{\what{\sigma}_{m,f,s}}{\sqrt{n}},\eta_m^{(n)}(s)f+z_{\frac{\alpha}{2}}\frac{\what{\sigma}_{m,f,s}}{\sqrt{n}}], 
    \end{equation*}
    where $z_{\frac{\alpha}{2}}$ is the upper $\frac{\alpha}{2}$-quantile of $\gN(0,1)$. Then we have 
    \begin{equation*}
        \lim_{n\to\infty}\PB\left[\eta_m(s)f\in\mathbf{CI}(\alpha)\right]=1-\alpha. 
    \end{equation*}
\end{corollary}
A canonical choice for estimating $p_{s,a}$ is the kernel density estimator. With a bandwidth sequence $h_n$ satisfying $h_n\to0$ and $nh_n\to\infty$, Equation~\eqref{eq:density_estimator_property} holds in probability~\citep{scott2015multivariate}. 
\begin{proof}[Proof of Corollary~\ref{cor:inference}]
    By the law of large numbers we know that $\what{\sigma}_{m,f,s}^2\to\sigma_{m,f,s}^2$ in probability as $n\to\infty$ and the conclusion follows from Slutsky's theorem. 
\end{proof}

\subsection{Limiting Covariance Structure and Efficiency}\label{subsec:limit_covariance}
The asymptotic variance in Corollary~\ref{cor:CLT_test_functional} is expressed through the finite-dimensional matrices $\bG_m$ and $\bSigma_m$, whose dimensions grow with the quantization level $m$. To understand the covariance structure underlying quantile-based distributional policy evaluation, it is natural to ask whether the finite-dimensional asymptotic variances admit a well-defined limit as $m\to\infty$ and, if so, how this limit relates to the efficiency bound of distributional policy evaluation problem. In this subsection, we answer this question by characterizing the limit of $\sigma_{m,f,s_0}^2$ as $m\rightarrow\infty$ through a pair of operators acting on an appropriate infinite-dimensional Hilbert space.

Define $\gL=\bigoplus_{s\in\gS}L^2[0,1]$, and $\|\bm\psi\|^2=\sum_{s\in\gS}\|\psi_s(\cdot)\|_{L^2}^2$ for $\bm\psi\in\gL$. 
Define the operator $\gK$ on $\gL$ by
\begin{equation*}
    (\gK\bm\psi)_s(\tau)=\frac{1}{p_{\eta^\pi(s)}\left(F^{-1}_{\eta^\pi(s)}(\tau)\right)}\sum_{a\in\gA,s'\in\gS}\pi(a\mid s)P(s^\prime\mid s,a)\int_0^1p_{s,a}\left(F^{-1}_{\eta^\pi(s)}(\tau)-\gamma F^{-1}_{\eta^\pi(s')}(t)\right)\psi_{s'}(t)\mathrm{d}t,
\end{equation*}
and we refer to $\gK$ as the quantile Bellman operator. The operator $\wtilde{\Sigma}$ on $\gL$ is given by
\begin{equation*}
    (\wtilde{\Sigma}\bm\psi)_s(\tau)=\int_0^1\frac{A_s(\tau,t)}{p_{\eta^\pi(s)}\left(F^{-1}_{\eta^\pi(s)}(\tau)\right)p_{\eta^\pi(s)}\left(F^{-1}_{\eta^\pi(s)}(t)\right)}\psi_s(t)\mathrm{d}t,
\end{equation*}
where $A\colon[0,1]\times[0,1]\to\RB^\gS$ is given by
\begin{equation*}
    A_s(\tau,t)=\sum_{a\in\gA}\pi(a\mid s)^2\cov_{(R,S)\sim \gP_R(\cdot\mid s,a)\otimes P(\cdot\mid s,a)}\left(F_{\eta^{\pi}(S)}\left(\frac{F^{-1}_{\eta^\pi(s)}(\tau)-R}{\gamma}\right),F_{\eta^\pi(S)}\left(\frac{F^{-1}_{\eta^\pi(s)}(t)-R}{\gamma}\right)\right). 
\end{equation*}

The following proposition shows that both $\gK$ and $\wtilde{\Sigma}$ are bounded operators from $\gL$ to $\gL$, whose proof is deferred to Appendix~\ref{Appendix_prop_K_invertible}. 
\begin{proposition}\label{prop:K_invertible}
    $\gK$ and $\wtilde{\Sigma}$ are bounded operators on $\gL$. Moreover, $\gI-\gamma \gK^*$ is invertible, where $\gK^*$ is the adjoint operator of $\gK$ in $\gL$. 
\end{proposition}

The limit of asymptotic variance is characterized as follows. 
The proof of this theorem is presented in Section~\ref{subsec:proof_outline_limit_cov}. 
\begin{theorem}\label{thm:variance_convergence}
Suppose that either Assumption~\ref{Assumption_reward_lower_bounded} or~\ref{Assumption_reward_bounded_density} holds, and that Assumption~\ref{Assumption_return_density} holds. 
For a test function $f\in C^1[0,(1-\gamma)^{-1}]$ and $s_0\in\gS$, define $\psi(s,\tau)=\ind\{s=s_0\}f'(F^{-1}_{\eta^\pi(s)}(\tau))$ and $\bu^*=(\gI-\gamma \gK^*)^{-1}\bm\psi$. Then we have
\begin{equation*}
    \lim_{m\to+\infty}\sigma_{m,f,s_0}^2=\langle \bu^*,\wtilde{\Sigma}\bu^*\rangle\coloneq\sigma_{f,s_0}^2. 
\end{equation*}
Moreover, denote $\bm\eta^{(n)}$ as the fixed point of $\what{\gT}^\pi_n$, namely $\bm\eta^{(n)}=\what{\gT}^\pi_n\bm\eta^{(n)}$. We have
\begin{equation*}
    \sqrt{n}\left(\eta^{(n)}(s_0)-\eta^\pi(s_0)\right)f\cd\gN(0,\sigma_{f,s_0}^2)
\end{equation*}
and $\sigma_{f,s_0}^2$ coincides with the semiparametric efficiency bound
for estimating $\eta^\pi(s_0)f$ in the model class $\gQ$. 
\end{theorem}

Theorem~\ref{thm:variance_convergence} provides a statistical interpretation of the limiting covariance structure. 
The quantity $\langle \bu^*,\wtilde{\Sigma}\bu^*\rangle$ first arises as the operator-theoretic limit of the asymptotic variances $\sigma_{m,f,s_0}^2$ in the finite-dimensional quantile approximation. 
Moreover, the theorem shows that this limit coincides exactly with the asymptotic variance of the nonparametric estimator $\eta^{(n)}(s_0)f$. 
Consequently, the sequence of finite-dimensional efficient quantile estimators preserves asymptotic efficiency as the quantization level $m\to\infty$. In particular, no statistical efficiency is lost when passing from the original infinite-dimensional return distribution to its finite-dimensional quantile approximation.

To gain intuition for the operator $\gK$, consider a probability measure vector $\bm\nu\in\sP^\gS$ and formally define 
\begin{equation*}
    Z_s(t)=\frac{\nu_s(-\infty,F^{-1}_{\eta^\pi(s)}(t)]}{p_{\eta^\pi(s)}\left(F^{-1}_{\eta^\pi(s)}(t)\right)}. 
\end{equation*}
Whenever the quantities involved are well defined, we have
\begin{equation*}
    (\gamma \gK \bZ)_s(t)=\frac{(\gT^\pi\bm\nu)_s(-\infty,F^{-1}_{\eta^\pi(s)}(t)]}{p_{\eta^\pi(s)}\left(F^{-1}_{\eta^\pi(s)}(t)\right)}. 
\end{equation*}
Therefore, the operator $\gK$ can be interpreted as the representation of the Bellman operator acting on distributions after transforming them from the space of probability measures to quantile coordinates.
The operator $\wtilde{\Sigma}$ admits a similar interpretation. 
$A_s$ can be viewed as the covariance kernel associated with the stochastic perturbations of the Bellman operator. 
It is the infinite-dimensional analogue of the covariance matrix $\bSigma_m$. The additional factors
\begin{equation*}
    \left[p_{\eta^\pi(s)}\left(F^{-1}_{\eta^\pi(s)}(\tau)\right)p_{\eta^\pi(s)}\left(F^{-1}_{\eta^\pi(s)}(t)\right)\right]^{-1}. 
\end{equation*}
arise from the quantile transformation. Consequently, $\wtilde{\Sigma}$ can be interpreted as the covariance operator induced by the stochastic perturbations of the Bellman operator when these perturbations are represented in the space of quantile functions.

At a technical level, the first challenge is that both $\bG_m$ and $\bSigma_m$ depend on $m$, and their dimensions diverge as $m\to\infty$. 
To overcome this difficulty, we view the discrete matrices as operators acting on the infinite-dimensional Hilbert space $\gL$ and establish operator-level convergence. 
A major obstacle is that the limiting operators contain factors of the form $p_{\eta^\pi(s)}(F^{-1}_{\eta^\pi(s)}(\tau))^{-1}$, which may become singular near the boundary quantile levels. 
We therefore develop a careful boundary analysis to control these singularities and justify the convergence of both the covariance and Jacobian structures. 

A second and more subtle challenge is to identify the statistical meaning of the limiting variance. The operator-theoretic limit naturally leads to the representation
\[
\left\langle (\gI-\gamma \gK^*)^{-1}\bm\psi,\wtilde{\Sigma}(\gI-\gamma \gK^*)^{-1}\bm\psi\right\rangle,
\]
which is expressed entirely in quantile coordinates. In contrast, the semiparametric efficiency bound established in \citet{zhang2025estimation} is formulated in the space of signed measures and involves the resolvent $(\gI-\gT^\pi)^{-1}$ acting on distributional perturbations. To bridge this gap, we establish an explicit correspondence between perturbations in the signed measure space and their representations in quantile coordinates. 
This correspondence shows that the operator $\gK$ is precisely the distributional Bellman operator expressed on the quantile scale.
Combined with an integration-by-parts argument, it allows us to transform the quantile-based limiting variance into the semiparametric efficiency bound.

\subsection{Berry--Esseen Bound}
In this subsection, we establish a Berry--Esseen bound for the functional $\sqrt{n}(\eta_m^{(n)}(s)-\eta_m(s))f$, which provides an explicit rate of convergence to the limiting distribution.
\begin{theorem}\label{thm:berry-esseen}
    Suppose that Assumptions~\ref{Assumption_reward_bounded_density} and~\ref{Assumption_return_density} hold, and $f\in C^{1,1}[0,(1-\gamma)^{-1}]$, meaning that $f$ has a Lipschitz continuous derivative on $[0,(1-\gamma)^{-1}]$. For every $m,n\in\NB$, we have
    \begin{equation}\label{eq:BE}
        \sup_{t\in\RB}\left\vert\PB\left[\frac{\sqrt{n}}{\sigma_{m,f,s}}\left(\eta_m^{(n)}(s)-\eta_m(s)\right)f\leq t\right]-\Phi(t)\right\vert\leq C(\gM,f)\brk{\frac{m^{13}\log^4 (em)\log^2 (en)}{n}}^{\frac{1}{4}},
    \end{equation}
    where $\Phi(\cdot)$ is the cumulative distribution function of the standard Gaussian distribution and $C(\gM,f)$ is a constant independent of $m$ and $n$.  
\end{theorem}

Theorem~\ref{thm:berry-esseen} shows that the Gaussian approximation error is bounded by $\widetilde O(m^{13/4} n^{-1/4})$ in Kolmogorov--Smirnov distance. This result is a quantitative refinement of the asymptotic normality established in Corollary~\ref{cor:CLT_test_functional}. 

The convergence rate should be contrasted with the classical quantile estimation literature. For sample quantiles, Berry--Esseen bounds of order $O(n^{-1/2})$ are available~\cite{lahiri2009berry}, while for quantile regression estimators nearly optimal $O((\log n)^{3/2}n^{-1/2})$ rates have also been established~\cite{Portnoy2012NearlyRA}. The slower $n^{-1/4}$ rate in our result can be attributed to two major difficulties. First, the projected Bellman equation involves indicator functions and is therefore inherently non-smooth. As a result, the higher-order stochastic expansions available for smooth $M$-estimators are not directly applicable in our setting. Second, unlike classical quantile estimators, the quantile parameters considered here are defined implicitly through the projected Bellman fixed-point equation and do not admit a direct representation through empirical distribution functions. Consequently, the empirical-distribution-based arguments underlying sharp Berry--Esseen bounds for classical quantile estimators are no longer available.

The key idea of the proof is to derive a stochastic expansion of the test functional and decompose it as $W+D$, where $W$ is a normalized sum of independent random variables and $D$ is a higher-order remainder term arising from the nonlinearity of the projected Bellman operator. We then apply a Berry--Esseen theorem for general nonlinear statistics to this decomposition. 
A crucial step is to establish quantitative bounds on the remainder term $D$ by two components: a local empirical-process fluctuation term and a second-order expansion remainder. We further control the effect of replacing a single observation by an independent copy on both components. Combining these bounds with the Berry--Esseen theorem for general nonlinear statistics yields the stated rate. Detailed proofs are provided in Section~\ref{subsec:proof_outline_BE}.
\section{Proof Outlines}\label{Section:proof_outline}
In this section, we present proofs of Theorems~\ref{thm:non_asymp_improved},~\ref{thm:approx_error},~\ref{thm:CLT},~\ref{thm:variance_convergence} and~\ref{thm:berry-esseen}. 

\subsection{Analysis of Non-asymptotic Bound}\label{subsec:proof_outline_nonasymp}
The proof of Theorem~\ref{thm:non_asymp_improved} consists of two steps. We first establish a concentration inequality for the empirical Bellman operator evaluated at the population fixed point $\bm\eta_m$. This yields a uniform bound on the perturbation of the corresponding quantile functions. We then combine this estimate with the $\gamma$-contraction property of $\bPi_m\gT^\pi$ under $\Bar{W}_\infty$ metric to obtain a fixed-point perturbation bound, which implies Theorem~\ref{thm:non_asymp_improved}. The key ingredient is the following concentration result, the proof of which is provided in Appendix~\ref{Appendix_omit_proof_nonasymp}. 
\begin{lemma}\label{lem:quantile_concentration}
    Suppose that either Assumption~\ref{Assumption_reward_lower_bounded} or~\ref{Assumption_reward_bounded_density} holds. Given $s\in\gS$ and $i\in[m]$, for every $\delta\in(0,1)$, we have
    \begin{equation*}
        \left|F^{-1}_{(\what{\gT}^\pi_n\bm\eta_m)(s)}(\tau_i)-F^{-1}_{(\gT^\pi\bm\eta_m)(s)}(\tau_i)\right|\leq C(\gM)\sqrt{\frac{m\log(6/\delta)}{n}}.
    \end{equation*}
    holds with probability at least $1-\delta$. Here $C(\gM)$ is a constant independent of $m$ and $n$. 
\end{lemma}

Applying Lemma~\ref{lem:quantile_concentration} and a union bound over all states and quantile levels, we obtain
\begin{equation*}
    \Bar W_\infty\left(\bPi_m\what{\gT}^\pi_n\bm\eta_m,\bPi_m\gT^\pi\bm\eta_m\right)\leq C(\gM)\sqrt{\frac{m\log(6|\gS|m/\delta)}{n}}
\end{equation*}
with probability at least $1-\delta$. Since $\bm\eta_m$ and $\bm\eta_m^{(n)}$ are fixed points of $\bPi_m\gT^\pi$ and $\bPi_m\what{\gT}^\pi_n$, respectively, we have
\begin{align*}
    \Bar{W}_\infty(\bm\eta_m^{(n)},\bm\eta_m)&\leq\Bar{W}_\infty(\bPi_m\what{\gT}^\pi_n\bm\eta_m^{(n)},\bPi_m\what{\gT}^\pi_n\bm\eta_m)+\Bar W_\infty(\bPi_m\what{\gT}^\pi_n\bm\eta_m,\bPi_m\gT^\pi\bm\eta_m)\\
    &\leq\gamma\Bar{W}_\infty(\bm\eta_m^{(n)},\bm\eta_m)+\Bar W_\infty(\bPi_m\what{\gT}^\pi_n\bm\eta_m,\bPi_m\gT^\pi\bm\eta_m). 
\end{align*}
Rearranging the above inequality yields Theorem~\ref{thm:non_asymp_improved}. The expectation upper bound is derived from Lemma~\ref{Lemma_expected_maximal_of_sub_gaussian}. 

To prove Theorem~\ref{thm:approx_error}, we use the contraction property of the projected Bellman operator to obtain
\begin{align*}
    \Bar{W}_\infty(\bm\eta_m,\bm\eta^\pi)\leq&\Bar{W}_\infty(\bPi_m\gT^\pi\bm\eta_m,\bPi_m\gT^\pi\bm\eta^\pi)+\Bar{W}_\infty(\bPi_m\bm\eta^\pi,\bm\eta^\pi)\\
    \leq&\gamma\Bar{W}_\infty(\bm\eta_m,\bm\eta^\pi)+\Bar{W}_\infty(\bPi_m\bm\eta^\pi,\bm\eta^\pi). 
\end{align*}
By the definition of $\bPi_m$, we know that
\begin{equation*}
    \Bar{W}_\infty(\bPi_m\bm\eta^\pi,\bm\eta^\pi)=\max_{i\in[m],s\in\gS}\brc{\max_{x\in[\frac{i-1}{m},\frac{i}{m}]}\left|F_{\eta^\pi(s)}^{-1}(x)-F_{\eta^\pi(s)}^{-1}(\tau_i)\right|}\leq\sup_{s\in\gS}\omega_s\left(\frac{1}{2m}\right). 
\end{equation*}
Therefore, 
\begin{equation*}
    \Bar{W}_\infty(\bm\eta_m,\bm\eta^\pi)\leq\frac{1}{1-\gamma}\Bar{W}_\infty(\bPi_m\bm\eta^\pi,\bm\eta^\pi)\leq\frac{1}{1-\gamma}\sup_{s\in\gS}\omega_s\left(\frac{1}{2m}\right). 
\end{equation*}

\subsection{Analysis of Asymptotic Distribution and Efficiency}\label{subsec:proof_outline_asymp}
First, we prove the $\sqrt{n}$-consistency of $\btheta_m^{(n)}$ under the weaker conditions. The proof of the lemma is deferred to Appendix~\ref{Appendix_lem_consistency_weaker_condition}. 
\begin{lemma}\label{lem:consistency_weaker_condition}
    For every fixed $m$, we have $\sqrt{n}(\btheta_m^{(n)}-\btheta_m)=O_P(1)$ as $n\to\infty$. 
\end{lemma}
With Lemma~\ref{lem:consistency_weaker_condition} in hand, we can reformulate the projected Bellman equation as a finite-dimensional $Z$-estimation problem and establish the asymptotic normality. For brevity, we denote $Q_{s,a}=\gP_R(\cdot\mid s,a)\otimes P(\cdot\mid s,a)$ and $\what{Q}^{(n)}_{s,a}=\what{\gP}_R^{(n)}(\cdot\mid s,a)\otimes\what{P}^{(n)}(\cdot\mid s,a)$. 
Define $\bH,\bH_n\colon\RB^{\gS\times[m]}\to\RB^{\gS\times[m]}$ as 
\begin{align*}
    (\bH(\btheta))_{s,i}&=\sum_{a\in\gA}\iint\frac{\pi(a\mid s)}{m}\sum_{k=1}^{m}\ind\{R_{s,a}+\gamma\theta(S_{s,a},k)<\theta(s,i)\}Q_{s,a}(\mathrm{d}(R_{s,a},S_{s,a}))&\\
    (\bH_n(\btheta))_{s,i}&=\sum_{a\in\gA}\iint\frac{\pi(a\mid s)}{m}\sum_{k=1}^{m}\ind\{R_{s,a}+\gamma\theta(S_{s,a},k)<\theta(s,i)\}\what{Q}^{(n)}_{s,a}(\mathrm{d}(R_{s,a},S_{s,a})), 
\end{align*}
and denote $(\bh^\btheta(R,S))_{s,i}=\sum_{a\in\gA}\frac{\pi(a\mid s)}{m}\sum_{k=1}^{m}\ind\{R_{s,a}+\gamma\theta(S_{s,a},k)<\theta(s,i)\}$. 

Then $\bH(\btheta_m)=\bT$ and $\bH_n(\btheta_m^{(n)})=\bT+\bm\varepsilon_n$, where $T_{s,i}=\tau_i$ and $\|\bm\varepsilon_n\|_\infty\leq n^{-1}$. 
The key regularity condition is the non-singularity of the Jacobian matrix, which is established in the following lemma. 
\begin{lemma}\label{lem:G_invertible}
    $\nabla \bH(\btheta_m)=\bG_m$ and $\bG_m$ is invertible for every $m\in\NB$. 
\end{lemma}
The proof of the lemma above is presented in Appendix~\ref{Appendix_lem_G_invertible}. 

Therefore, we have
\begin{align*}
    \bm\varepsilon_n&=\bH_n(\btheta_m^{(n)})-\bH(\btheta_m)\\
    &=\bH(\btheta_m^{(n)})-\bH(\btheta_m)+\bH_n(\btheta_m^{(n)})-\bH(\btheta_m^{(n)})\\
    &=\bG_m(\btheta_m^{(n)}-\btheta_m)+o_P(\Vert\btheta_m^{(n)}-\btheta_m\Vert)+\bH_n(\btheta_m^{(n)})-\bH(\btheta_m^{(n)})
\end{align*}
Denote $\bZ_n(\btheta)=\sqrt{n}[\bH_n(\btheta)-\bH(\btheta)]$ and rearranging the terms, we have
\begin{equation*}
    \sqrt{n}(\btheta_m^{(n)}-\btheta_m)=-\bG_m^{-1}\bZ_n(\btheta_m)-\bG_m^{-1}[\bZ_n(\btheta_m^{(n)})-\bZ_n(\btheta_m)]+o_P(1).  
\end{equation*}
To control the empirical process remainder $\bZ_n(\btheta_m^{(n)})-\bZ_n(\btheta_m)$, we require the following Donsker property, which is proved in Appendix~\ref{Appendix_lem_donsker}. 
\begin{lemma}\label{lemma:donsker}
    Denote $\gH_{s,i}=\{h^\btheta_{s,i}-h^{\btheta_m}_{s,i}\colon\|\btheta-\btheta_m\|_\infty\leq(1-\gamma)^{-1}\}$, then for every $(s,i)$, $\gH_{s,i}$ is $\bigotimes_{a\in\gA}Q_{s,a}$-Donsker, namely the empirical process indexed by $\gH_{s,i}$ satisfies a functional central limit theorem. 
\end{lemma}
By Lemma~\ref{lemma:donsker} and~\citet{van2023weak}, we know that $\bZ_n(\btheta_m^{(n)})-\bZ_n(\btheta_m)=o_P(1)$. 
Finally, by the multivariate central limit theorem, 
\begin{equation*}
    -\bG_m^{-1}\bZ_n(\btheta_m)\cd\gN(\bm{0}_m,\bG_m^{-1}\bSigma_m\bG_m^{-\top}), 
\end{equation*}
hence the asymptotic normality follows. 

To characterize the semiparametric efficiency bound, consider a regular parametric submodel $\brc{Q^\epsilon}$ with score $\bm g$, namely $\mathrm{d}Q^\epsilon_{s,a}=(1+\epsilon g_{s,a})\mathrm{d}Q_{s,a}$. 
Denote the distributional Bellman operator corresponding to $Q^\epsilon$ as $\gT^\pi_\epsilon$, and the quantile parameters of the fixed point of $\bPi_m\gT^\pi_\epsilon$ as $\btheta_m^\epsilon$. 
The following lemma characterize the pathwise derivative of the parameter and it is proved in Appendix~\ref{Appendix_lem_influence_function}.   
\begin{lemma}\label{lem:influence_function}
    Define the matrix $\bm Y_m\in\RB^{(\gS\times[m])\times(\gS\times\gA)}$ by
    \begin{equation*}
        (\bm Y_m)_{(s,i),(s^\prime,a)}=\ind\brc{s=s^\prime}\frac{\pi(a\mid s)}{m}\sum_{k=1}^{m}\ind\{R_{s,a}+\gamma\theta(S_{s,a},k)<\theta(s,i)\}
    \end{equation*}
    Then
    \begin{equation*}
        \left(\frac{\rd\btheta_m^\epsilon}{\rd\epsilon}\Big|_{\epsilon=0}\right)_{s,i}=-\sum_{s^\prime,a}\EB_{Q_{s^\prime,a}}[(\bG^{-1}_m(\bm Y_m-\EB\bm Y_m))_{(s,i),(s^\prime,a)}g_{s^\prime,a}]. 
    \end{equation*}
\end{lemma}
Lemma~\ref{lem:influence_function} shows that the efficient influence
function of $\btheta_m$ is $\bm\phi_m=-\bG_m^{-1}(\bm Y_m-\EB\bm Y_m)$. 
According to~\citet{van2000asymptotic}, the efficiency lower bound is given by
\[
\var_Q(\bm\phi_m)=\bG_m^{-1}\bSigma_m\bG_m^{-\top}.
\]
Therefore the estimator $\btheta_m^{(n)}$ attains the semiparametric efficiency bound and is semiparametrically efficient.

\subsection{Analysis of Limiting Covariance Structure and Efficiency}
\label{subsec:proof_outline_limit_cov}

To study the asymptotic variance as $m\rightarrow\infty$, we introduce an operator-theoretic representation and identify its infinite-dimensional limit. First, we rewrite $\bG_m=\bD_m-\gamma \bW_m$, where
\begin{align*}
    (\bD_m)_{(s,i),(s^\prime,j)}&=\ind\{(s,i)=(s^\prime,j)\}\sum_{a\in\gA}\pi(a\mid s)\sum_{\tilde{s}\in\mathcal{S},k\in[m]}\frac{P(\tilde{s}\mid s,a)}{m}p_{s,a}(\theta_m(s,i)-\gamma\theta_m(\tilde{s},k))\\
    (\bW_m)_{(s,i),(s^\prime,j)}&=\sum_{a\in\gA}\pi(a\mid s)\frac{P(s^\prime\mid s,a)}{m}p_{s,a}(\theta_m(s,i)-\gamma\theta_m(s^\prime,j)). 
\end{align*}
Moreover, define $\bK_m=\bD_m^{-1}\bW_m$ and the asymptotic variance can be expressed as
\begin{equation*}
    \sigma_{m,f,s}^2=\bm\varphi_{m,f,s}^\top(\bI-\gamma \bK_m)^{-1}(\bD_m^{-1}\bSigma_m\bD_m^{-1})(\bI-\gamma \bK_m)^{-\top}\bm\varphi_{m,f,s}
\end{equation*}

To identify the limit of this expression, we introduce the averaging and embedding operators
\begin{align*}
    \gR_m:\gL\rightarrow\RB^{\gS\times m}&,\quad (\gR_m\bm\psi)_{s,i}=m\int_{\frac{i-1}{m}}^{\frac{i}{m}}\psi_s(\tau)\mathrm{d}\tau,\\
    \gE_m:\RB^{\gS\times m}\rightarrow\gL&,\quad (\gE_m\bm\psi)(s,\tau)=\sum_{i=1}^m\psi_s(\tau_i)\ind_{[\tau_i,\tau_{i+1})}(\tau).
\end{align*}

Furthermore, we define 
\begin{align*}
    \psi_m(s,\tau)&=\ind\{s=s_0\}\sum_{i=1}^mf^\prime\left(F^{-1}_{(\gT^\pi\bm\eta_m)(s)}(\tau_i)\right)\ind_{[\frac{i-1}{m},\frac{i}{m})}(\tau),\\
    \bu_m&=\gE_m(\bI-\gamma \bK_m)^{-\top}\gR_m\bm\psi_m
\end{align*}
and $\wtilde{\bSigma}_m$ on $\gL$ as 
\begin{equation*}
    \wtilde{\bSigma}_m\psi=\frac{1}{m}\gE_m\bD_m^{-1}\bSigma_m\bD_m^{-1}\gR_m\bm\psi. 
\end{equation*}

Let $\gK$ and $\wtilde\Sigma$ denote the limiting operators introduced in Section~\ref{subsec:limit_covariance}. The next lemma establishes convergence of the discrete operators to their infinite-dimensional counterparts. 
Its proof is deferred to Appendix~\ref{Appendix_lem_operator_convergence}.
\begin{lemma}\label{lem:operator_convergence}
    Suppose either Assumption~\ref{Assumption_reward_lower_bounded} or~\ref{Assumption_reward_bounded_density} holds, and that Assumption~\ref{Assumption_return_density} holds. We have
    \[
    \lim_{m\to+\infty}\norm{\gE_m\bK_m\gR_m-\gK}=0,\quad \lim_{m\rightarrow+\infty}\norm{\wtilde{\bSigma}_m-\wtilde{\Sigma}}=0. 
    \]
\end{lemma}

The following proposition summarizes the key properties of the finite-dimensional representation and it is proved in Appendix~\ref{Appendix_prop_properties_summarize}.  
\begin{proposition}\label{prop:properties_summarize}
    The following statements hold: 
    \begin{itemize}
        \item $\bu_m$ is the unique solution in $\gL$ of the equation
        \begin{equation*}
            (\gI-\gamma \gE_m\bK_m^\top \gR_m)\bu=\bm\psi_m; 
        \end{equation*}
        \item The asymptotic variance admits the representation $\sigma_{m,f,s_0}^2=\langle \bu_m,\wtilde{\bSigma}_m\bu_m\rangle$; 
        \item As $m\rightarrow\infty$, $\bm\psi_m\rightarrow\bm\psi$ in $\gL$; 
        \item The adjoint operator of $\gE_m\bK_m\gR_m$ is given by $(\gE_m\bK_m\gR_m)^*=\gE_m\bK_m^\top \gR_m$. 
    \end{itemize}
\end{proposition}

By Proposition~\ref{prop:K_invertible} and Lemma~\ref{lem:operator_convergence}, we know that $\gI-\gamma \gE_m\bK_m\gR_m$ is invertible for all sufficiently large $m$. Then Proposition~\ref{prop:properties_summarize} implies that $\bu_m\rightarrow \bu$ where $\bu=(\gI-\gamma \gK^*)^{-1}\bm\psi$. Consequently,
\begin{equation*}
    \sigma_{m,f,s_0}^2=\langle \bu_m,\wtilde\bSigma_m \bu_m\rangle\rightarrow\langle \bu,\wtilde\Sigma \bu\rangle,
\end{equation*}
which proves the first claim of Theorem~\ref{thm:variance_convergence}. 

For the second claim of the theorem, first we figure out that~\citet{zhang2025estimation} have established that if $\bm\nu=(\nu_s)_{s\in\gS}$ where $\nu_s$ is a finite signed measure on $[0,(1-\gamma)^{-1}]$ with zero total mass for every $s\in\gS$, the equation $(\gI-\gT^\pi)\bm\mu=\bm\nu$ admits a unique solution and we denote it as $\bm\mu=(\gI-\gT^\pi)^{-1}\bm\nu$. For every $s\in\gS$, $[(\gI-\gT^\pi)^{-1}\bm\nu]_s$ is also a finite signed measure on $[0,(1-\gamma)^{-1}]$ with zero total mass. Furthermore, they established that
\begin{equation*}
    \sqrt{n}\left(\eta^{(n)}(s_0)-\eta^\pi(s_0)\right)f\cd\gN(0,\wtilde{\sigma}_{f,s_0}^2), 
\end{equation*}
\begin{equation*}
    \wtilde{\sigma}_{f,s_0}^2=\var_{(R,S)}\bigl([(\gI-\gT^\pi)^{-1}\bm\nu]_{s_0}f\bigr),  
\end{equation*}
where $\bm\nu$ is a random measure with zero total mass defined as
\begin{equation*}
    \nu_s(R_{s},S_{s})(B)=\sum_{a\in\gA}\pi(a\mid s)\left[(b_{R_{s,a},\gamma})_\#\eta^\pi(S_{s,a})\right](B)-\eta^\pi(s)(B), 
\end{equation*}
for every Borel set $B$, and 
\begin{equation*}
    (R,S)=(R_{s,a},S_{s,a})_{(s,a)\in\gS\times\gA}\sim\bigotimes_{s,a}Q_{s,a}. 
\end{equation*}

It remains to identify $\wtilde{\sigma}_{f,s_0}^2=\sigma_{f,s_0}^2$. We introduce an equivalent representation on a weighted quantile space. A weighted $L^2[0,1]$ space with weight function $w$ is defined as
\begin{equation*}
    L^2_w[0,1]=\left\{\psi\mid \norm{\psi}_w^2=\int_0^1\frac{|\psi|^2}{w}<\infty\right\}. 
\end{equation*}
Then we can define
\begin{equation*}
    \gL_{\bm w}=\bigoplus_{s\in\gS}L^2_{w_s}[0,1]
\end{equation*}
with norm $\norm{\bm\psi}^2_{\bm w}=\sum_{s\in\gS}\norm{\psi_s(\cdot)}_{w_s}^2$ where $w_s=p_{\eta^\pi(s)}\circ F^{-1}_{\eta^\pi(s)}$. 
Now we define 
\begin{align*}
    (\Bar{\gK}\bm\psi)_s(t)&=\sum_{s^\prime\in\gS}\langle \Bar{T}_{s,s^\prime}(t,\cdot),\psi_{s^\prime}(\cdot)\rangle_{w_{s^\prime}}\\
    (\Bar{\Sigma}\bm\psi)_s(t)&=\langle A_s(t,\cdot),\psi_s(\cdot)\rangle_{w_s}, 
\end{align*}
where
\begin{equation*}
    \Bar{T}_{s,s^\prime}(t,\tau)=\sum_{a\in\gA}\pi(a\mid s)P(s^\prime\mid s,a)p_{s,a}\left(F^{-1}_{\eta^\pi(s)}(t)-\gamma F^{-1}_{\eta^\pi(s^\prime)}(\tau)\right). 
\end{equation*}
The following lemma provides an equivalent representation of the limiting variance.
\begin{lemma}\label{lem:var_expression}
    $\gI-\gamma\Bar{\gK}$ is invertible and we have
    \begin{equation*}
        \sigma_{f,s_0}^2=\var_{(R,S^\prime)}\left(\left\langle[(\gI-\gamma \Bar{\gK})^{-1}\bm Z]_{s_0},f^\prime\circ F^{-1}_{\eta^\pi(s_0)}\right\rangle_{w_{s_0}}\right), 
    \end{equation*}
    where $\bZ$ is a mean zero random vector function defined as
    \begin{equation*}
        Z_s(t;R_{s},S_{s})=\sum_{a\in\gA}\pi(a\mid s)F_{\eta^\pi(S_{s,a})}\left(\frac{F^{-1}_{\eta^\pi(s)}(t)-R_{s,a}}{\gamma}\right)-t. 
    \end{equation*}
\end{lemma}

The key observation is that the random vector function $\bZ$
provides a quantile-space representation of the signed
measure $\bm\nu$. 
Under this identification,
the operator $\gamma\Bar \gK$ acts as the distributional
Bellman operator $\gT^\pi$ on the cumulative mass functions
of signed measures, which we summarize in the following lemma. 
\begin{lemma}\label{lem:correspondence}
    We have
    \begin{equation*}
        Z_s(t;R_{s},S_{s})=\nu_s(R_{s},S_{s})(-\infty,F^{-1}_{\eta^\pi(s)}(t)], 
    \end{equation*}
    and
    \begin{equation*}
        (\gamma \Bar{\gK}\bZ)_s(t;R_{s},S_{s})=(\gT^\pi\bm\nu)_s(R_{s},S_{s})(-\infty,F^{-1}_{\eta^\pi(s)}(t)]. 
    \end{equation*}
    Therefore, $[(\gI-\gamma\Bar{\gK})^{-1}\bZ]_s(t)=[(\gI-\gT^\pi)^{-1}\bm\nu]_s(-\infty,F^{-1}_{\eta^\pi(s)}(t)]$.  
\end{lemma}

The proofs of Lemma~\ref{lem:var_expression} and~\ref{lem:correspondence} are presented in Appendix~\ref{Appendix_lem_var_expression} and~\ref{Appendix_lem_correspondence} respectively. By the two lemmas we have
\begin{align*}
    \left\langle[(\gI-\gamma \Bar{\gK})^{-1}\bZ]_{s},f^\prime\circ F^{-1}_{\eta^\pi(s)}\right\rangle_{w_s}=&\int_0^1\frac{\left[(\gI-\gT^\pi)^{-1}\bm\nu\right]_s(R_{s},S_{s})(-\infty,F^{-1}_{\eta^\pi(s)}(t)]}{p_{\eta^\pi(s)}\left(F^{-1}_{\eta^\pi(s)}(t)\right)}f^\prime\left(F^{-1}_{\eta^\pi(s)}(t)\right)\rd t\\
    =&\int_0^{(1-\gamma)^{-1}}\left[(\gI-\gT^\pi)^{-1}\bm\nu\right]_s(R_{s},S_{s})(-\infty,x]f^\prime\left(x\right)\rd x\\
    =&-[(\gI-\gT^\pi)^{-1}\bm\nu]_{s}f, 
\end{align*}
where the last equality follows from integration by parts and the conclusion follows. 

Finally, in order to establish the semiparametric efficiency result in the infinite-dimensional problem, we prove the following lemma. 
\begin{lemma}\label{lem:infinite_dim_efficiency}
    Suppose $\bm\psi=(\psi_s)_{s\in\gS}$ where $\psi_s$ is Lipschitz continuous for every $s$. 
    Denote the fixed point of $\gT^\pi_\epsilon$ as $\bm\eta^\pi_\epsilon$ and define $\bm Y$ by
    \begin{equation*}
        Y_{s,(s^\prime,a)}=\ind\brc{s=s^\prime}\pi(a\mid s)(b_{R_{s,a},\gamma})_\#\eta^\pi(S_{s,a}), 
    \end{equation*}
    namely $\bm Y$ is a matrix whose entries are random measures with shape $\gS\times(\gS\times\gA)$. Then for any bounded score function $\bm g$ satisfying $\EB_{Q_{s,a}}[g_{s,a}]=0$, 
    \begin{equation*}
        \frac{\rd\<\bm\eta^\pi_\epsilon,\bm\psi\>}{\rd\epsilon}\Bigg|_{\epsilon=0}=\sum_{s^\prime,a}\EB_{Q_{s^\prime,a}}\left[g_{s^\prime,a}\left\langle(\gI-\gT^\pi)^{-1}\left(\bm Y_{\cdot,(s^\prime,a)}-\EB\bm Y_{\cdot,(s^\prime,a)}\right),\bm\psi\right\rangle\right], 
    \end{equation*}
    where $\<\bm\eta,\bm\psi\>=\sum_{s\in\gS}\int \psi_s\rd\eta(s)$. 
\end{lemma}
Lemma~\ref{lem:infinite_dim_efficiency} is proved in Appendix~\ref{Appendix_lem_infinite_dim_efficiency}. 
According to~\citet{van2000asymptotic}, when $\psi_{s^\prime}=\ind\brc{s^\prime=s}f$, the efficiency lower bound is given by
\begin{align*}
    &\sum_{s^\prime,a}\var_{Q_{s^\prime,a}}\prn{\left\langle(\gI-\gT^\pi)^{-1}(\bm Y_{\cdot,(s^\prime,a)}-\EB\bm Y_{\cdot,(s^\prime,a)}),\bm\psi\right\rangle}\\
    =&\var_Q\brk{\prn{\sum_{a\in\gA}(\gI-\gT^\pi)^{-1}\left(\bm Y_{\cdot,(s,a)}-\EB\bm Y_{\cdot,(s,a)}\right)}_sf}\\
    =&\var_Q\left([(\gI-\gT^\pi)^{-1}\bm\nu]_sf\right)=\sigma_{f,s}^2.
\end{align*}

\subsection{Analysis of Berry--Esseen Bound}\label{subsec:proof_outline_BE}
\label{subsec:proof_BE}
In this section we present the proof of Theorem \ref{thm:berry-esseen}. The proof is based on the following Berry--Esseen theorem for general nonlinear statistic due to
\citet{Chen_2007,shao2022berry,Shao_2016}.
\begin{theorem}\label{thm:SZ}
Let $T = W + D$ where $W = \sum_{i=1}^n \xi_i$ with $\xi_i=h_i(X_i)\in\RB$, $\mathbb{E}\xi_i = 0$, 
$\sum_{i=1}^n \mathbb{E}|\xi_i|^2=1$. Let $O$ be a measurable set, then for any random variables $\Delta \geq \vert D\vert\ind_O$ 
and $\Delta^{(i)}$ independent of $X_i$,
\begin{equation*}
    \sup_{t\in\RB} |\PB[T\leq t] - \Phi(t)| \lesssim \sum_{j=1}^n \EB|\xi_j|^3 + \EB\Delta + \sum_{i=j}^n \EB[|\xi_j||\Delta-\Delta^{(j)}|]+\PB(O^c),
\end{equation*}
where $\Phi(\cdot)$ is the cumulative distribution function of standard normal distribution. 
\end{theorem}

We apply Theorem~\ref{thm:SZ} to the normalized functional
\begin{equation*}
    T\coloneq\sqrt{n}\sigma_{m,s,f}^{-1}\left(\eta_m^{(n)}(s)-\eta_m(s)\right)f. 
\end{equation*}
To this end, we decompose $T=W+D$, where $W$ is the linear term arising from the empirical process and $D$ is the nonlinear remainder. Define $\GB_n(\theta)=\sqrt{n}\bm\varphi_{m,f,s}^\top \bG_m^{-1}[\bH_n(\btheta)-\bH(\btheta)]$, then we have $T=W+D$, where
\begin{align*}
    W=&-\frac{\sqrt{n}}{\sigma_{m,s,f}}\bm\varphi_{m,f,s}^\top \bG_m^{-1}[\bH_n(\btheta_m)-\bH(\btheta_m)]\coloneq-\frac{1}{\sqrt{n}}\sum_{i=1}^{n}\xi_i\\
    D=&\frac{\sqrt{n}}{\sigma_{m,s,f}}\bm\varphi_{m,f,s}^\top \bG_m^{-1}[\bH_n(\btheta_m^{(n)})-\bH(\theta_m)]-\frac{\sqrt{n}}{\sigma_{m,s,f}}\bm\varphi_{m,f,s}^\top \bG_m^{-1}[\bH(\btheta_m^{(n)})-\bH(\btheta_m)-\bG_m(\btheta_m^{(n)}-\btheta_m)]\\
    &-\sigma_{m,s,f}^{-1}[\GB_n(\btheta_m^{(n)})-\GB_n(\btheta_m)]+\frac{\sqrt{n}}{m}\sum_{i=1}^{m}\left[f(\theta_m^{(n)}(s,i))-f(\theta_m(s,i))-f^\prime(\theta_m(s,i))(\theta_m^{(n)}(s,i)-\theta_m(s,i))\right]. 
\end{align*}
The following lemma, proved in Appendix~\ref{Appendix_lem_BE_delta}, shows that the remainder term $D$ can be controlled by two simpler quantities. In addition, $\lesssim$ in the following three lemmas hides a positive constant independent of $m$ and $n$. 
\begin{lemma}\label{lem:BE_delta}
Suppose that Assumptions~\ref{Assumption_reward_bounded_density} and~\ref{Assumption_return_density} hold. Define $O_n=\{\|\btheta_m^{(n)}-\btheta_m\|_\infty\leq\delta_n\}$, and we have $\vert D\vert\ind_{O_n}\lesssim\Delta_1+\Delta_2$, where $C$ is a universal constant and 
\begin{align*}
    \Delta_1&=\sup_{\norm{\btheta-\btheta_m}_\infty\leq\delta_n}\vert\GB_n(\btheta)-\GB_n(\btheta_m)\vert\\
    \Delta_2&=mn^{-\frac{1}{2}}+m\sqrt{n}\norm{\btheta_m^{(n)}-\btheta_m}_\infty^2\ind_{O_n}.
\end{align*}
\end{lemma}

Therefore, to apply Theorem~\ref{thm:SZ}, we need to control the expectations of $\Delta_1$ and $\Delta_2$. Moreover, it remains to construct suitable random variable $\Delta_1^{(i)}$, $\Delta_2^{(i)}$ and control the difference terms appearing in Theorem~\ref{thm:SZ}. To this end, we introduce a sample-replacement construction. 
Let $\{(\wtilde{X}_1^{(s,a)},\cdots,\wtilde{X}_n^{(s,a)})\}_{(s,a)\in\gS\times\gA}$ be an iid copy of the samples. We can define $\bH_n^{(j)}(\btheta)$, $\GB^{(j)}_n(\btheta)$, $\Delta_1^{(j)}$, $O_n^{(j)}$ and $\what{\gT}^\pi_{n,j}$ by replacing $X_j^{(s,a)}$ with $\wtilde{X}_j^{(s,a)}$ for all $(s,a)$. Moreover, define $\btheta^{(n,j)}_m$ to be the fixed point of $\bPi_m\what{\gT}^\pi_{n,j}$ and 
\begin{equation*}
    \Delta_2^{(j)}=mn^{-\frac{1}{2}}+m\sqrt{n}\norm{\btheta_m^{(n,j)}-\btheta_m}_\infty^2\ind_{O_n^{(j)}}. 
\end{equation*}
With these definitions in hand, the required bounds are summarized in Lemma~\ref{lem:BE_term_12} and~\ref{lem:BE_term_3}, whose proofs are deferred to Appendix~\ref{Appendix_lem_BE_term_12} and~\ref{Appendix_lem_BE_term_3} respectively. 
\begin{lemma}\label{lem:BE_term_12}
Suppose that Assumptions~\ref{Assumption_reward_bounded_density} and~\ref{Assumption_return_density} hold. We have
\begin{equation}\label{eq:third_moment}
    \sum_{i=1}^{n}\EB|\xi_i|^3\lesssim m^2n
\end{equation}
and
\begin{align}
    \EB\Delta_1&\lesssim m\log (em)\left[\sqrt{\delta_nm\log (em)}+\frac{\log (en)}{\sqrt{n}}\right]\label{eq:delta_1_ub}\\
    \EB\Delta_2&\lesssim \frac{m^2\log (em)}{\sqrt{n}}\label{eq:delta_2_ub}
\end{align}
\end{lemma}
\begin{lemma}\label{lem:BE_term_3}
    Suppose that Assumptions~\ref{Assumption_reward_bounded_density} and~\ref{Assumption_return_density} hold. For every $m,n$, we have
    \begin{align}
        \sum_{j=1}^{n}\EB\vert\xi_j\vert\vert\Delta_1-\Delta_1^{(j)}\vert&\lesssim mn^{\frac{1}{2}}\sqrt{\delta_n}\label{eq:delta_1i}\\
        \sum_{j=1}^{n}\EB\vert\xi_j\vert\vert\Delta_2-\Delta_2^{(j)}\vert&\lesssim m^2n^{\frac{3}{2}}\left\{\delta_n^2\exp\left(-\frac{n\delta_n^2}{m}\right)+\sqrt{\frac{m\log (em)}{n}}\left[\sqrt{\frac{\delta_nm\log (em)}{n}}+\frac{\log (en)}{n}+\delta_n^2\right]\right\}\label{eq:delta_2i}
    \end{align}
\end{lemma}

Now we can present the proof of the Berry--Esseen bound. 
\begin{proof}[Proof of Theorem~\ref{thm:berry-esseen}]
From Theorem~\ref{thm:non_asymp_improved}, 
\begin{equation*}
    \PB(O_n^c)\lesssim m\exp\prn{-\frac{Cn\delta_n^2}{m}},
\end{equation*}
where $C$ is a constant independent of $m$ and $n$. Choose
\begin{equation*}
    \delta_n=C(\gM,f)\sqrt{\frac{m\log (en)}{n}}
\end{equation*}
for a sufficiently large constant $C(\gM,f)$ independent of $m$ and $n$. 

Substituting the bounds in Lemmas~\ref{lem:BE_delta},~\ref{lem:BE_term_12} and~\ref{lem:BE_term_3} into Theorem~\ref{thm:SZ},
we obtain
\begin{align*}
    &\sup_{t\in\RB}\left\vert\PB\left[\frac{\sqrt{n}}{\sigma_{m,f,s}}\left(\eta_m^{(n)}(s)-\eta_m(s)\right)f\leq t\right]-\Phi(t)\right\vert\\
    \lesssim&\prn{\frac{m^{13}\log^4 (em)\log^2 (en)}{n}}^{\frac{1}{4}}\left[1+\prn{\frac{m^{13}\log^4 (em)\log^2 (en)}{n}}^{\frac{1}{4}}\right].  
\end{align*}
If $[m^{13}n^{-1}\log^4 (em)\log^2(en)]^{1/4}\geq 1$, the conclusion follows trivially. Otherwise, 
\begin{equation*}
    \sup_{t\in\RB}\left\vert\PB\left[\frac{\sqrt{n}}{\sigma_{m,f,s}}\left(\eta_m^{(n)}(s)-\eta_m(s)\right)f\leq t\right]-\Phi(t)\right\vert\lesssim2\prn{\frac{m^{13}\log^4 (em)\log^2 (en)}{n}}^{\frac{1}{4}}.  
\end{equation*}
Therefore Equation~\eqref{eq:BE} holds. 
\end{proof}




\section{Numerical Experiments}\label{Section:numerical_experiments}
In this section we conduct numerical simulations to validate our theoretical findings as well as the proposed inferential procedures.

\subsection{Experimental Setup}
We perform the simulations in a tabular MDP with $\gS=\{s_1,s_2\}$, $\gA=\{a_1\}$ and $\gamma=0.9$. Since $|\gA|=1$, we can omit it in our notation. 
The transition probabilities are defined as
\begin{align*}
    P(s_1\mid s_1)=0.7&,\ P(s_2\mid s_1)=0.3,\\
    P(s_1\mid s_2)=0.4&,\ P(s_2\mid s_2)=0.6, 
\end{align*}
and the density functions of $\gP_R(\cdot\mid s_1)$ and $\gP_R(\cdot\mid s_2)$ are defined as
\begin{equation*}
    p_{s_1}(x)=\frac{\pi}{2}\sin(\pi x)\ind_{[0,1]}(x),\ p_{s_2}(x)=4x(1-x^2)\ind_{[0,1]}(x), 
\end{equation*}
which satisfy Assumption~\ref{Assumption_reward_bounded_density}. We choose $f(x)=\sin x-x\cos x$ as the test function. 
We conduct the experiments under $m\in\{20,50,100\}$ and $n\in\{10,100,1000,10000\}$. For each pair $(m,n)$, we perform $1000$ independent Monte Carlo replications. In each replication, $\bm\eta_m$ and $\bm\eta_m^{(n)}$ are computed via quantile dynamic programming.

\subsection{Finite-sample Convergence Behavior}
We investigate the finite-sample convergence behavior and verify our non-asymptotic results. We report the average of $\Bar{W}_\infty(\bm\eta_m^{(n)},\bm\eta_m)$ over the $1000$ replications and regress $\log\Bar{W}_\infty(\bm\eta_m^{(n)},\bm\eta_m)$ on $\log n$. The results are presented in Figure~\ref{fig:nonasymp} and Table~\ref{tab:nonasymp}. 
The estimated slopes for all cases are consistently close to $-\frac{1}{2}$, with $R^2$ values exceeding $0.999$, providing strong numerical evidence for the $n^{-\frac{1}{2}}$ convergence rate established in Theorem~\ref{thm:non_asymp_improved}. We also observe that the estimated intercepts exhibit only a weak dependence on $m$ over the range considered in the experiment.

\begin{figure}[!htbp]
    \centering
    \includegraphics[width=\textwidth]{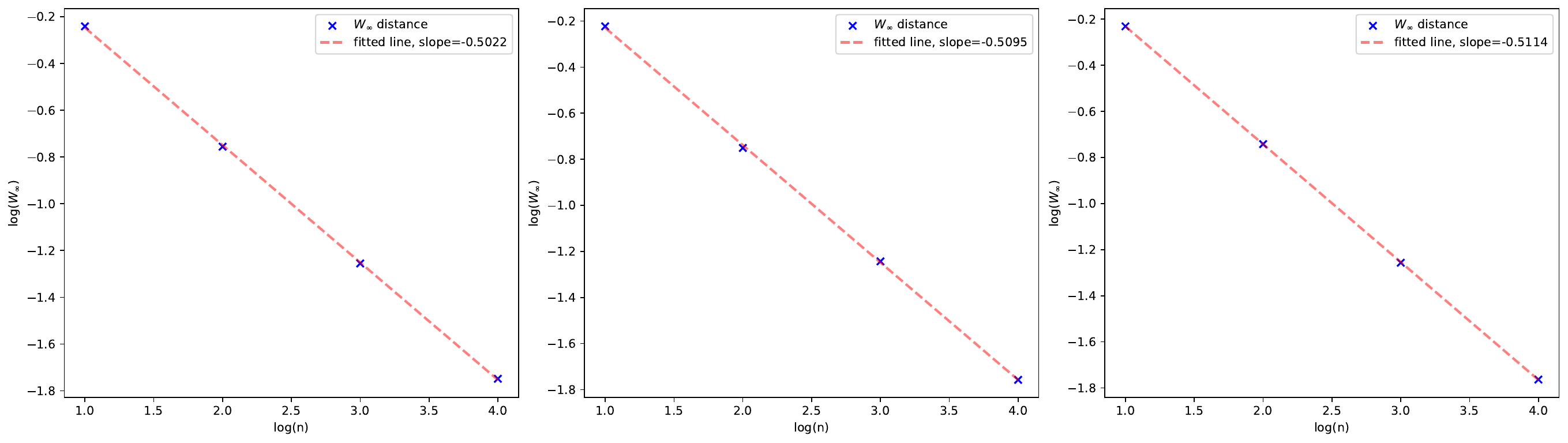}
    \caption{The statistical error $\Bar{W}_\infty(\bm\eta_m^{(n)},\bm\eta_m)$ with different sample sizes. From left to right: $m=20$, $m=50$, $m=100$.}
    \label{fig:nonasymp}
\end{figure}

\begin{table}[t!]
\centering
    \begin{tabular}{cccc}
    \hline
    $m$ & slope & intercept & $R^2$\\
    \hline
    20 & -0.5022 & 0.5881 & 0.9999\\
    50 & -0.5095 & 0.6457 & 0.9998 \\
    100 & -0.5114 & 0.6447 & 1.0000 \\
    \hline
    \end{tabular}
    \caption{Estimated regression coefficients and coefficient of determination ($R^2$) for the regression of $\log \Bar{W}_\infty(\eta_m^{(n)},\eta_m)$ on $\log n$ under different choices of $m$.}
    \label{tab:nonasymp}
\end{table}

\subsection{Asymptotic Normality}
We next investigate the asymptotic normality of the test functional predicted by Corollary~\ref{cor:CLT_test_functional}. Figure~\ref{fig:asymp} presents the QQ plots of the standardized estimator
\[
\frac{\sqrt{n}}{\sigma_{m,f,s_1}}\left(\eta_m^{(n)}(s_1)-\eta_m(s_1)\right)f
\]
for $n=10000$ and different choices of $m$. In all cases, the empirical quantiles closely follow the reference line, indicating good agreement with the standard normal distribution.

An additional numerical evidence is provided in Table~\ref{tab:asymp}. The empirical skewness and excess kurtosis are both close to zero for all values of $m$, suggesting that the standardized estimator is approximately symmetric and exhibits Gaussian-like tail behavior. Moreover, the Shapiro--Wilk test does not reject the null hypothesis of normality at the $0.05$ significance level in any of the cases considered. These observations are consistent with the asymptotic normality established in Corollary~\ref{cor:CLT_test_functional}.
\begin{figure}[!htbp]
    \centering
    \includegraphics[width=\textwidth]{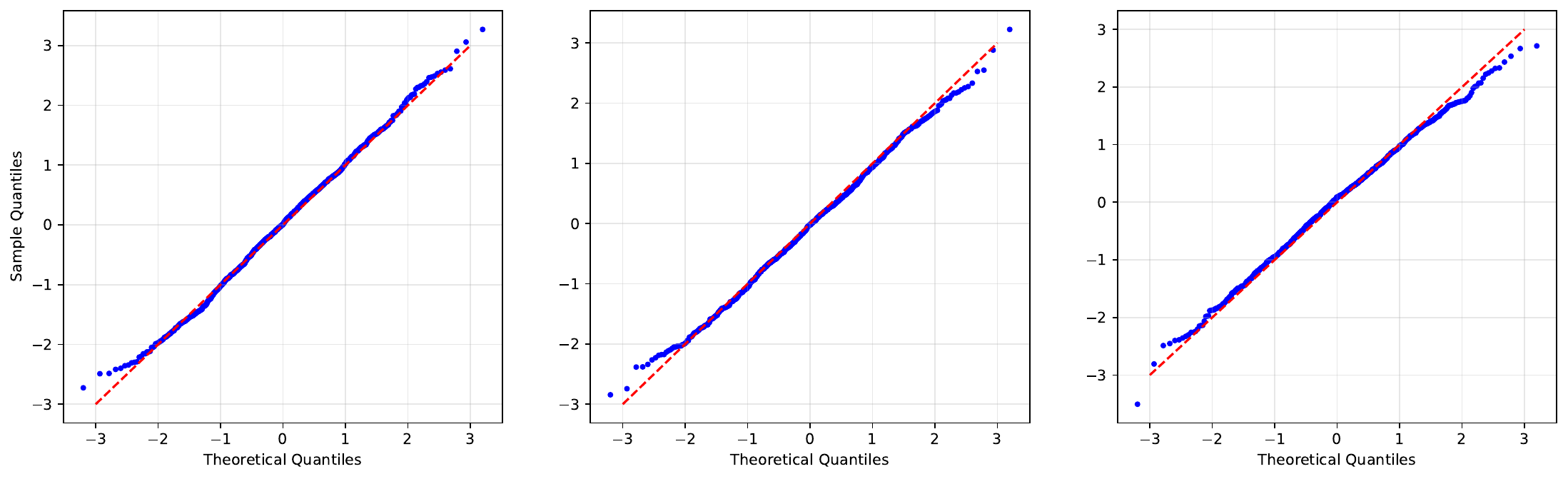}
    \caption{The QQ plot of standardized estimator $\sqrt{n}(\eta_m^{(n)}(s_1)-\eta_m(s_1))f/\sigma_{m,f,s_1}$ for sample size $n=10000$. From left to right: $m=20$, $m=50$, $m=100$.}
    \label{fig:asymp}
\end{figure}

\begin{table}[t!]
\centering
    \begin{tabular}{cccc}
    \hline
    $m$ & skewness & excess kurtosis & Shapiro--Wilk $p$-value\\
    \hline
    20 & 0.0347 & -0.2063 & 0.2491\\
    50 & 0.0485 & -0.2085 & 0.4592 \\
    100 & -0.1411 & -0.0679 & 0.2873 \\
    \hline
    \end{tabular}
    \caption{Skewness, excess kurtosis, and Shapiro--Wilk $p$-values of the standardized estimator under sample size $n=10000$ and different choices of $m$.}
    \label{tab:asymp}
\end{table}

\subsection{Validity of Inferential Procedures}
Finally, we evaluate the finite-sample performance of the proposed inferential procedure. Following Corollary~\ref{cor:inference}, we estimate the reward density $p_{s_1}$ and $p_{s_2}$ using Gaussian kernel density estimators with Scott's rule for bandwidth selection~\citep{scott2015multivariate}. Based on the resulting plug-in variance estimator $\what{\sigma}_{m,f,s_1}^2$, we construct nominal $0.95$ confidence intervals for $\eta_m(s_1)f$ by $\eta_m^{(n)}(s_1)f \pm1.96\what{\sigma}_{m,f,s_1}n^{-\frac{1}{2}}$. 

The coverage rates over $1000$ independent Monte Carlo replications are reported in Table~\ref{tab:CI}. For all choices of $m$, the coverage probabilities increase rapidly as the sample size grows. While noticeable finite-sample deviations are observed when $n=10$, the coverage rates are already close to the nominal $0.95$ level for $n=100$, and remain stable for larger sample sizes. These results provide strong numerical evidence for the validity of the plug-in confidence intervals proposed in Corollary~\ref{cor:inference}.
\begin{table}[t!]
\centering
    \begin{tabular}{c|cccc}
    \hline
     & $n=10$ & $n=100$ & $n=1000$ & $n=10000$\\
    \hline
    $m=20$ & 0.859 & 0.940 & 0.944 & 0.948 \\
    $m=50$ & 0.832 & 0.950 & 0.945 & 0.957 \\
    $m=100$ & 0.854 & 0.934 & 0.966 & 0.966 \\
    \hline
    \end{tabular}
    \caption{Coverage rate of our proposed confidence interval for $\eta_m(s_1)f$ under different choices of $m$ and $n$.}
    \label{tab:CI}
\end{table}
\section{Conclusions}\label{Section:conclusion}
In this paper, we have analyzed the statistical performance of distributional reinforcement learning with quantile parametrization from both non-asymptotic and asymptotic perspectives. We have presented non-asymptotic rate for $\Bar{W}_\infty(\bm\eta_m^{(n)},\bm\eta_m)$. We have also derived that the standardized quantile parameter $\sqrt{n}(\btheta_m^{(n)}-\btheta_m)$ converges weakly to a Gaussian distribution and that our quantile fixed point estimator $\btheta_m^{(n)}$ is semiparametrically efficient. 
Moreover, we have established a Berry--Esseen bound for the statistical functional $\sqrt{n}(\eta_m^{(n)}(s)-\eta_m(s))f$ given an initial state $s$. 
Based on our theoretical findings, we have devised inferential procedures for integral functionals induced by smooth test functions of the quantile fixed point $\bm\eta_m$. 

Beyond these finite-dimensional results, we have studied the asymptotic regime where the number of quantiles $m$ increases to infinity. By developing an operator-theoretic representation of the quantile-projected Bellman equation, we have characterized the limiting covariance structure and proved that the asymptotic variance converges to the semiparametric efficiency bound of the underlying non-parametric model for distributional policy evaluation. These findings reveal that quantile parametrization preserves the statistical efficiency of the underlying distributional policy evaluation problem in the large-quantile limit. This provides a theoretical justification for the widespread use of quantile-based methods in distributional reinforcement learning. 


One future direction is to improve the dependence on the number of quantiles $m$ in the Berry--Esseen bound. Another promising direction for future work is to extend the present theory beyond the model-based setting considered in this paper. This includes both offline reinforcement learning without access to a generative model and online quantile-based algorithms such as quantile temporal difference learning. This might give rise to a wider range of inferential applications in reinforcement learning.

\bibliography{ref}

@book{bdr2022,
    title={Distributional Reinforcement Learning},
    author={Marc G. Bellemare and Will Dabney and Mark Rowland},
    publisher={{MIT} Press},
    note={\url{http://www.distributional-rl.org}},
    year={2023}
}

@article{shao2022berry,
author = {Qi-Man Shao and Zhuo-Song Zhang},
title = {{Berry–Esseen bounds for multivariate nonlinear statistics with applications to {M}-estimators and stochastic gradient descent algorithms}},
volume = {28},
journal = {Bernoulli},
number = {3},
publisher = {Bernoulli Society for Mathematical Statistics and Probability},
pages = {1548 -- 1576},
year = {2022},
}

@article{Chen_2007,
   title={Normal approximation for nonlinear statistics using a concentration inequality approach},
   volume={13},
   number={2},
   journal={Bernoulli},
   publisher={Bernoulli Society for Mathematical Statistics and Probability},
   author={Chen, Louis H.Y. and Shao, Qi-Man},
   year={2007},
}

@article{Shao_2016,
   title={Cramér type moderate deviation theorems for self-normalized processes},
   volume={22},
   number={4},
   journal={Bernoulli},
   publisher={Bernoulli Society for Mathematical Statistics and Probability},
   author={Shao, Qi-Man and Zhou, Wen-Xin},
   year={2016},
}

@book{van2023weak,
  title={Weak Convergence and Empirical Processes: With Applications to Statistics},
  author={van der Vaart, AW and Wellner, Jon A},
  year={2023},
  publisher={Springer Nature}
}

@article{bellemare2020autonomous,
  title={Autonomous navigation of stratospheric balloons using reinforcement learning},
  author={Bellemare, Marc G and Candido, Salvatore and Castro, Pablo Samuel and Gong, Jun and Machado, Marlos C and Moitra, Subhodeep and Ponda, Sameera S and Wang, Ziyu},
  journal={Nature},
  volume={588},
  number={7836},
  pages={77--82},
  year={2020},
  publisher={Nature Publishing Group UK London}
}

@article{bodnar2019quantile,
  title={Quantile qt-opt for risk-aware vision-based robotic grasping},
  author={Bodnar, Cristian and Li, Adrian and Hausman, Karol and Pastor, Peter and Kalakrishnan, Mrinal},
  journal={arXiv preprint arXiv:1910.02787},
  year={2019}
}

@book{vershynin_2018, place={Cambridge}, series={Cambridge Series in Statistical and Probabilistic Mathematics}, title={High-Dimensional Probability: An Introduction with Applications in Data Science}, DOI={10.1017/9781108231596}, publisher={Cambridge University Press}, author={Vershynin, Roman}, year={2018}, collection={Cambridge Series in Statistical and Probabilistic Mathematics}}

@article{massart1990tight,
  title={The tight constant in the Dvoretzky-Kiefer-Wolfowitz inequality},
  author={Massart, Pascal},
  journal={The annals of Probability},
  pages={1269--1283},
  year={1990},
  publisher={JSTOR}
}

@book{van2000asymptotic,
  title={Asymptotic statistics},
  author={van der Vaart, AW},
  volume={3},
  year={2000},
  publisher={Cambridge university press}
}

@article{silver2018general,
  title={A general reinforcement learning algorithm that masters chess, shogi, and Go through self-play},
  author={Silver, David and Hubert, Thomas and Schrittwieser, Julian and Antonoglou, Ioannis and Lai, Matthew and Guez, Arthur and Lanctot, Marc and Sifre, Laurent and Kumaran, Dharshan and Graepel, Thore and others},
  journal={Science},
  volume={362},
  number={6419},
  pages={1140--1144},
  year={2018},
  publisher={American Association for the Advancement of Science}
}

@article{vinyals2019grandmaster,
  title={Grandmaster level in StarCraft II using multi-agent reinforcement learning},
  author={Vinyals, Oriol and Babuschkin, Igor and Czarnecki, Wojciech M and Mathieu, Micha{\"e}l and Dudzik, Andrew and Chung, Junyoung and Choi, David H and Powell, Richard and Ewalds, Timo and Georgiev, Petko and others},
  journal={Nature},
  volume={575},
  number={7782},
  pages={350--354},
  year={2019},
  publisher={Nature Publishing Group}
}

@article{kober2013reinforcement,
  title={Reinforcement learning in robotics: A survey},
  author={Kober, Jens and Bagnell, J Andrew and Peters, Jan},
  journal={The International Journal of Robotics Research},
  volume={32},
  number={11},
  pages={1238--1274},
  year={2013},
  publisher={SAGE Publications Sage UK: London, England}
}

@misc{openai2023gpt4,
      title={GPT-4 Technical Report}, 
      author={OpenAI},
      year={2023},
      eprint={2303.08774},
      archivePrefix={arXiv},
      primaryClass={cs.CL}
}

@article{ouyang2022training,
  title={Training language models to follow instructions with human feedback},
  author={Ouyang, Long and Wu, Jeff and Jiang, Xu and Almeida, Diogo and Wainwright, Carroll L and Mishkin, Pamela and Zhang, Chong and Agarwal, Sandhini and Slama, Katarina and Ray, Alex and others},
  journal={arXiv preprint arXiv:2203.02155},
  year={2022}
}

@book{sutton2018reinforcement,
  title={Reinforcement learning: An introduction},
  author={Sutton, Richard S and Barto, Andrew G},
  year={2018},
  publisher={MIT press}
}

@misc{sutton2004,   
    title = {The reward hypothesis},   
    url = {http://incompleteideas.net/rlai.cs.ualberta.ca/RLAI/rewardhypothesis.html},   
    author = {Sutton, Richard S},   
    year = {2004},   
}

@article{ghysels2005there,
  title={There is a risk-return trade-off after all},
  author={Ghysels, Eric and Santa-Clara, Pedro and Valkanov, Rossen},
  journal={Journal of financial economics},
  volume={76},
  number={3},
  pages={509--548},
  year={2005},
  publisher={Elsevier}
}

@article{lavori2004dynamic,
  title={Dynamic treatment regimes: practical design considerations},
  author={Lavori, Philip W and Dawson, Ree},
  journal={Clinical trials},
  volume={1},
  number={1},
  pages={9--20},
  year={2004},
  publisher={Sage Publications Sage CA: Thousand Oaks, CA}
  }

@inproceedings{bellemare2017distributional,
  title={A distributional perspective on reinforcement learning},
  author={Bellemare, Marc G and Dabney, Will and Munos, R{\'e}mi},
  booktitle={International conference on machine learning},
  pages={449--458},
  year={2017},
  organization={PMLR}
}

@inproceedings{morimura2010nonparametric,
  title={Nonparametric return distribution approximation for reinforcement learning},
  author={Morimura, Tetsuro and Sugiyama, Masashi and Kashima, Hisashi and Hachiya, Hirotaka and Tanaka, Toshiyuki},
  booktitle={Proceedings of the 27th International Conference on Machine Learning (ICML-10)},
  pages={799--806},
  year={2010}
}

@article{simon1956dynamic,
  title={Dynamic programming under uncertainty with a quadratic criterion function},
  author={Simon, Herbert A},
  journal={Econometrica, Journal of the Econometric Society},
  pages={74--81},
  year={1956},
  publisher={JSTOR}
}

@article{theil1957note,
  title={A note on certainty equivalence in dynamic planning},
  author={Theil, Henri},
  journal={Econometrica: Journal of the Econometric Society},
  pages={346--349},
  year={1957},
  publisher={JSTOR}
}

@article{hua2019gan,
  title={GAN-powered deep distributional reinforcement learning for resource management in network slicing},
  author={Hua, Yuxiu and Li, Rongpeng and Zhao, Zhifeng and Chen, Xianfu and Zhang, Honggang},
  journal={IEEE Journal on Selected Areas in Communications},
  volume={38},
  number={2},
  pages={334--349},
  year={2019},
  publisher={IEEE}
}

@article{fawzi2022discovering,
  title={Discovering faster matrix multiplication algorithms with reinforcement learning},
  author={Fawzi, Alhussein and Balog, Matej and Huang, Aja and Hubert, Thomas and Romera-Paredes, Bernardino and Barekatain, Mohammadamin and Novikov, Alexander and R Ruiz, Francisco J and Schrittwieser, Julian and Swirszcz, Grzegorz and others},
  journal={Nature},
  volume={610},
  number={7930},
  pages={47--53},
  year={2022},
  publisher={Nature Publishing Group}
}

@article{naeem2020generative,
  title={A generative adversarial network enabled deep distributional reinforcement learning for transmission scheduling in internet of vehicles},
  author={Naeem, Faisal and Seifollahi, Sattar and Zhou, Zhenyu and Tariq, Muhammad},
  journal={IEEE Transactions on Intelligent Transportation Systems},
  volume={22},
  number={7},
  pages={4550--4559},
  year={2020},
  publisher={IEEE}
}

@inproceedings{dabney2018distributional,
  title={Distributional reinforcement learning with quantile regression},
  author={Dabney, Will and Rowland, Mark and Bellemare, Marc and Munos, R{\'e}mi},
  booktitle={Proceedings of the AAAI Conference on Artificial Intelligence},
  year={2018}
}

@inproceedings{dabney2018implicit,
  title={Implicit quantile networks for distributional reinforcement learning},
  author={Dabney, Will and Ostrovski, Georg and Silver, David and Munos, R{\'e}mi},
  booktitle={International conference on machine learning},
  pages={1096--1105},
  year={2018},
  organization={PMLR}
}

@inproceedings{freirich2019distributional,
  title={Distributional multivariate policy evaluation and exploration with the bellman gan},
  author={Freirich, Dror and Shimkin, Tzahi and Meir, Ron and Tamar, Aviv},
  booktitle={International Conference on Machine Learning},
  pages={1983--1992},
  year={2019},
  organization={PMLR}
}

@article{doan2018gan,
  title={Gan q-learning},
  author={Doan, Thang and Mazoure, Bogdan and Lyle, Clare},
  journal={arXiv preprint arXiv:1805.04874},
  year={2018}
}

@inproceedings{rowland2018analysis,
  title={An analysis of categorical distributional reinforcement learning},
  author={Rowland, Mark and Bellemare, Marc and Dabney, Will and Munos, R{\'e}mi and Teh, Yee Whye},
  booktitle={International Conference on Artificial Intelligence and Statistics},
  pages={29--37},
  year={2018},
  organization={PMLR}
}

@article{rowland2023analysis,
  title={An analysis of quantile temporal-difference learning},
  author={Rowland, Mark and Munos, R{\'e}mi and Azar, Mohammad Gheshlaghi and Tang, Yunhao and Ostrovski, Georg and Harutyunyan, Anna and Tuyls, Karl and Bellemare, Marc G and Dabney, Will},
  journal={arXiv preprint arXiv:2301.04462},
  year={2023}
}

@article{wu2023distributional,
  title={Distributional Offline Policy Evaluation with Predictive Error Guarantees},
  author={Wu, Runzhe and Uehara, Masatoshi and Sun, Wen},
  journal={arXiv preprint arXiv:2302.09456},
  year={2023}
}

@inproceedings{thomas2015high,
  title={High-confidence off-policy evaluation},
  author={Thomas, Philip and Theocharous, Georgios and Ghavamzadeh, Mohammad},
  booktitle={Proceedings of the AAAI Conference on Artificial Intelligence},
  year={2015}
}

@inproceedings{jiang2016doubly,
  title={Doubly robust off-policy value evaluation for reinforcement learning},
  author={Jiang, Nan and Li, Lihong},
  booktitle={International Conference on Machine Learning},
  pages={652--661},
  year={2016},
  organization={PMLR}
}

@inproceedings{hao2021bootstrapping,
  title={Bootstrapping fitted q-evaluation for off-policy inference},
  author={Hao, Botao and Ji, Xiang and Duan, Yaqi and Lu, Hao and Szepesvari, Csaba and Wang, Mengdi},
  booktitle={International Conference on Machine Learning},
  pages={4074--4084},
  year={2021},
  organization={PMLR}
}

@article{zhu2023uncertainty,
  title={Uncertainty Quantification and Exploration for Reinforcement Learning},
  author={Zhu, Yi and Dong, Jing and Lam, Henry},
  journal={Operations Research},
  year={2023},
  publisher={INFORMS}
}

@article{shi2022statistical,
  title={Statistical inference of the value function for reinforcement learning in infinite-horizon settings},
  author={Shi, Chengchun and Zhang, Sheng and Lu, Wenbin and Song, Rui},
  journal={Journal of the Royal Statistical Society Series B: Statistical Methodology},
  volume={84},
  number={3},
  pages={765--793},
  year={2022},
  publisher={Oxford University Press}
}

@inproceedings{li2023statistical,
  title={A statistical analysis of polyak-ruppert averaged q-learning},
  author={Li, Xiang and Yang, Wenhao and Liang, Jiadong and Zhang, Zhihua and Jordan, Michael I},
  booktitle={International Conference on Artificial Intelligence and Statistics},
  pages={2207--2261},
  year={2023},
  organization={PMLR}
}

@article{li2023online,
  title={Online statistical inference for nonlinear stochastic approximation with Markovian data},
  author={Li, Xiang and Liang, Jiadong and Zhang, Zhihua},
  journal={arXiv preprint arXiv:2302.07690},
  year={2023}
}

@article{yang2022toward,
  title={Toward theoretical understandings of robust Markov decision processes: Sample complexity and asymptotics},
  author={Yang, Wenhao and Zhang, Liangyu and Zhang, Zhihua},
  journal={The Annals of Statistics},
  volume={50},
  number={6},
  pages={3223--3248},
  year={2022},
  publisher={Institute of Mathematical Statistics}
}

@article{chandak2021universal,
  title={Universal off-policy evaluation},
  author={Chandak, Yash and Niekum, Scott and da Silva, Bruno and Learned-Miller, Erik and Brunskill, Emma and Thomas, Philip S},
  journal={Advances in Neural Information Processing Systems},
  volume={34},
  pages={27475--27490},
  year={2021}
}

@inproceedings{huang2022off,
  title={Off-Policy Risk Assessment for Markov Decision Processes},
  author={Huang, Audrey and Leqi, Liu and Lipton, Zachary and Azizzadenesheli, Kamyar},
  booktitle={International Conference on Artificial Intelligence and Statistics},
  pages={5022--5050},
  year={2022},
  organization={PMLR}
}

@article{yang2019fully,
  title={Fully parameterized quantile function for distributional reinforcement learning},
  author={Yang, Derek and Zhao, Li and Lin, Zichuan and Qin, Tao and Bian, Jiang and Liu, Tie-Yan},
  journal={Advances in neural information processing systems},
  volume={32},
  year={2019}
}

@article{peng2024statistical,
  title={Statistical efficiency of distributional temporal difference learning},
  author={Peng, Yang and Zhang, Liangyu and Zhang, Zhihua},
  journal={Advances in Neural Information Processing Systems},
  volume={37},
  pages={24724--24761},
  year={2024}
}

@article{peng2025finite,
  title={A Finite Sample Analysis of Distributional TD Learning with Linear Function Approximation},
  author={Peng, Yang and Jin, Kaicheng and Zhang, Liangyu and Zhang, Zhihua},
  journal={arXiv preprint arXiv:2502.14172},
  year={2025}
}

@article{zhang2025estimation,
  title={Estimation and inference in distributional reinforcement learning},
  author={Zhang, Liangyu and Peng, Yang and Liang, Jiadong and Yang, Wenhao and Zhang, Zhihua},
  journal={The Annals of Statistics},
  volume={53},
  number={5},
  pages={1987--2011},
  year={2025},
  publisher={Institute of Mathematical Statistics}
}

@article{Qi03072025,
author = {Zhengling Qi and Chenjia Bai and Zhaoran Wang and Lan Wang},
title = {Distributional Off-Policy Evaluation in Reinforcement Learning},
journal = {Journal of the American Statistical Association},
volume = {120},
number = {551},
pages = {1517--1530},
year = {2025}
}

@article{bahadur1966note,
  title={A note on quantiles in large samples},
  author={Bahadur, R Raj},
  journal={The Annals of Mathematical Statistics},
  volume={37},
  number={3},
  pages={577--580},
  year={1966},
  publisher={JSTOR}
}

@article{kiefer1967bahadur,
  title={On Bahadur's representation of sample quantiles},
  author={Kiefer, Jack},
  journal={The Annals of Mathematical Statistics},
  volume={38},
  number={5},
  pages={1323--1342},
  year={1967},
  publisher={JSTOR}
}

@article{koenker2005quantile,
  title={Quantile regression [M]},
  author={Koenker, Roger},
  journal={Econometric Society Monographs, Cambridge University Press, Cambridge},
  year={2005}
}

@article{lahiri2009berry,
author = {S. N. Lahiri and S. Sun},
title = {{A Berry–Esseen theorem for sample quantiles under weak dependence}},
volume = {19},
journal = {The Annals of Applied Probability},
number = {1},
publisher = {Institute of Mathematical Statistics},
pages = {108 -- 126},
year = {2009}
}

@misc{chen2025smoothedsgdquantilesbahadur,
      title={Smoothed SGD for quantiles: Bahadur representation and Gaussian approximation}, 
      author={Likai Chen and Georg Keilbar and Wei Biao Wu},
      year={2025},
      journal={arXiv preprint arXiv:2505.13299}
}

@article{Portnoy2012NearlyRA,
  title={Nearly root-\$n\$ approximation for regression quantile processes},
  author={Stephen Portnoy},
  journal={Annals of Statistics},
  year={2012},
  volume={40},
  pages={1714-1736},
}

@inproceedings{rowland2023statistical,
  title={The statistical benefits of quantile temporal-difference learning for value estimation},
  author={Rowland, Mark and Tang, Yunhao and Lyle, Clare and Munos, R{\'e}mi and Bellemare, Marc G and Dabney, Will},
  booktitle={International Conference on Machine Learning},
  pages={29210--29231},
  year={2023},
  organization={PMLR}
}

@article{zhou2020noncrossing,
 author = {Zhou, Fan and Wang, Jianing and Feng, Xingdong},
 journal = {Advances in Neural Information Processing Systems},
 title = {Non-Crossing Quantile Regression for Distributional Reinforcement Learning},
 volume = {33},
 year = {2020}
}

@book{scott2015multivariate,
  title={Multivariate density estimation},
  author={Scott, David W},
  year={2015},
  publisher={Wiley Online Library}
}

@book{brezis2011functional,
  title={Functional analysis, Sobolev spaces and partial differential equations},
  author={Br{\'e}zis, Haim},
  year={2011},
  publisher={Springer}
}

@book{boucheron2013concentration,
    author = {Boucheron, Stéphane and Lugosi, Gábor and Massart, Pascal},
    title = {Concentration Inequalities: A Nonasymptotic Theory of Independence},
    publisher = {Oxford University Press},
    year = {2013},
    month = {02},
}
\bibliographystyle{abbrvnat}
\newpage

\appendix
\section{Preliminary Lemmas}\label{Appendix_prelim_lemmas}
\subsection{Proof of Proposition~\ref{prop:return_density}}\label{Appendix_prop_return_density}
\begin{proof}
    Define 
    \begin{equation*}
        G_H^\pi(s)=\sum_{t=0}^{H} \gamma^tR_t,
    \end{equation*}
    where $S_0=s$, $A_t|S_t\sim\pi(\cdot\mid S_t)$, $R_t\sim\gP_R(\cdot\mid S_t,A_t)$ and ${{S_{t+1}}\mid{(S_t,A_t)}\sim P({\cdot}\mid{S_t,A_t})}$.
    The density of $G_H^\pi(s)$ can be written as
    \begin{equation*}
        p^{G_H}_s(x)=\sum_{\text{all possible length-H path $\tau$}}q(\tau)p^R_\tau(x)
    \end{equation*}
    with $x\in[0,(1-\gamma)^{-1}]$.
    Suppose $\tau=(s,a_0,s_1,...,s_H,a_H)$, then 
    \begin{equation*}
        q(\tau)=\pi(a_0\mid s)P(s_1\mid s,a_0)\pi(a_1\mid s_1)P(s_2\mid s_1,a_1)\cdots P(s_H\mid s_{H-1},a_{H-1})\pi(a_H\mid s_H)
    \end{equation*}
    is the probability of sampling path $\tau$ and 
    \begin{equation*}
        p_\tau^R(x)= \brk{\prn{\prn{\prn{p^R_{s,a_0}(\cdot)\ast \frac{p^R_{s_1,a_1}(\cdot/\gamma)}{\gamma}}\ast \frac{p^R_{s_2,a_2}( \cdot/\gamma^2)}{\gamma^2}}\cdots}\ast \frac{p^R_{s_H,a_H}(\cdot/\gamma^H )}{\gamma^H}}(x)
    \end{equation*}
    is the density of random variable $\sum_{t=0}^H\gamma^t \prn{R_t \mid \tau}$.
    Here $p^R_{s,a}(x)$ is the density of $\gP_R(\mathrm{d}r\mid s,a)$ and $\prn{R_t \mid \tau}\sim \gP_R(\cdot \mid s_t,a_t)$.
    According to Lemma~\ref{Lemma_bounded_density_after_convolution}, $\sup_x\abs{p_\tau^R(x)}\leq C_0$.
    Thus $\sup_x\abs{p^{G_H}_s(x)}\leq C_0$, hence $G_H^\pi(s)$ has bounded density.
    This also implies 
    \begin{equation*}
        \abs{F_s^{G_H}(x)-F_s^{G_H}(y)}\leq C_0\abs{x-y},
    \end{equation*}
    namely the distribution function of $G_H^\pi(s)$ is $C_0$-Lipschitz continuous.
    Let $F^G_s(x)$ be the distribution function of $G^\pi(s)$. For any $x$, we have
    \begin{equation*}
        \begin{aligned}
        \abs{F^{G_H}_s(x)-F^G_s(x)}&=\PB(G_H^\pi(s)\leq x)-\PB(G^\pi(s)\leq x)\\
        &\leq \PB(G_H^\pi(s)\leq x)-\PB\prn{G_H^\pi(s)+\frac{\gamma^H}{1-\gamma}\leq x}\\
        &=\abs{F_s^{G_H}(x)-F_s^{G_H}\prn{x-\frac{\gamma^H}{1-\gamma}}}\\
        &\leq \frac{C_0\gamma^H}{1-\gamma}.
        \end{aligned}
    \end{equation*}
    The inequality above implies that $\|F^{G_H}_s-F^G_s\|_\infty\to0$ as $H\to\infty$. To prove the existence and properties of the density of the return, we establish the equicontinuity of $p_s^{G_H}$. In fact, we only need to notice that $p^R_{s,a_0}(\cdot)\ast [p^R_{s_1,a_1}(\cdot/\gamma)/\gamma]$ is uniform continuous on $\RB$ for any $(a_0,s_1,a_1)$ and hence has a modulus of continuity $\omega_{s,a_0,s_1,a_1}(\varrho)<\infty$. Therefore, by Lemma~\ref{Lemma_modulus_of_continuity_after_convolution}, we know that the continuity of modulus of $p_s^{G_H}$ satisfies
    \begin{equation*}
        \omega_s^H(\varrho)\leq\max_{(a_0,s_1,a_1)\in\gA\times\gS\times\gA}\omega_{s,a_0,s_1,a_1}(\varrho)<+\infty. 
    \end{equation*}
    Applying Arzela--Ascoli theorem, we deduce that $\eta^\pi(s)$ has continuous density $p_{\eta^\pi(s)}$ upper bounded by $C_0$ and $p_{\eta^\pi(s)}(0)=p_{\eta^\pi(s)}((1-\gamma)^{-1})=0$. 

    To prove that $p_{\eta^\pi(s)}>0$ on $(0,(1-\gamma)^{-1})$, we use the distributional Bellman equation. Suppose that there exists $x\in(0,(1-\gamma)^{-1})$ and $s\in\gS$ such that $p_{\eta^\pi(s)}(x)=0$, then by distributional Bellman equation,
    \begin{equation*}
        \sum_{a\in\gA,s^\prime\in\gS}\pi(a\mid s)P(s^\prime\mid s,a)\int p_{s,a}(x-\gamma y)p_{\eta^\pi(s^\prime)}(y)\rd y=0. 
    \end{equation*}
    By Assumption~\ref{assump:density} and the continuity of $p_{\eta^\pi(s^\prime)}$, we know that there exists $s^\prime$ such that $p_{\eta^\pi(s^\prime)}=0$ on $[(x-1)/\gamma,x/\gamma]\cap[0,(1-\gamma)^{-1}]$. Denote $b_l(x)=(x-l)/\gamma$, $l=0,1$ and use the argument repeatedly, then we deduce that for every $N$, there exists $s_N\in\gS$ such that
    \begin{equation*}
        p_{s_N}(y)=0,\ y\in[b_1^{(N)}(x),b_0^{(N)}(x)]\cap\brk{0,\frac{1}{1-\gamma}}, 
    \end{equation*}
    where $b_l^{(N)}$ is the $N$-times composition of $b_l$. However, since $x\in(0,(1-\gamma)^{-1})$, $b_0^{(N)}(x)\to+\infty$ and $b_1^{(N)}(x)\to-\infty$ as $N\to+\infty$. Therefore, for sufficiently large $N$, $[0,(1-\gamma)^{-1}]\subseteq[b_1^{(N)}(x),b_0^{(N)}(x)]$. This implies that there exists $s$ such that $p_{\eta^\pi(s)}\equiv0$, which is a contradiction. Therefore, Proposition~\ref{prop:return_density} holds. 
\end{proof}

\subsection{Other Preliminary Lemmas}
In this subsection, we prove several lemmas about $\gT^\pi\bm\eta_m$, which will be useful in the following proofs. First, we prove that the density of $\gT^\pi\bm\eta_m$ converges uniformly. 
\begin{lemma}\label{lem:pdf_uniform_convergence}
    For every $s\in\gS$, $\lim_{m\to+\infty}\norm{p_{(\gT^\pi\bm\eta_m)(s)}-p_{\eta^\pi(s)}}_\infty=0$. 
\end{lemma}

\begin{proof}
    We have
    \begin{align*}
        &\left|p_{(\gT^\pi\bm\eta_m)(s)}(x)-p_{\eta^\pi(s)}(x)\right|\\
        =&\left|\sum_{a\in\gA,s^\prime\in\gS}\pi(a\mid s)P(s^\prime\mid s,a)\left[\frac{1}{m}\sum_{j=1}^mp_{s,a}\left(x-\gamma\theta_m(s^\prime,j)\right)-\int p_{s,a}(x-\gamma y)p_{\eta^\pi(s^\prime)}(y)\rd y\right]\right|\\
        \leq&\sum_{a\in\gA,s^\prime\in\gS}\pi(a\mid s)P(s^\prime\mid s,a)\left[\frac{1}{m}\sum_{j=1}^m\left|p_{s,a}\left(x-\gamma F^{-1}_{(\gT^\pi\bm\eta_m)(s^\prime)}(\tau_j)\right)-p_{s,a}\left(x-\gamma F^{-1}_{\eta^\pi(s^\prime)}(\tau_j)\right)\right|\right]\\
        &+\sum_{a\in\gA,s^\prime\in\gS}\pi(a\mid s)P(s^\prime\mid s,a)\left|\frac{1}{m}\sum_{j=1}^mp_{s,a}\left(x-\gamma F^{-1}_{\eta^\pi(s^\prime)}(\tau_j)\right)-\int_0^1p_{s,a}\left(x-\gamma F^{-1}_{\eta^\pi(s^\prime)}(\tau)\right)\rd\tau\right|\\
        \coloneq &I_1+I_2
    \end{align*}
    To bound $I_2$, we notice that as $m\to+\infty$, 
    \begin{align*}
        &\left|\frac{1}{m}\sum_{j=1}^mp_{s,a}\left(x-\gamma F^{-1}_{\eta^\pi(s^\prime)}(\tau_j)\right)-\int_0^1p_{s,a}\left(x-\gamma F^{-1}_{\eta^\pi(s^\prime)}(\tau)\right)\rd\tau\right|\\
        \leq&\sum_{j=1}^m\int_{\frac{j-1}{m}}^{\frac{j}{m}}\left|p_{s,a}\left(x-\gamma F^{-1}_{\eta^\pi(s^\prime)}(\tau)\right)-p_{s,a}\left(x-\gamma F^{-1}_{\eta^\pi(s^\prime)}(\tau_j)\right)\right|\rd\tau. 
    \end{align*}
    For $j$ such that $0\in[x-\gamma F^{-1}_{\eta^\pi(s^\prime)}((j-1)/m),x-\gamma F^{-1}_{\eta^\pi(s^\prime)}(j/m)]$ or $1\in[x-\gamma F^{-1}_{\eta^\pi(s^\prime)}((j-1)/m),x-\gamma F^{-1}_{\eta^\pi(s^\prime)}(j/m)]$, we have
    \begin{equation*}
        \int_{\frac{j-1}{m}}^{\frac{j}{m}}\left|p_{s,a}\left(x-\gamma F^{-1}_{\eta^\pi(s^\prime)}(\tau)\right)-p_{s,a}\left(x-\gamma F^{-1}_{\eta^\pi(s^\prime)}(\tau_j)\right)\right|\rd\tau\leq\frac{2C_0}{m}. 
    \end{equation*}
    Therefore, as $m\to\infty$, 
    \begin{align*}
        &\left|\frac{1}{m}\sum_{j=1}^mp_{s,a}\left(x-\gamma F^{-1}_{\eta^\pi(s^\prime)}(\tau_j)\right)-\int_0^1p_{s,a}\left(x-\gamma F^{-1}_{\eta^\pi(s^\prime)}(\tau)\right)\rd\tau\right|\\
        \leq& \frac{4C_0}{m}+\sup_{\substack{y,z\in[0,1],\\|y-z|\leq\omega_{s^\prime}((2m)^{-1})}}|p_{s,a}(y)-p_{s,a}(z)|\to0.
    \end{align*}
    To bound $I_1$, for every $s^\prime\in\gS$ and $x$, define
    \begin{align*}
        S_0(s^\prime,x)=&\brc{j\in[m]\colon\left(x-\gamma F^{-1}_{(\gT^\pi\bm\eta_m)(s^\prime)}(\tau_j)\right)\left(x-\gamma F^{-1}_{\eta^\pi(s^\prime)}(\tau_j)\right)\leq0},\\
        S_1(s^\prime,x)=&\brc{j\in[m]\colon\left(x-\gamma F^{-1}_{(\gT^\pi\bm\eta_m)(s^\prime)}(\tau_j)-1\right)\left(x-\gamma F^{-1}_{\eta^\pi(s^\prime)}(\tau_j)-1\right)\leq0},\\
        S_2(s^\prime,x)=&[m]-S_0(s^\prime,x)\cup S_1(s^\prime,x),
    \end{align*}
    then
    \begin{align*}
        &\frac{1}{m}\sum_{j=1}^m\left|p_{s,a}\left(x-\gamma F^{-1}_{(\gT^\pi\bm\eta_m)(s^\prime)}(\tau_j)\right)-p_{s,a}\left(x-\gamma F^{-1}_{\eta^\pi(s^\prime)}(\tau_j)\right)\right|\\
        \leq&\frac{2C_0}{m}(|S_0(s^\prime,x)|+|S_1(s^\prime,x)|)+\sup_{\substack{y,z\in[0,1],\\|y-z|\leq\Bar{W}_\infty(\gT^\pi\bm\eta_m,\bm\eta^\pi)}}|p_{s,a}(y)-p_{s,a}(z)|.
    \end{align*}
    $j\in S_0(s^\prime,x)$ is equivalent to
    \begin{align*}
        \left[\tau_j-F_{(\gT^\pi\bm\eta_m)(s^\prime)}\left(\frac{x}{\gamma}\right)\right]\left[\tau_j-F_{\eta^\pi(s^\prime)}\left(\frac{x}{\gamma}\right)\right]\leq 0, 
    \end{align*}
    hence
    \begin{equation*}
        \frac{1}{m}|S_0(s^\prime,x)|\leq\left|F_{(\gT^\pi\bm\eta_m)(s^\prime)}\prn{\frac{x}{\gamma}}-F_{\eta^\pi(s^\prime)}\prn{\frac{x}{\gamma}}\right|\leq\norm{F_{(\gT^\pi\bm\eta_m)(s^\prime)}-F_{\eta^\pi(s^\prime)}}_\infty. 
    \end{equation*}
    Similarly, 
    \begin{equation*}
        \frac{1}{m}|S_1(s^\prime,x)|\leq\norm{F_{(\gT^\pi\bm\eta_m)(s^\prime)}-F_{\eta^\pi(s^\prime)}}_\infty.
    \end{equation*}
    However, $\Bar{W}_\infty(\gT^\pi\bm\eta_m,\bm\eta)\leq\gamma\Bar{W}_\infty(\bm\eta_m,\bm\eta)\to 0$ as $m\to+\infty$ by Theorem~\ref{thm:approx_error}, which implies $(\gT^\pi\bm\eta_m)(s^\prime)$ converges to $\eta^\pi(s^\prime)$ weakly. Since $F_{\eta^\pi(s^\prime)}$ is continuous, $\|F_{(\gT^\pi\bm\eta_m)(s^\prime)}-F_{\eta^\pi(s^\prime)}\|_\infty\to 0$ by Lemma~\ref{lem:polya_thm}. Therefore we have $I_1\to0$ as $m\to\infty$ and the conclusion follows. 
\end{proof}

Next, we prove that the density at the quantiles are strictly positive.
\begin{lemma}\label{lem:quantile_density_positive}
    For every $m,s,i$, $p_{(\gT^\pi\bm\eta_m)(s)}(\theta_m(s,i))>0$. 
\end{lemma}
\begin{proof}
    We assume there exists $m,s,i$ such that
    \begin{equation*}
        p_{(\gT^\pi\bm\eta_m)(s)}(\theta_m(s,i))=\sum_{a\in\gA,s^\prime\in\gS}\pi(a\mid s)P(s^\prime\mid s,a)\brk{\frac{1}{m}\sum_{j=1}^mp_{s,a}(\theta_m(s,i)-\gamma\theta_m(s^\prime,j))}=0. 
    \end{equation*}
    We can assume $\pi(a\mid s)P(s^\prime\mid s,a)>0$ for all $a,s^\prime$ without loss of generality. By Assumption~\ref{assump:density} and~\ref{Assumption_no_boundary_quantile}, $\theta_m(s,i)-\gamma\theta_m(s^\prime,j)\in\RB\setminus[0,1]$ for every $(s^\prime,j)$. Therefore, there exists sufficiently small $\delta_0>0$ such that $\theta_m(s,i)-\gamma\theta_m(s^\prime,j)\in\RB\setminus[-2\delta_0,1+2\delta_0]$ and hence $p_{(\gT^\pi\bm\eta_m)(s)}(x)=0$ for $x\in[\theta_m(s,i)-\delta_0,\theta_m(s,i)]$. Therefore, $F_{(\gT^\pi\bm\eta_m)(s)}(\theta_m(s,i)-\delta_0)=F_{(\gT^\pi\bm\eta_m)(s)}(\theta_m(s,i))=\tau_i$. However, this contradicts the definition of $\theta_m(s,i)$ as $\theta_m(s,i)=F_{(\gT^\pi\bm\eta_m)(s)}^{-1}(\tau_i)$. 
\end{proof}

Finally we prove the following lemma, which provides a lower bound of the density function of $\gT^\pi\bm\eta_m$ at the boundary and is essential in the following analysis. 
\begin{lemma}\label{lem:density_lower_bound}
    Suppose either Assumption~\ref{Assumption_reward_lower_bounded} or~\ref{Assumption_reward_bounded_density} holds. For every $m,s$ and $x\in[F^{-1}_{(\gT^\pi\bm\eta_m)(s)}(0),F^{-1}_{(\gT^\pi\bm\eta_m)(s)}(0)+\kappa]$, we have
    \begin{equation}\label{eq:density_left_boundary_lower_bound}
        F_{(\gT^\pi\bm\eta_m)(s)}(x)\lesssim xp_{(\gT^\pi\bm\eta_m)(s)}(x). 
    \end{equation}
    Similarly, for $x\in[F^{-1}_{(\gT^\pi\bm\eta_m)(s)}(1)-\kappa,F^{-1}_{(\gT^\pi\bm\eta_m)(s)}(1)]$, we have
    \begin{equation}\label{eq:density_right_boundary_lower_bound}
        1-F_{(\gT^\pi\bm\eta_m)(s)}(x)\lesssim \prn{F^{-1}_{(\gT^\pi\bm\eta_m)(s)}(1)-x}p_{(\gT^\pi\bm\eta_m)(s)}(x)
    \end{equation}
\end{lemma}
\begin{proof}
    We only prove Equation~\eqref{eq:density_left_boundary_lower_bound} and~\eqref{eq:density_right_boundary_lower_bound} can be proved similarly. 

    Denote
    \begin{equation*}
        S(x)=\brc{(s^\prime,j):x-\gamma\theta_m(s^\prime,j)\in[0,1]}. 
    \end{equation*}
    Then for every $x\in[F^{-1}_{(\gT^\pi\bm\eta_m)(s)}(0),F^{-1}_{(\gT^\pi\bm\eta_m)(s)}(0)+\kappa]$ and $(s^\prime,j)\in S(x)$, $x-\gamma\theta_m(s^\prime,j)\in[0,\kappa]$. Suppose Assumption~\ref{Assumption_reward_lower_bounded} holds and we have
    \begin{align*}
        F_{(\gT^\pi\bm\eta_m)(s)}(x)=&\sum_{a\in\gA,s^\prime\in\gS}\pi(a\mid s)\brk{\sum_{(s^\prime,j)\in S(x)}\frac{P(s^\prime\mid s,a)}{m}\int_{0}^{x}p_{s,a}(z-\gamma\theta_m(s^\prime,j))\rd z}\\
        \leq&\sum_{a\in\gA,s^\prime\in\gS}\pi(a\mid s)\brk{\sum_{(s^\prime,j)\in S(x)}\frac{P(s^\prime\mid s,a)}{m}Cx}\\
        \leq&\frac{Cx}{c}\sum_{a\in\gA,s^\prime\in\gS}\pi(a\mid s)\brk{\sum_{(s^\prime,j)\in S(x)}\frac{P(s^\prime\mid s,a)}{m}p_{s,a}(x-\gamma\theta_m(s^\prime,j))}\\
        =&\frac{C}{c}xp_{(\gT^\pi\bm\eta_m)(s)}(x). 
    \end{align*}

    Suppose Assumption~\ref{Assumption_reward_bounded_density} holds. By the monotoncity of $p_{s,a}$, we have 
    \begin{equation*}
        \int_{0}^{x}p_{s,a}(z-\gamma\theta_m(s^\prime,j))\rd z\leq xp_{s,a}(x-\gamma\theta_m(s^\prime,j)). 
    \end{equation*}
    Therefore, 
    \begin{align*}
        F_{(\gT^\pi\bm\eta_m)(s)}(x)=&\sum_{a\in\gA,s^\prime\in\gS}\pi(a\mid s)\brk{\sum_{(s^\prime,j)\in S(x)}\frac{P(s^\prime\mid s,a)}{m}\int_{0}^{x}p_{s,a}(z-\gamma\theta_m(s^\prime,j))\rd z}\\
        \leq&\sum_{a\in\gA,s^\prime\in\gS}\pi(a\mid s)\brk{\sum_{(s^\prime,j)\in S(x)}\frac{P(s^\prime\mid s,a)}{m}xp_{s,a}(x-\gamma\theta_m(s^\prime,j))}\\
        =&xp_{(\gT^\pi\bm\eta_m)(s)}(x). \qedhere
    \end{align*}
\end{proof}

\section{Omitted Proofs of Non-asymptotic Analysis}\label{Appendix_omit_proof_nonasymp}
We need a technical lemma first.
\begin{lemma}\label{lem:u_lower_bound}
    For every $m\in\NB$ and $s\in\gS$, we have
    \begin{align*}
        \inf_{i\in[m],0<|u|\leq(1-\gamma)^{-1}}\left|\frac{F_{(\gT^\pi\bm\eta_m)(s)}\left(F_{(\gT^\pi\bm\eta_m)(s)}^{-1}(\tau_i)+u\right)-\tau_i}{\tau_i(1-\tau_i)u}\right|\geq c(\gM)>0, 
    \end{align*}
    where $c(\gM)$ is a constant that does not depend on $m$. 
\end{lemma}

\begin{proof}[Proof of Lemma~\ref{lem:u_lower_bound}]
    We denote
    \begin{equation*}
        V_{m,s}(\tau,u)=\frac{F_{(\gT^\pi\bm\eta_m)(s)}\left(F_{(\gT^\pi\bm\eta_m)(s)}^{-1}(\tau)+u\right)-\tau}{\tau(1-\tau)u}=\frac{\int_0^1p_{(\gT^\pi\bm\eta_m)(s)}\left(F^{-1}_{(\gT^\pi\bm\eta_m)(s)}(\tau)+tu\right)\rd t}{\tau(1-\tau)}. 
    \end{equation*}
    Define 
    \begin{equation*}
        \kappa_0=\frac{1}{2}\min_{s\in\gS}\brc{\min\brc{F_{\eta^\pi(s)}\prn{\frac{\kappa}{2}},1-F_{\eta^\pi(s)}\prn{(1-\gamma)^{-1}-\frac{\kappa}{2}}}}.
    \end{equation*}
    By Lemma~\ref{lem:pdf_uniform_convergence}, $V_{m,s}$ uniformly converges to
    \begin{equation*}
        V_{s}(\tau,u)=\frac{F_{(\gT^\pi\bm\eta_m)(s)}\left(F_{(\gT^\pi\bm\eta_m)(s)}^{-1}(\tau)+u\right)-\tau}{\tau(1-\tau)u}=\frac{\int_0^1p_{\eta^\pi(s)}\left(F^{-1}_{\eta^\pi(s)}(\tau)+tu\right)\rd t}{\tau(1-\tau)}
    \end{equation*}
    on $S_\kappa=[\kappa_0,1-\kappa_0]\times[-(1-\gamma)^{-1},(1-\gamma)^{-1}]$. Therefore, there exists a positive integer $M$ which only depends on $\gM$ such that for every $m>M$ and $s\in\gS$, 
    \begin{equation*}
        \inf_{(\tau,u)\in S}V_{m,s}(\tau,u)\geq\frac{1}{2}\inf_{(\tau,u)\in S}V_{s}(\tau,u)
    \end{equation*}
    and
    \begin{equation*}
        \sup_{s\in\gS}\norm{F^{-1}_{(\gT^\pi\bm\eta_m)(s)}-F^{-1}_{\eta^\pi(s)}}\leq\frac{\kappa}{2}. 
    \end{equation*}
    Fix $s\in\gS$ and for $(\tau,u)\in S_\kappa$ satisfying
    \begin{equation*}
        F^{-1}_{\eta^\pi(s)}\prn{\max\brc{\frac{\tau}{2},2\tau-1}}\leq F^{-1}_{\eta^\pi(s)}(\tau)+u\leq F^{-1}_{\eta^\pi(s)}\prn{\min\brc{2\tau,\frac{1+\tau}{2}}}, 
    \end{equation*}
    we have
    \begin{equation*}
        V_{s}(\tau,u)\geq\frac{F_{(\gT^\pi\bm\eta_m)(s)}\left(F_{(\gT^\pi\bm\eta_m)(s)}^{-1}(\tau)+u\right)-\tau}{u}\geq\inf_{\tau\in\brk{\kappa_0/2,1-\kappa_0/2}}p_{\eta^\pi(s)}\prn{F^{-1}_{\eta^\pi(s)}(\tau)}. 
    \end{equation*}
    Otherwise, we have
    \begin{equation*}
        V_s(\tau,u)\geq\frac{\max\brc{\tau,1-\tau}}{2\tau(1-\tau)u}\geq\frac{1-\gamma}{2}. 
    \end{equation*}
    Now we consider the case where $\tau_i\in[0,\kappa_0]$ and $m>M$, and the case where $\tau_i\in[1-\kappa_0,1]$ can be dealt with similarly. For $\tau_i\in[0,\kappa_0]$, when
    \begin{equation*}
        F^{-1}_{(\gT^\pi\bm\eta_m)(s)}(\tau_i)+u\geq F^{-1}_{(\gT^\pi\bm\eta_m)(s)}(2\tau_i),\ \text{or}\ F^{-1}_{(\gT^\pi\bm\eta_m)(s)}(\tau_i)+u\leq F^{-1}_{(\gT^\pi\bm\eta_m)(s)}(\tau_i/2), 
    \end{equation*}
    we have
    \begin{equation*}
        \left|\frac{F_{(\gT^\pi\bm\eta_m)(s)}\left(F_{(\gT^\pi\bm\eta_m)(s)}^{-1}(\tau_i)+u\right)-\tau_i}{\tau_i(1-\tau_i)u}\right|\geq\frac{1}{2(1-\tau_i)u}\geq \frac{1-\gamma}{2}. 
    \end{equation*}
    Otherwise, we have $F^{-1}_{(\gT^\pi\bm\eta_m)(s)}(2\tau_i)\leq\kappa$ and 
    \begin{equation*}
        \frac{F_{(\gT^\pi\bm\eta_m)(s)}\left(F_{(\gT^\pi\bm\eta_m)(s)}^{-1}(\tau_i)+u\right)-\tau_i}{u}=p_{(\gT^\pi\bm\eta_m)(s)}(\xi)
    \end{equation*}
    where $F^{-1}_{(\gT^\pi\eta_m)(s)}(\tau_i/2)\leq\xi \leq F^{-1}_{(\gT^\pi\eta_m)(s)}(2\tau_i)$. Therefore, 
    \begin{equation*}
        \left|\frac{F_{(\gT^\pi\bm\eta_m)(s)}\left(F_{(\gT^\pi\bm\eta_m)(s)}^{-1}(\tau_i)+u\right)-\tau_i}{\tau_i(1-\tau_i)u}\right|\gtrsim \frac{\tau_i/2}{\tau_i(1-\tau_i)\xi}\geq\frac{1-\gamma}{2}. 
    \end{equation*}
    by Lemma~\ref{lem:density_lower_bound}. Therefore, for every $m>M$ and $s\in\gS$, 
    \begin{equation*}
        \inf_{i\in[m],0<|u|\leq(1-\gamma)^{-1}}V_{m,s}(\tau_i,u)\geq\frac{1}{4}\min\brc{1-\gamma,\inf_{s\in\gS,\tau\in\brk{\kappa_0/2,1-\kappa_0/2}}p_{\eta^\pi(s)}\prn{F^{-1}_{\eta^\pi(s)}(\tau)}}
    \end{equation*}
    For $m\leq M$ and $\tau_i$, $W_{m,s,i}(u)\coloneq V_{m,s}(\tau_i,u)$ is continuous and positive by Assumption~\ref{Assumption_no_boundary_quantile}, therefore has a positive lower bound. 
\end{proof}

\begin{proof}[Proof of Lemma~\ref{lem:quantile_concentration}]
    For $u>0$, denote
    \begin{equation*}
        \Delta_{s,i}^+(u)=F_{(\gT^\pi\bm\eta_m)(s)}(\theta_m(s,i)+u)-\tau_i,\quad\Delta_{s,i}^-(u)=\tau_i-F_{(\gT^\pi\bm\eta_m)(s)}(\theta_m(s,i)-u). 
    \end{equation*}
    Fix $s\in\gS$ and $i\in[m]$ and we have
    \begin{align*}
        &\left\{F^{-1}_{(\what{\gT}^\pi_n\bm\eta_m)(s)}(\tau_i)-F^{-1}_{(\gT^\pi\bm\eta_m)(s)}(\tau_i)\geq u\right\}\\
        =&\left\{F_{(\what{\gT}^\pi_n\bm\eta_m)(s)}(\theta_m(s,i)+u)-F_{(\gT^\pi\bm\eta_m)(s)}(\theta_m(s,i)+u)\leq-\Delta_{s,i}^+(u)\right\}\\
        \subseteq&\left\{\sum_{a,s^\prime,j}\frac{\pi(a\mid s)P(s^\prime\mid s,a)}{m}(\what{F}^{(n)}_{s,a}-F_{s,a})(\theta_m(s,i)+u-\gamma\theta_m(s^\prime,j))\leq-\frac{1}{3}\Delta_{s,i}^+(u)\right\}\\
        &\bigcup\left\{\sum_{a,s^\prime,j}\frac{\pi(a\mid s)}{m}[\what{P}^{(n)}(s^\prime\mid s,a)-P(s^\prime\mid s,a)]F_{s,a}(\theta_m(s,i)+u-\gamma\theta_m(s^\prime,j))\leq-\frac{1}{3}\Delta_{s,i}^+(u)\right\}\\
        &\bigcup\left\{Z_{s,i}^{(n)}\leq-\frac{1}{3}\Delta_{s,i}^+(u)\right\}\\
        \coloneq &A^{(1)}_{s,i}\cup A^{(2)}_{s,i}\cup A^{(3)}_{s,i}, 
    \end{align*}
    where
    \begin{equation*}
        Z_{s,i}^{(n)}=\sum_{a,s^\prime,j}\frac{\pi(a\mid s)}{m}[\what{P}^{(n)}(s^\prime\mid s,a)-P(s^\prime\mid s,a)](\what{F}^{(n)}_{s,a}-F_{s,a})(\theta_m(s,i)+u-\gamma\theta_m(s^\prime,j)). 
    \end{equation*}
    Applying Bernstein's inequality (Lemma~\ref{Lemma_Bernstein_Inequality}) and we have 
    \begin{equation*}
        \PB\left(A^{(k)}_{s,i}\right)\leq \exp\left[-\frac{Cn(\Delta_{s,i}^+(u))^2}{\tau_i(1-\tau_i)+\Delta_{s,i}^+(u)}\right]. 
    \end{equation*}
    for $k=1,2$. 
    To bound $\PB(A_{s,i}^{(3)})$, we bound the MGF of the random variable $Z_{s,i}^{(n)}$. Denote
    \begin{equation*}
        w_{s,i,a,s^\prime}=\frac{1}{m}\sum_{j=1}^m (\what{F}^{(n)}_{s,a}-F_{s,a})(\theta_m(s,i)+u-\gamma\theta_m(s^\prime,j))
    \end{equation*}
    and for $\lambda\in\RB$, we have
    \begin{equation*}
        \EB\exp(\lambda Z_{s,i}^{(n)})\leq\sum_{a\in\gA}\pi(a\mid s)\EB\exp\left(\lambda\sum_{s^\prime\in\gS}w_{s,i,a,s^\prime}\left[\what{P}^{(n)}(s^\prime\mid s,a)-P(s^\prime\mid s,a)\right]\right). 
    \end{equation*}
    Condition on $\what{F}^{(n)}_{s,a}$ and we have
    \begin{align*}
        &\EB_{S^{\prime(s,a)}}\exp\left(\lambda\sum_{s^\prime\in\gS}w_{s,i,a,s^\prime}\left[\what{P}^{(n)}(s^\prime\mid s,a)-P(s^\prime\mid s,a)\right]\right)\\
        =&\EB_{S^{\prime(s,a)}}\exp\left(\frac{\lambda}{n}\sum_{i=1}^n\sum_{s^\prime\in\gS}w_{s,i,a,s^\prime}\left[\ind\{S^{\prime(s,a)}_i=s^\prime\}-P(s^\prime\mid s,a)\right]\right)\\
        =&\left[\EB_{S^{\prime(s,a)}}\exp\left(\frac{\lambda}{n}\sum_{i=1}^n\sum_{s^\prime\in\gS}w_{s,i,a,s^\prime}\left[\ind\{S^{\prime(s,a)}_i=s^\prime\}-P(s^\prime\mid s,a)\right]\right)\right]^n
    \end{align*}
    and
    \begin{equation*}
        \left|\sum_{s^\prime\in\gS}w_{s,i,a,s^\prime}\left[\ind\{S^{\prime(s,a)}_i=s^\prime\}-P(s^\prime\mid s,a)\right]\right|\leq 2\norm{\what{F}^{(n)}_{s,a}-F_{s,a}}_\infty. 
    \end{equation*}
    Therefore, 
    \begin{equation*}
        \EB\exp(\lambda Z_{s,i}^{(n)})\leq\EB\exp\left(\frac{C\lambda^2}{n}\norm{\what{F}^{(n)}_{s,a}-F_{s,a}}_\infty^2\right)
    \end{equation*}
    by Hoeffding's lemma (Lemma~\ref{Lemma_Hoeffding_lemma}). By DKW inequality (Lemma~\ref{Lemma_DKW_Inequality}), $\|\what{F}^{(n)}_{s,a}-F_{s,a}\|_\infty$ is $\frac{C}{\sqrt{n}}$-sub-gaussian. Therefore, 
    \begin{align*}
        \EB\exp(\lambda Z_{s,i}^{(n)})\leq\left(1-\frac{C\lambda^2}{n^2}\right)^{-\frac{1}{2}}\leq\exp\left(\frac{C\lambda^2}{n^2}\right)
    \end{align*}
    for $|\lambda|\lesssim\frac{1}{n}$, which implies that $Z_{s,i}^{(n)}$ is $\frac{C}{n}$-sub-exponential. We conclude that
    \begin{equation*}
        \PB\left(A^{(3)}_{s,i}\right)\leq \exp\left(-cn\Delta_{s,i}^+(u)\right). 
    \end{equation*}
    Therefore, 
    \begin{equation*}
        \PB\left[F^{-1}_{(\what{\gT}^\pi_n\bm\eta_m)(s)}(\tau_i)-F^{-1}_{(\gT^\pi\bm\eta_m)(s)}(\tau_i)\geq u\right]\leq 3\exp\left[-\frac{Cn(\Delta_{s,i}^+(u))^2}{\tau_i(1-\tau_i)+\Delta_{s,i}^+(u)}\right]. 
    \end{equation*}
    A similar argument yields that
    \begin{equation*}
        \PB\left[F^{-1}_{(\what{\gT}^\pi_n\bm\eta_m)(s)}(\tau_i)-F^{-1}_{(\gT^\pi\bm\eta_m)(s)}(\tau_i)\leq-u\right]\leq 3\exp\left[-\frac{Cn(\Delta_{s,i}^-(u))^2}{\tau_i(1-\tau_i)+\Delta_{s,i}^-(u)}\right] 
    \end{equation*}
    and hence
    \begin{equation*}
        \PB\left[\left|F^{-1}_{(\what{\gT}^\pi_n\bm\eta_m)(s)}(\tau_i)-F^{-1}_{(\gT^\pi\bm\eta_m)(s)}(\tau_i)\right|\geq u\right]\leq 6\exp\left[-\frac{Cn\left(\Delta_{s,i}^+(u)\wedge\Delta_{s,i}^-(u)\right)^2}{\tau_i(1-\tau_i)+\Delta_{s,i}^+(u)\wedge\Delta_{s,i}^-(u)}\right]. 
    \end{equation*}
    Lemma~\ref{lem:u_lower_bound} implies that $\min\{\Delta_{s,i}^+(u),\Delta_{s,i}^-(u)\}\geq c\tau_i(1-\tau_i)u$ and we have
    \begin{equation*}
        \PB\left[\left|F^{-1}_{(\what{\gT}^\pi_n\bm\eta_m)(s)}(\tau_i)-F^{-1}_{(\gT^\pi\bm\eta_m)(s)}(\tau_i)\right|\geq u\right]\leq 6\exp\left[-\frac{Cc(\gM)^2n\tau_i(1-\tau_i)u^2}{1+c(\gM)u}\right]
    \end{equation*}
    where $C$ is a universal constant and $c(\gM)$ is the constant defined in Lemma~\ref{lem:u_lower_bound}. Let the right hand side equal $\delta$ and we derive that with probability at least $1-\delta$, 
    \begin{equation*}
        \left|F^{-1}_{(\what{\gT}^\pi_n\bm\eta_m)(s)}(\tau_i)-F^{-1}_{(\gT^\pi\bm\eta_m)(s)}(\tau_i)\right|\leq \frac{C}{c(\gM)}\brk{\sqrt{\frac{m\log(6/\delta)}{n}}+\frac{m\log(6/\delta)}{n}},  
    \end{equation*}
    where $C$ is a universal constant. Finally, we notice that from the proof of Lemma~\ref{lem:u_lower_bound}, $c(\gM)\lesssim1-\gamma$. Therefore, if 
    \begin{equation*}
        C\sqrt{\frac{m\log(6/\delta)}{n}}\gtrsim 1, 
    \end{equation*}
    we have
    \begin{equation*}
        \left|F^{-1}_{(\what{\gT}^\pi_n\bm\eta_m)(s)}(\tau_i)-F^{-1}_{(\gT^\pi\bm\eta_m)(s)}(\tau_i)\right|\leq \frac{C}{c(\gM)}\sqrt{\frac{m\log(6/\delta)}{n}}
    \end{equation*}
    always holds. Otherwise, 
    \begin{equation*}
        \left|F^{-1}_{(\what{\gT}^\pi_n\bm\eta_m)(s)}(\tau_i)-F^{-1}_{(\gT^\pi\bm\eta_m)(s)}(\tau_i)\right|\leq \frac{C(1+C)}{c(\gM)}\sqrt{\frac{m\log(6/\delta)}{n}}. 
    \end{equation*}
    Therefore Lemma~\ref{lem:quantile_concentration} holds. 
\end{proof}

\section{Omitted Proofs of Asymptotic Analysis}\label{Appendix_omit_proof_asymp}
\subsection{Proof of Lemma~\ref{lem:consistency_weaker_condition}}\label{Appendix_lem_consistency_weaker_condition}
\begin{proof}
    Since we only need to prove the conclusion for every fixed integer $m$, the constants in this proof may depend on $m$. We know that
    \begin{align*}
        \norm{\btheta_m^{(n)}-\btheta_m}_\infty=\Bar{W}_\infty(\bm\eta_m^{(n)},\bm\eta_m)&\leq\frac{1}{1-\gamma}\Bar{W}_\infty(\bPi_m\what{\gT}^\pi_n\bm\eta_m,\bPi_m\gT^\pi\bm\eta_m)\\
        &=\frac{1}{1-\gamma}\sup_{s\in\gS,i\in[m]}\left|F^{-1}_{(\what{\gT}^\pi_n\bm\eta_m)(s)}(\tau_i)-F^{-1}_{(\gT^\pi\bm\eta_m)(s)}(\tau_i)\right|. 
    \end{align*}
    Using the same notation and argument as in Appendix~\ref{Appendix_omit_proof_nonasymp}, we deduce that for $u>0$,
    \begin{equation*}
        \PB\brk{\left|F^{-1}_{(\what{\gT}^\pi_n\bm\eta_m)(s)}(\tau_i)-F^{-1}_{(\gT^\pi\bm\eta_m)(s)}(\tau_i)\right|>u}\lesssim\exp\brk{-Cn\left(\Delta_{s,i}^+(u)\wedge\Delta_{s,i}^-(u)\right)^2}. 
    \end{equation*}
    According to Lemma~\ref{lem:quantile_density_positive}, for every $(s,i)$, there exists sufficiently small $u(s,i)>0$ such that
    \begin{equation*}
        p_{(\gT^\pi\bm\eta_m)(s)}(x)\geq\frac{1}{2}p_{(\gT^\pi\bm\eta_m)(s)}(\theta_m(s,i)),\ x\in[\theta_m(s,i)-u(s,i),\theta_m(s,i)+u(s,i)]. 
    \end{equation*}
    Therefore, 
    \begin{equation*}
        \PB\brk{\left|F^{-1}_{(\what{\gT}^\pi_n\bm\eta_m)(s)}(\tau_i)-F^{-1}_{(\gT^\pi\bm\eta_m)(s)}(\tau_i)\right|>u}\lesssim\exp\prn{-Cnu^2},\ u\leq u(s,i), 
    \end{equation*}
    and
    \begin{align*}
        \EB\brk{\sqrt{n}\left|F^{-1}_{(\what{\gT}^\pi_n\bm\eta_m)(s)}(\tau_i)-F^{-1}_{(\gT^\pi\bm\eta_m)(s)}(\tau_i)\right|}&\lesssim \sqrt{n}\brk{\exp\prn{-Cnu(s,i)^2}+\int_0^{u(s,i)}\exp\prn{-Cnu^2}\rd u}\\
        &\lesssim u(s,i)^{-1}. 
    \end{align*}
    Therefore,
    \begin{equation*}
        \sqrt{n}\EB[\|\btheta_m^{(n)}-\btheta_m\|_\infty]\lesssim|\gS|m\left[\min_{s\in\gS,i\in[m]}u(s,i)\right]^{-1}. 
    \end{equation*}
    The right hand side does not depend on $n$ and the conclusion follows. 
\end{proof}
\subsection{Proof of Lemma~\ref{lem:G_invertible}}\label{Appendix_lem_G_invertible}
\begin{proof}
    A direct calculation shows that $\nabla \bH(\btheta_m)=\bG_m$. For every $(s,i)$, We have
    \begin{align*}
        &(\bG_m)_{(s,i),(s,i)}-\sum_{(s^\prime,j)\neq(s,i)}\left|(\bG_m)_{(s^\prime,j),(s,i)}\right|\\
        =&(1-\gamma)\sum_{a\in\gA,s^\prime\in\gS}\pi(a\mid s)P(s^\prime\mid s,a)\brk{\frac{1}{m}\sum_{j=1}^mp_{s,a}(\theta_m(s,i)-\gamma\theta_m(s^\prime,j))}\\
        =&(1-\gamma)p_{(\gT^\pi\bm\eta_m)(s)}(\theta_m(s,i))>0, 
    \end{align*}
    where the last inequality follows from Lemma~\ref{lem:quantile_density_positive}. Therefore $\bG_m$ is diagonal dominant and hence invertible. 
\end{proof}
\subsection{Proof of Lemma~\ref{lemma:donsker}}\label{Appendix_lem_donsker}

\begin{proof}[Proof of Lemma~\ref{lemma:donsker}]
We prove a upper bound of the VC dimension of the function class $\gH_{s,i}$. We only need to bound the VC dimension of $\wtilde{\gH}_{s,i}=\{ h^\btheta_{s,i}(R,S) \colon \|\btheta-\btheta_m\|_\infty\leq(1-\gamma)^{-1}\}$. 

For any $(R,S)$ and $k$, denote $\ind\{R_{s,a}+\gamma\theta(S_{s,a},k)<\theta(s,i)\}=\ind\{\ell_{R,S,a,k}(\btheta)>0\}$, where $\ell_{R,S,a,k}(\btheta)=\theta(s,i)-\gamma\theta(S_{s,a},k)-R_{s,a}$ is an affine function defined on $\RB^{\gS\times m}$. Denote the maximum shattered set as $\{(R^{(t)},S^{(t)},u^{(t)})\}$, $t=1,\dots,N$, then the value of $\ind\{h^\btheta_{s,i}(R^{(t)},S^{(t)})>u^{(t)}\}$ is defined by the signs of $N|\gA|m$ affine functions
\begin{equation*}
    \ell_{R^{(t)},S^{(t)},a,k}(\btheta), \quad t=1,\dots,N,\ a\in\gA,\ k=1,\dots,m
\end{equation*}
However, $N|\gA|m$ hyperplanes separates $\RB^{\gS\times m}$ into at most
\begin{equation*}
    \sum_{j=0}^{\vert\gS\vert m} \binom{N|\gA|m}{j}\le \left(\frac{eN|\gA|m}{\vert\gS\vert m}\right)^{\vert\gS\vert m}
\end{equation*}
areas. Therefore, 
\begin{equation*}
    2^N \leq\left(\frac{eN|\gA|m}{\vert\gS\vert m}\right)^{\vert\gS\vert m}, 
\end{equation*}
which indicates that $N\lesssim \vert\gS\vert m\log(e|\gA| m)$. 

Since every function $\phi$ in $\gH_{s,i}$ satisfies $\|\phi\|_\infty\leq 1$, this class is $\bigotimes_{a\in\gA}Q_{s,a}$-Donsker according to~\citet{van2023weak}. 
\end{proof}

\subsection{Proof of Lemma~\ref{lem:influence_function}}\label{Appendix_lem_influence_function}
\begin{proof}
    Define the model family
    \begin{equation*}
        \gQ=\brc{(Q_{s,a})_{s,a}:Q_{s,a}\ll\lambda\otimes\#_\gS,\ \forall (s,a)}, 
    \end{equation*}
    where $\lambda$ is the Lebesgue measure on $[0,1]$ and $\#_\gS$ is the counting measure on $\gS$. We can regard the fixed point of the projected Bellman operator as a functional $\btheta_m\colon\gQ\to\RB^{\gS\times[m]}$. For every bounded efficient score $\bm g$ satisfying $\EB_Q[\bm g]={0}$, define $\rd Q^\epsilon_{s,a}=(1+\epsilon g_{s,a})\rd Q_{s,a}$. Then we differentiate $\bH^\epsilon(\btheta_m^\epsilon)=\bT$ with regard to $\epsilon$ and derive
    \begin{equation*}
        \bG_m\frac{\rd\btheta_m^\epsilon}{\rd\epsilon}\Bigg|_{\epsilon=0}+\frac{\rd \bH^\epsilon(\btheta_m)}{\rd\epsilon}\Bigg|_{\epsilon=0}=0. 
    \end{equation*}
    However, a direct calculation shows that
    \begin{equation*}
        \prn{\frac{\rd \bH^\epsilon(\btheta_m)}{\rd\epsilon}\Bigg|_{\epsilon=0}}_{s,i}=\left\<(\bm Y_m)_{(s,i),\cdot},\bm g\right\>, 
    \end{equation*}
    where the inner product is defined as
    \begin{equation*}
        \<\bm g_1,\bm g_2\>=\sum_{s,a}\EB_{Q_{s,a}}\brk{g_{1,s,a}g_{2,s,a}}. 
    \end{equation*}
    Therefore Lemma~\ref{lem:influence_function} holds. 
\end{proof}

\section{Omitted Proofs of Limiting Covariance Structure}\label{Appendix_omit_proof_lim_cov}
\subsection{Proof of Proposition~\ref{prop:K_invertible}}\label{Appendix_prop_K_invertible}
\begin{proof}
First we prove that $\gK$ and $\wtilde{\Sigma}$ are bounded operators in $\gL$. For $\gK$ and every $\bm\psi\in\gL$ such that $\|\bm\psi\|=1$, 
\begin{align*}
    \norm{\gK\bm\psi}^2\leq&\sum_{s\in\gS}\iint_{[0,1]^2}\left|\frac{\sum_{a,s^\prime}\pi(a\mid s)P(s^\prime\mid s,a)p_{s,a}\prn{F^{-1}_{\eta^\pi(s)}(\tau)-\gamma F^{-1}_{\eta^\pi(s^\prime)}(t)}}{p_{\eta^\pi(s)}\prn{F^{-1}_{\eta^\pi(s)}(\tau)}}\right|^2\rd \tau\rd t\\
    \leq&C_0\sum_{s\in\gS}\int_0^1\frac{\int_0^1\sum_{a,s^\prime}\pi(a\mid s)P(s^\prime\mid s,a)p_{s,a}\prn{F^{-1}_{\eta^\pi(s)}(\tau)-\gamma F^{-1}_{\eta^\pi(s^\prime)}(t)}\rd t}{p_{\eta^\pi(s)}\prn{F^{-1}_{\eta^\pi(s)}(\tau)}^2}\rd\tau\\
    =&\sum_{s\in\gS}\int_0^1\frac{C_0}{p_{\eta^\pi(s)}\prn{F^{-1}_{\eta^\pi(s)}(\tau)}}\rd\tau=\frac{C_0|\gS|}{1-\gamma}. 
\end{align*}
Therefore $\gK$ is bounded. For $\wtilde{\Sigma}$, denote
\begin{equation*}
    Y_s(r,s^\prime;\tau)=
    F_{\eta^{\pi}(s')}\left(\frac{F^{-1}_{\eta^\pi(s)}(\tau)-r}{\gamma}\right)
\end{equation*}
and define independent random variables $(R_{s,a},S_{s,a})\sim Q_{s,a}$. We notice that
\begin{align*}
    &|A_{s}(\tau,t)|^2=\left[\sum_{a\in\gA}\pi(a\mid s)^2\cov\left(Y_{s}(R_{s,a},S_{s,a};\tau),Y_{s}(R_{s,a},S_{s,a} ;t)\right)\right]^2\\
    \leq&\left[\sum_{a\in\gA}\var(\pi(a\mid s)Y_{s}(R_{s,a},S_{s,a} ;\tau))\right]\left[\sum_{a\in\gA}\var(\pi(a\mid s)Y_{s}(R_{s,a},S_{s,a} ;t))\right]\\
    =&\var\left[\sum_{a\in\gA}\pi(a\mid s)Y_{s}(R_{s,a},S_{s,a} ;\tau)\right]\cdot\var\left[\sum_{a\in\gA}\pi(a\mid s)Y_{s}(R_{s,a},S_{s,a} ;t)\right]\\
    \leq& \tau(1-\tau)t(1-t).
\end{align*}
where the last inequality follows from 
\[
\EB\sum_{a\in\gA}\pi(a\mid s)Y_{s}(R_{s,a},S_{s,a} ;t)=F_{(\gT^\pi\bm\eta^\pi)(s)}\prn{F^{-1}_{\eta^\pi(s)}(t)}=t
\]
and $0\leq\sum_{a\in\gA}\pi(a\mid s)Y_{s}(R_{s,a},S_{s,a} ;t)\leq1$. Therefore, for any $\bm\psi\in\gL$ with $\|\bm\psi\|=1$, 
\begin{align*}
    \norm{\wtilde{\Sigma}\bm\psi}^2\leq&\sum_{s\in\gS}\iint_{[0,1]^2}\left|\frac{A_s(\tau,t)}{p_{\eta^\pi(s)}\prn{F^{-1}_{\eta^\pi(s)}(\tau)}p_{\eta^\pi(s)}\prn{F^{-1}_{\eta^\pi(s)}(t)}}\right|^2\rd\tau\rd t\\
    \leq&\sum_{s\in\gS}\brk{\int_0^1\frac{t(1-t)}{p_{\eta^\pi(s)}\prn{F^{-1}_{\eta^\pi(s)}(t)}^2}\rd t}^2\\
    =&\sum_{s\in\gS}\brk{\int_0^{(1-\gamma)^{-1}}\frac{F_{\eta^\pi(s)}(x)(1-F_{\eta^\pi(s)}(x))}{p_{\eta^\pi(s)}(x)}\rd x}^2
\end{align*}
By Proposition~\ref{prop:return_density}, we only need to prove that for some $\delta_0>0$, 
\begin{equation*}
    \int_0^{\delta_0}\frac{F_{\eta^\pi(s)}(x)}{p_{\eta^\pi(s)}(x)}\rd x<+\infty. 
\end{equation*}
According to distributional Bellman equation, 
\begin{equation*}
    \frac{F_{\eta^\pi(s)}(x)}{p_{\eta^\pi(s)}(x)}=\frac{\sum_{a,s^\prime}\pi(a\mid s)P(s^\prime\mid s,a)F_{s,a}*p_{\eta^\pi(s^\prime)}(\cdot/\gamma)(x)}{\sum_{a,s^\prime}\pi(a\mid s)P(s^\prime\mid s,a)p_{s,a}*p_{\eta^\pi(s^\prime)}(\cdot/\gamma)(x)}. 
\end{equation*}
However, argue as in the proof of Lemma~\ref{lem:density_lower_bound}, we know that $F_{s,a}(x)\lesssim xp_{s,a}(x)$ for $x\in[0,\kappa]$. Therefore, 
\begin{align*}
    &\frac{F_{s,a}*p_{\eta^\pi(s^\prime)}(\cdot/\gamma)(x)}{p_{s,a}*p_{\eta^\pi(s^\prime)}(\cdot/\gamma)(x)}=\frac{\int_0^x F_{s,a}(x-y)p_{\eta^\pi(s^\prime)}(y/\gamma)\rd y}{\int_0^x p_{s,a}(x-y)p_{\eta^\pi(s^\prime)}(y/\gamma)\rd y}\\
    \leq&\frac{\int_0^x (x-y)p_{s,a}(x-y)p_{\eta^\pi(s^\prime)}(y/\gamma)\rd y}{\int_0^x p_{s,a}(x-y)p_{\eta^\pi(s^\prime)}(y/\gamma)\rd y}\leq\kappa
\end{align*}
for $x\in[0,\kappa]$, which implies that $\wtilde{\Sigma}$ is bounded. 

By the property of the integral operator we know that $\gK$ is compact. Therefore the Fredholm alternative~\citep{brezis2011functional} implies that $\gI-\gamma\gK^*$ is invertible if and only if it is injective. Therefore, it suffices to prove injectivity.

Define $\gL^\infty=\bigoplus_{s\in\gS}L^\infty[0,1]$ with norm $\norm{\bm\psi}_\infty=\sup_s\norm{\psi_s}_\infty$. We have
\begin{align*}
    &\norm{\gamma\gK\bm\psi}_\infty\\
    \leq&\norm{\bm\psi}_\infty\sup_{s\in\gS,\tau\in(0,1)}\frac{\gamma}{p_{\eta^\pi(s)}\left(F^{-1}_{\eta^\pi(s)}(\tau)\right)}\sum_{a\in\gA,s'\in\gS}\pi(a\mid s)P(s^\prime\mid s,a)\int_0^1p_{s,a}\left(F^{-1}_{\eta^\pi(s)}(\tau)-\gamma F^{-1}_{\eta^\pi(s')}(t)\right)\psi_{s'}(t)\mathrm{d}t\\
    =&\norm{\bm\psi}_\infty\sup_{s\in\gS,\tau\in(0,1)}\frac{\gamma p_{(\gT^\pi\bm\eta^\pi)(s)}\left(F^{-1}_{\eta^\pi(s)}(\tau)\right)}{p_{\eta^\pi(s)}\left(F^{-1}_{\eta^\pi(s)}(\tau)\right)}=\gamma\norm{\bm\psi}_\infty, 
\end{align*}
which imlies that $\gamma \gK$ is a $\gamma$-contraction on $\gL^\infty$. Therefore, for every $\bv\in\gL^\infty$, the Neumann series
\begin{equation*}
    \bz_\bv\coloneq \sum_{k=0}^{\infty}(\gamma\gK)^k\bv
\end{equation*}
is well-defined in $\gL_\infty$ and $\bz_\bv$ satisfies that $(\gI-\gamma\gK)\bz_\bv=\bv$. Now suppose that there exists $\bw\in\gL$ such that $(\gI-\gamma \gK^*)\bw=0$, then for every $\bv\in\gL^\infty\subset\gL$, we have 
\[
\langle \bw,\bv\rangle=\langle \bw,(\gI-\gamma \gK)\bz_\bv\rangle=\langle (\gI-\gamma \gK^*)\bw,\bz_\bv\rangle=0. 
\]
However, $\gL^\infty$ is dense in $\gL$, so $\bw=0$, which implies that $\gI-\gamma\gK^*$ is injective and hence invertible.  
\end{proof}

\subsection{Proof of Proposition~\ref{prop:properties_summarize}}\label{Appendix_prop_properties_summarize}
\begin{proof}
    For the first claim, the expression $\bu=\bm\psi_m+\gamma \gE_m\bm K_m^\top \gR_m\bu$ shows that $u_s(t)\equiv u(s,i)$ for $t\in[(i-1)/m,i/m)$ for every $s\in\gS$ and $i\in[m]$. Therefore, the vector $\bu_0\coloneq (u(s,i))_{s\in\gS,i\in[m]}\in\RB^{\gS\times[m]}$ satisfies
    \begin{equation*}
        (\bI-\gamma \bK_m^\top)\bu_0=\bm\varphi_{m,f,s}. 
    \end{equation*}
    The existence can be checked directly and the uniqueness is deduced from the invertibility of $\bI-\gamma \bK_m$. The second claim can also be checked directly. 

    For the third claim, we notice that
    \begin{align*}
        \norm{\bm\psi_m-\bm\psi}^2\lesssim&\frac{1}{m}\sum_{i=1}^m\left|f^\prime\prn{F^{-1}_{(\gT^\pi\bm\eta_m)(s_0)}(\tau_i)}-f^\prime\prn{F^{-1}_{\eta^\pi(s_0)}(\tau_i)}\right|^2\\
        &+\sum_{i=1}^m\int_{\frac{i-1}{m}}^{\frac{i}{m}}\left|f^\prime\prn{F^{-1}_{\eta^\pi(s_0)}(\tau_i)}-f^\prime\prn{F^{-1}_{\eta^\pi(s_0)}(\tau)}\right|^2\rd\tau\\
        \leq&\sup_{\substack{|y-z|\leq\Bar{W}_\infty(\gT^\pi\bm\eta_m,\bm\eta^\pi)\\y,z\in[0,(1-\gamma)^{-1}]}}\abs{f^\prime(y)-f^\prime(z)}+\sup_{\substack{|t_1-t_2|\leq(2m)^{-1}\\t_1,t_2\in[0,1]}}\abs{f^\prime\prn{F^{-1}_{\eta^\pi(s_0)}(t_1)}-f^\prime\prn{F^{-1}_{\eta^\pi(s_0)}(t_2)}}
    \end{align*}
    The first term tends to $0$ as $m\to\infty$ because of the continuity of $f^\prime$ om $[0,(1-\gamma)^{-1}]$ and Theorem~\ref{thm:approx_error}. The second term tends to $0$ because of the continuity of $f^\prime\circ F^{-1}_{\eta^\pi(s_0)}$ on $[0,1]$. 

    For the fourth claim, we notice that $\gE_m\bK_m\gR_m$ is an integral operator expressed as
    \begin{equation*}
        (\gE_m\bK_m\gR_m\bm\psi)_s(\tau)=\sum_{s^\prime\in\gS}\int_0^1 \brk{\sum_{i,j=1}^m(\bK_m)_{(s,i),(s^\prime,j)}\ind_{[\frac{i-1}{m},\frac{i}{m})}(\tau)\ind_{[\frac{j-1}{m},\frac{j}{m})}(t)}\psi_{s^\prime}(t)\rd t. 
    \end{equation*}
    Therefore, the adjoint operator $(\gE_m\bK_m\gR_m)^*$ is given by
    \begin{equation*}
        [(\gE_m\bK_m\gR_m)^*\bm\psi]_s(\tau)=\sum_{s^\prime\in\gS}\int_0^1 \brk{\sum_{i,j=1}^m(\bK_m)_{(s^\prime,j),(s,i)}\ind_{[\frac{i-1}{m},\frac{i}{m})}(\tau)\ind_{[\frac{j-1}{m},\frac{j}{m})}(t)}\psi_{s^\prime}(t)\rd t. 
    \end{equation*}
    Therefore $(\gE_m\bK_m\gR_m)^*=\gE_m\bK_m^\top \gR_m$. 
\end{proof}

\subsection{Proof of Lemma~\ref{lem:operator_convergence}}\label{Appendix_lem_operator_convergence}
\begin{proof}
Define
\begin{align*}
    T_m^{(s,s^\prime)}(\tau,t)&=\frac{1}{p_{(\gT^\pi\bm\eta_m)(s)}\left(F^{-1}_{(\gT^\pi\bm\eta_m)(s)}(\tau)\right)}\sum_{a\in\gA}\pi(a\mid s)P(s^\prime\mid s,a)p_{s,a}\left(F^{-1}_{(\gT^\pi\bm\eta_m)(s)}(\tau)-\gamma F^{-1}_{(\gT^\pi\bm\eta_m)(s^\prime)}(t)\right),\\
    B_m^{(s)}(\tau,t)&=\frac{A_{s,m}(\tau,t)}{p_{(\gT^\pi\bm\eta_m)(s)}\left(F^{-1}_{(\gT^\pi\bm\eta_m)(s)}(\tau)\right)p_{(\gT^\pi\bm\eta_m)(s)}\left(F^{-1}_{(\gT^\pi\bm\eta_m)(s)}(t)\right)},
\end{align*}
where
\begin{align*}
    A_{s,m}(\tau,t)&=\sum_{a\in\gA}\pi(a\mid s)^2\cov_{(R,S)\sim \gP_R(\cdot\mid s,a)\otimes P(\cdot\mid s,a)}\left(Y_{s,m}(R,S;\tau),Y_{s,m}(R,S;t)\right),\\
    Y_{s,m}(R,S;\tau)&=\frac{1}{m}\sum_{k=1}^m\ind\{R+\gamma\theta_m(S,j)<F^{-1}_{\gT^\pi\bm\eta_m(s)}(\tau)\},
\end{align*}
and
\begin{align*}
        T^{(s,s^\prime)}(\tau,t)&=\frac{1}{p_{\eta^\pi(s)}\left(F^{-1}_{\eta^\pi(s)}(\tau)\right)}\sum_{a\in\gA}\pi(a\mid s)P(s^\prime\mid s,a)p_{s,a}\left(F^{-1}_{\eta^\pi(s)}(\tau)-\gamma F^{-1}_{\eta^\pi(s^\prime)}(t)\right),\\
    B^{(s)}(\tau,t)&=\frac{A_{s}(\tau,t)}{p_{\eta^\pi(s)}\left(F^{-1}_{\eta^\pi(s)}(\tau)\right)p_{\eta^\pi(s)}\left(F^{-1}_{\eta^\pi(s)}(t)\right)}.
\end{align*}
Then we can check that
\begin{align*}
    (\gE_m\bK_m\gR_m\bm\psi)_s(\tau)&=\sum_{s^\prime\in\gS}\int_0^1 \sum_{i,j=1}^m\left(T_m^{(s,s^\prime)}(\tau_i,\tau_j)\ind_{[\frac{i-1}{m},\frac{i}{m})}(\tau)\ind_{[\frac{j-1}{m},\frac{j}{m})}(t)\right)\psi_{s^\prime}(t)\mathrm{d}t,\\
    (\gK\bm\psi)_s(\tau)&=\sum_{s^\prime\in\gS}\int_0^1 T^{(s,s^\prime)}(\tau,t)\psi_{s^\prime}(t)dt,\\
    (\wtilde{\bSigma}_m\bm\psi)_s(\tau)&=\int_0^1 \sum_{i,j=1}^m\left(B_m^{(s)}(\tau_i,\tau_j)\ind_{[\frac{i-1}{m},\frac{i}{m})}(\tau)\ind_{[\frac{j-1}{m},\frac{j}{m})}(t)\right)\psi_{s}(t)\mathrm{d}t,\\
    (\wtilde{\Sigma}\bm\psi)_s(\tau)&=\int_0^1 B^{(s)}(\tau,t)\psi_{s}(t)dt.
\end{align*}

Therefore for every $\bm\psi\in\gL$ with $\|\bm\psi\|=1$, we have
\begin{align*}
    \norm{(\gE_m\bK_m\gR_m-\gK)\bm\psi}^2\leq&\sum_{s,s^\prime\in\gS}\iint\nolimits_{[0,1]^2} \left|T^{(s,s^\prime)}(\tau,t)-\sum_{i,j=1}^m\left(T^{(s,s^\prime)}(\tau_i,\tau_j)\ind_{[\frac{i-1}{m},\frac{i}{m})}(\tau)\ind_{[\frac{j-1}{m},\frac{j}{m})}(t)\right)\right|^2\rd t\rd\tau\\
    &+\frac{1}{m^2}\sum_{s,s^\prime\in\gS}\sum_{i,j=1}^{m}|T_m^{(s,s^\prime)}(\tau_i,\tau_j)-T^{(s,s^\prime)}(\tau_i,\tau_j)|^2
\end{align*}
The first term vanishes as $m\rightarrow+\infty$ because of the properties of piecewise constant approximation. 

Now we bound the second term. For every $\epsilon>0$, we decompose the first term as
\[
\frac{1}{m^2}\left(\sum_{\tau_i\leq\epsilon}+\sum_{\tau_i\geq 1-\epsilon}+\sum_{\epsilon<\tau_i<1-\epsilon}\right)\sum_{s,s^\prime\in\gS,j\in[m]}|T_m^{(s,s^\prime)}(\tau_i,\tau_j)-T^{(s,s^\prime)}(\tau_i,\tau_j)|^2. 
\]
By Lemma~\ref{lem:pdf_uniform_convergence}, $\|T_m^{(s,s^\prime)}-T^{(s,s^\prime)}\|_\infty\to0$ on $[\epsilon,1-\epsilon]\times[0,1]$. Therefore the third summation vanishes as $m\to\infty$. 

Now we bound the first summation, and the second can be dealt with similarly. We have
\begin{align*}
    &\frac{1}{m^2}\sum_{\tau_i\leq\epsilon}\sum_{s,s^\prime\in\gS,j\in[m]}|T_m^{(s,s^\prime)}(\tau_i,\tau_j)|^2\\
    =&\sum_{s\in\gS}\frac{1}{m}\sum_{\tau_i\leq\epsilon}\frac{1}{p_{(\gT^\pi\bm\eta_m)(s)}\left(F^{-1}_{(\gT^\pi\bm\eta_m)(s)}(\tau_i)\right)^2}\sum_{s^\prime,j}\frac{1}{m}\left[\sum_{a\in\gA}\pi(a\mid s)P(s^\prime\mid s,a)p_{s,a}\left(F^{-1}_{\gT^\pi\bm\eta_m(s)}(\tau_i)-\gamma F^{-1}_{\gT^\pi\bm\eta_m(s^\prime)}(\tau_j)\right)\right]^2\\
    \leq& C_0\sum_{s\in\gS}\frac{1}{m}\sum_{\tau_i\leq\epsilon}\frac{1}{p_{(\gT^\pi\bm\eta_m)(s)}\left(F^{-1}_{(\gT^\pi\bm\eta_m)(s)}(\tau_i)\right)}
\end{align*}

By Lemma~\ref{lem:density_lower_bound}, we know that for sufficiently large $m$ and sufficiently small $\epsilon$, 
\begin{align*}
    \frac{1}{m}\sum_{\tau_i\leq\epsilon} \frac{1}{p_{(\gT^\pi\bm\eta_m)(s)}\left(F^{-1}_{(\gT^\pi\bm\eta_m)(s)}(\tau_i)\right)}&\lesssim\frac{1}{m}\sum_{\tau_i\leq\epsilon}\frac{F^{-1}_{(\gT^\pi\bm\eta_m)(s)}(\tau_i)}{\tau_i}\\
    &\leq \Bar{W}_\infty(\gT^\pi\bm\eta_m,\eta^\pi)\sum_{\tau_i\leq\epsilon}\frac{1}{i}+\frac{1}{m}\sum_{\tau_i\leq\epsilon}\frac{F^{-1}_{\eta^\pi(s)}(\tau_i)}{\tau_i}\\
    &\lesssim \log(\epsilon m)\cdot\sup_{s\in\gS}\omega_s\left(\frac{1}{2m}\right)+\frac{1}{m}\sum_{\tau_i\leq\epsilon}\frac{F^{-1}_{\eta^\pi(s)}(\tau_i)}{\tau_i}
\end{align*}
where the last inequality follows from Lemma~\ref{thm:approx_error}. We deduce from Assumption~\ref{Assumption_return_density} that for sufficiently small $\epsilon$, 
\[
\limsup_{m\rightarrow+\infty}\frac{1}{m}\sum_{\tau_i\leq\epsilon}\frac{1}{p_{(\gT^\pi\bm\eta_m)(s)}\left(F^{-1}_{(\gT^\pi\bm\eta_m)(s)}(\tau_i)\right)}\lesssim \int_0^\epsilon\frac{F^{-1}_{\eta^\pi(s)}(t)}{t}\mathrm{d}t. 
\]
Put all together and we derive that for every sufficiently small $\epsilon>0$, 
\begin{align*}
    &\limsup_{m\rightarrow+\infty}\sup_{\|\bm\psi\|=1}\norm{(\gE_m\bK_m\gR_m-\gK)\bm\psi}^2\\
    \lesssim& \sum_{s\in\gS}\left(\int_0^\epsilon\frac{F^{-1}_{\eta^\pi(s)}(t)}{t}\mathrm{d}t+\int_{1-\epsilon}^1\frac{(1-\gamma)^{-1}-F^{-1}_{\eta^\pi(s)}(t)}{1-t}\mathrm{d}t\right)+\sum_{s,s^\prime}\iint\nolimits_{S_\epsilon}|T^{(s,s^\prime)}|^2
\end{align*}
Where $S_\epsilon=[0,1]\times[0,1]\setminus[\epsilon,1-\epsilon]\times[0,1]$. Let $\epsilon\rightarrow0$ and we conclude that $\gE_m\bK_m\gR_m$ is bounded on $\gL$ and $\|\gE_m\bK_m\gR_m-\gK\|\to0$ as $m\to\infty$. 

By a similar argument, we know that in order to prove $\|\wtilde{\bSigma}_m-\wtilde{\Sigma}\|\rightarrow0$, we only need to prove that for every $s\in\gS$, 
\[
\limsup_{\epsilon\rightarrow0}\limsup_{m\rightarrow+\infty}\frac{1}{m^2}\left(\sum_{\tau_i\leq\epsilon}+\sum_{\tau_i\geq1-\epsilon}\right)\sum_{j=1}^m|B_m^{(s)}(\tau_i,\tau_j)|^2=0.
\]

First, denote $Q_{s,a}=\gP_R(\cdot\mid s,a)\otimes P(\cdot\mid s,a)$ and define independent random variables $(R_{s,a},S_{s,a})\sim Q_{s,a}$. We notice that
\begin{align*}
    &|A_{s,m}(\tau,t)|^2=\left[\sum_{a\in\gA}\pi(a\mid s)^2\cov_{Q_{s,a}}\left(Y_{s,m}(R_{s,a},S_{s,a} ;\tau),Y_{s,m}(R_{s,a},S_{s,a} ;t)\right)\right]^2\\
    \leq&\left[\sum_{a\in\gA}\var_{Q_{s,a}}(\pi(a\mid s)Y_{s,m}(R_{s,a},S_{s,a} ;\tau))\right]\left[\sum_{a\in\gA}\var_{Q_{s,a}}(\pi(a\mid s)Y_{s,m}(R_{s,a},S_{s,a} ;t))\right]\\
    =&\var\left[\sum_{a\in\gA}\pi(a\mid s)Y_{s,m}(R_{s,a},S_{s,a} ;\tau)\right]\cdot\var\left[\sum_{a\in\gA}\pi(a\mid s)Y_{s,m}(R_{s,a},S_{s,a} ;t)\right]\\
    \leq& \tau(1-\tau)t(1-t)
\end{align*}
where the last inequality follows from 
\[
\EB\sum_{a\in\gA}\pi(a\mid s)Y_{s,m}(R_{s,a},S_{s,a} ;t)=F_{(\gT^\pi\bm\eta_m)(s)}\prn{F^{-1}_{(\gT^\pi\bm\eta_m)(s)}(t)}=t
\]
and $0\leq\sum_{a\in\gA}\pi(a\mid s)Y_{s,m}(R_{s,a},S_{s,a} ;t)\leq1$. 
Therefore, we only need to prove that 
\[
\limsup_{\epsilon\rightarrow0}\limsup_{m\rightarrow+\infty}\frac{1}{m}\sum_{\tau_i\leq\epsilon}\frac{\tau_i}{p_{(\gT^\pi\bm\eta_m)(s)}\left(F^{-1}_{(\gT^\pi\bm\eta_m)(s)}(\tau_i)\right)^2}=0.
\]
Apply Lemma~\ref{lem:density_lower_bound} and we derive that
\begin{align*}
    &\frac{1}{m}\sum_{\tau_i\leq\epsilon}\frac{\tau_i}{p_{(\gT^\pi\bm\eta_m)(s)}\left(F^{-1}_{(\gT^\pi\bm\eta_m)(s)}(\tau_i)\right)^2}\\
    \lesssim&\frac{1}{m}\sum_{\tau_i\leq\epsilon}\frac{\left[F^{-1}_{(\gT^\pi\bm\eta_m)(s)}(\tau_i)\right]^2}{\tau_i}\leq\frac{1}{(1-\gamma)m}\sum_{\tau_i\leq\epsilon}\frac{F^{-1}_{(\gT^\pi\bm\eta_m)(s)}(\tau_i)}{\tau_i}. 
\end{align*}
Therefore, for every $s\in\gS$ and sufficiently small $\epsilon$, 
\[
\limsup_{m\rightarrow+\infty}\frac{1}{m^2}\left(\sum_{\tau_i\leq\epsilon}+\sum_{\tau_i\geq1-\epsilon}\right)\sum_{j=1}^m|B_m^{(s)}(\tau_i,\tau_j)|^2\lesssim\int_0^\epsilon\frac{F^{-1}_{\eta^\pi(s)}(t)}{t}\mathrm{d}t+\int_{1-\epsilon}^1\frac{(1-\gamma)^{-1}-F^{-1}_{\eta^\pi(s)}(t)}{1-t}\mathrm{d}t.
\]
Let $\epsilon\to0$ and the conclusion follows. 
\end{proof}

\subsection{Proof of Lemma~\ref{lem:var_expression}}\label{Appendix_lem_var_expression}
\begin{proof}
    First we prove the boundedness of $\Bar{\gK}$ and the invertibility of $\gI-\gamma\Bar{\gK}$. We have
    \begin{align*}
        \sup_{\|\bm\psi\|_{\bm w}=1}\norm{\Bar{\gK}\bm\psi}_{\bm w}^2&\leq\sum_{s\in\gS}\iint\nolimits_{[0,1]^2}\sum_{a,s^\prime}\frac{\abs{\pi(a\mid s)P(s^\prime\mid s,a)p_{s,a}\prn{F^{-1}_{\eta^\pi(s)}(\tau)-\gamma F^{-1}_{\eta^\pi(s^\prime)}(t)}}^2}{w_s(\tau)w_{s^\prime}(t)}\rd t\rd\tau\\
        &\leq\sum_{s\in\gS}\int_0^1\frac{C_0^2}{w_s(\tau)}\rd \tau=\frac{C_0^2|\gS|}{1-\gamma}. 
    \end{align*}
    To prove the invertibility, we define
    \begin{equation*}
        L_w^\infty[0,1]=\brc{\psi|\norm{\psi}_{w,\infty}=\sup_{t\in[0,1]}\frac{|\psi(t)|}{w(t)}<+\infty}, 
    \end{equation*}
    and
    \begin{equation*}
        \gL_{\bm w}^\infty=\bigoplus_{s\in\gS}L_{w_s}^\infty[0,1]
    \end{equation*}
    with $\|\bm\psi\|_{\bm w,\infty}=\sup_{s\in\gS}\|\psi_s\|_{w_s,\infty}$ for $\bm\psi\in\gL_{\bm w}^\infty$. A similar calculation as in Appendix~\ref{Appendix_prop_K_invertible} shows that $\gamma\Bar{\gK}$ is a $\gamma$-contraction on $\gL_{\bm w}^\infty$. Therefore, $\gI-\gamma\Bar{\gK}^*$ is injective on $\gL_{\bm w}$ and the Fredhlom alternative implies that it is invertible. 
    
    Now we prove that $(\gI-\gamma\Bar{\gK}^*)^{-1}\bm\psi=(\gI-\gamma \gK^*)^{-1}\bm\psi$ a.e., and we only need to prove that for every $\bm\phi\in\gL$,
    \begin{equation*}
        \<\bm\psi,\bm\phi\>=\<(\gI-\gamma \gK^*)\Bar{\bu}^*,\bm\phi\>, 
    \end{equation*}
    where we denote $\Bar{\bu}^*=(\gI-\gamma\Bar{\gK}^*)^{-1}\bm\psi$ and the inner product is taken in $\gL$. 

        We define ${\gW}$ by $({\gW}\bm\phi)_s=w_s\phi_s$, then
    \begin{align*}
        \<(\gI-\gamma \gK^*)\Bar{\bu}^*,\bm\phi\>&=\<\Bar{\bu}^*,(\gI-\gamma \gK)\bm\phi\>\\
        &=\<\Bar{\bu}^*,\gW(\gI-\gamma \gK)\bm\phi\>_{\bm w}\\
        &=\<\bm\psi,(\gI-\gamma\Bar{\gK})^{-1}\gW(\gI-\gamma \gK)\bm\phi\>_{\bm w}. 
    \end{align*}
    A direct calculation shows that $(\gI-\gamma\Bar{\gK})\gW\bm\phi=\gW(\gI-\gamma \gK)\bm\phi$, hence we have
    \begin{equation*}
        \<(\gI-\gamma \gK^*)\Bar{\bu}^*,\bm\phi\>=\<\bm\psi,\gW\bm\phi\>_{\bm w}=\<\bm\psi,\bm\phi\>. 
    \end{equation*}
   Moreover, we have
   \begin{equation*}
       \cov_{(R,S)}\prn{Z_s(\tau;R,S),Z_{s^\prime}(t;R,S)}=\ind\brc{s=s^\prime}A_s(\tau,t). 
   \end{equation*}
   Therefore,
   \begin{align*}
       \sigma_{f,s}^2=&\sum_{s\in\gS}\iint\nolimits_{[0,1]^2}\frac{\cov_{(R,S)}\prn{Z_s(\tau;R,S),Z_{s}(t;R,S)}}{w_s(\tau)w_s(t)}\Bar{u}^*_s(\tau)\Bar{u}^*_s(t)\rd\tau\rd t\\
       =&\EB_{(R,S)}\brk{\sum_{s\in\gS}\int_0^1\frac{Z_{s}(t;R,S)\Bar{u}^*_s(t)}{w_s(t)}\rd t}^2 \\
       =&\var_{(R,S)}\left(\left\langle \bZ,(\gI-\gamma\Bar{\gK}^*)^{-1}\bm\psi\right\rangle_{\bm w}\right)\\
       =&\var_{(R,S)}\left(\left\langle [(\gI-\gamma\Bar{\gK})^{-1}\bZ]_s,\psi_s\right\rangle_{w_s}\right). \qedhere
   \end{align*}
\end{proof}

\subsection{Proof of Lemma~\ref{lem:correspondence}}\label{Appendix_lem_correspondence}
\begin{proof}
    We prove the following general version: if $\bm\nu=(\nu_s)_{s\in\gS}$ where $\nu_s$ is a finite signed measure on $[0,(1-\gamma)^{-1}]$ with zero total mass for every $s\in\gS$, and $Z^{\bm\nu}_s(t)=\nu_s(-\infty,F^{-1}_{\eta^\pi(s)}(t)]$, then
    \begin{equation*}
        (\gamma\Bar{\gK}\bZ^{\bm\nu})_s(t)=(\gT^\pi\bm\nu)_s(-\infty,F^{-1}_{\eta^\pi(s)}(t)]. 
    \end{equation*}
    In fact,
    \begin{align*}
        (\gamma\Bar{\gK}\bZ^{\bm\nu})_s(t)=&\gamma\sum_{a,s^\prime}\pi(a\mid s)P(s^\prime\mid s,a)\int_0^1\frac{p_{s,a}\prn{F^{-1}_{\eta^\pi(s)}(t)-\gamma F^{-1}_{\eta^\pi(s^\prime)}(\tau)}\nu_{s^\prime}(-\infty,F^{-1}_{\eta^\pi(s^\prime)}(\tau)]}{w_{s^\prime}(\tau)}\rd\tau\\
        =&\gamma\sum_{a,s^\prime}\pi(a\mid s)P(s^\prime\mid s,a)\int_0^1p_{s,a}\prn{F^{-1}_{\eta^\pi(s)}(t)-\gamma x}\nu_{s^\prime}(-\infty,x]\rd x\\
        =&(\gT^\pi\bm\nu)_s(-\infty,F^{-1}_{\eta^\pi(s)}(t)]. 
    \end{align*}
    Therefore, denote $\bm\mu=(\gI-\gT^\pi)^{-1}\bm\nu$ and we have
    \begin{equation*}
        [(\gI-\gamma\Bar{\gK})\bZ^{\bm\mu}]_s(t;R_{s},S_{s})=[(\gI-\gT^\pi)\bm\mu]_s(-\infty,F^{-1}_{\eta^\pi(s)}(t)]=\nu_s(-\infty,F^{-1}_{\eta^\pi(s)}(t)]=Z_s(t;R_{s},S_{s})
    \end{equation*}
    and Lemma~\ref{lem:correspondence} holds. 
\end{proof}

\subsection{Proof of Lemma~\ref{lem:infinite_dim_efficiency}}\label{Appendix_lem_infinite_dim_efficiency}
\begin{proof}
    For any bounded efficient score $\bm g$, We have
    \begin{equation*}
        \frac{\<\bm\eta^\pi_\epsilon,\bm\psi\>-\<\bm\eta^\pi,\bm\psi\>}{\epsilon}=\left\<\frac{\gT^\pi\bm\eta^\pi_\epsilon-\gT^\pi\bm\eta^\pi}{\epsilon},\bm\psi\right\>+\left\<\frac{\gT^\pi_\epsilon\bm\eta^\pi-\gT^\pi\bm\eta^\pi}{\epsilon},\bm\psi\right\>+\left\<\frac{(\gT^\pi_\epsilon-\gT^\pi)(\bm\eta^\pi_\epsilon-\bm\eta^\pi)}{\epsilon},\bm\psi\right\>. 
    \end{equation*}
    Therefore, 
    \begin{equation*}
        \left\<\frac{(\gI-\gT^\pi)(\bm\eta^\pi_\epsilon-\bm\eta^\pi)}{\epsilon},\bm\psi\right\>=\left\<\sum_{s,a}\EB_{Q_{s,a}}[g_{s,a}\bm Y_{\cdot,(s,a)}],\bm\psi\right\>+\left\<\frac{(\gT^\pi_\epsilon-\gT^\pi)(\bm\eta^\pi_\epsilon-\bm\eta^\pi)}{\epsilon},\bm\psi\right\>, 
    \end{equation*}
    Now we prove that
    \begin{equation*}
        \lim_{\epsilon\to0}\left\<\frac{(\gT^\pi_\epsilon-\gT^\pi)(\bm\eta^\pi_\epsilon-\bm\eta^\pi)}{\epsilon},\bm\psi\right\>=0. 
    \end{equation*}
    We have
    \begin{equation*}
        \left\langle\frac{(\gT^\pi_\epsilon-\gT^\pi)(\bm\eta^\pi_\epsilon-\bm\eta^\pi)}{\epsilon},\bm\psi\right\rangle
        =\EB\brk{\sum_{s\in\gS,a\in\gA}\pi(a\mid s)[(b_{R_{s,a},\gamma})_\#(\eta^\pi_\epsilon(S_{s,a})-\eta^\pi(S_{s,a}))f]g_{s,a}(R_{s,a},S_{s,a})}
    \end{equation*}
    and for every $(r,s^\prime)$,
    \begin{align*}
        \abs{(b_{r,\gamma})_\#(\eta^\pi_\epsilon(s^\prime)-\eta^\pi(s^\prime))f}&\leq W_1((b_{r,\gamma})_\#\eta^\pi_\epsilon(s^\prime),(b_{r,\gamma})_\#\eta^\pi(s^\prime))\\
        &=\inf_{X\sim\eta^\pi_\epsilon(s^\prime),Y\sim\eta^\pi(s^\prime)}\EB\abs{(r+\gamma X)-(r+\gamma Y)}\\
        &=W_1(\eta^\pi_\epsilon(s^\prime),\eta^\pi(s^\prime)).
    \end{align*}
    However,
    \begin{equation*}
        \Bar{W}_1(\bm\eta^\pi_\epsilon,\bm\eta^\pi)\leq\Bar{W}_1(\gT^\pi_\epsilon\bm\eta^\pi_\epsilon,\gT^\pi_\epsilon\bm\eta^\pi)+\Bar{W}_1(\gT^\pi_\epsilon\bm\eta^\pi,\gT^\pi\bm\eta^\pi)\leq\gamma\Bar{W}_1(\bm\eta^\pi_\epsilon,\bm\eta^\pi)+\Bar{W}_1(\gT^\pi_\epsilon\bm\eta^\pi,\gT^\pi\bm\eta^\pi)
    \end{equation*}
    and
    \begin{equation*}
        W_1(\gT^\pi_\epsilon\bm\eta^\pi(s),\gT^\pi\bm\eta^\pi(s))=\sup_{f\text{ is }1-\text{Lip}}\abs{\gT^\pi_\epsilon\bm\eta^\pi(s)f-\gT^\pi\bm\eta^\pi(s)f}\lesssim\epsilon\|\bm g\|_\infty. 
    \end{equation*}
    Combine the two equations and we know that
    \begin{equation*}
        \abs{\left\langle\frac{(\gT^\pi_\epsilon-\gT^\pi)(\bm\eta^\pi_\epsilon-\bm\eta^\pi)}{\epsilon},\bm\psi\right\rangle}=O(\epsilon).
    \end{equation*}
    The proof is completed by the linearity of $\gI-\gT^\pi$. 
\end{proof}
\section{Omitted Proofs of Berry--Esseen Bound}\label{Appendix_omit_proof_BE}
In this section, $\lesssim$ hides positive constants which are independent of $m$ and $n$ and only depend on $\gM$ and $f$. 
First, we figure out that Theorem~\ref{thm:variance_convergence} implies that
\begin{equation*}
    0<\inf_m\sigma_{m,f,s}\leq\sup_m\sigma_{m,f,s}<+\infty. 
\end{equation*}
\subsection{Proof of Lemma~\ref{lem:BE_delta}}\label{Appendix_lem_BE_delta}
\begin{proof}
For any $\btheta\in\RB^{\gS\times m}$, we define diagonal matrix $\bD(\btheta)$ and matrix $\bW(\btheta)$ by
\begin{align*}
    \prn{\bD(\btheta)}_{(s,i),(s,i)}&=\sum_{a\in\gA}\pi(a\mid s)\sum_{\tilde{s}\in\gS,k\in[m]}\frac{P(\tilde{s}\mid s,a)}{m}p_{s,a}(\theta(s,i)-\gamma\theta(\tilde{s},k))\\
    \prn{\bW(\btheta)}_{(s,i),(s^\prime,j)}&=\sum_{a\in\gA}\pi(a\mid s)\frac{P(s^\prime\mid s,a)}{m}p_{s,a}(\theta(s,i)-\gamma\theta(s^\prime,j)),\\
    \bG(\btheta)&=\bD(\btheta)-\gamma \bW(\btheta). 
\end{align*}
Lemma~\ref{lem:operator_convergence} implies that
\begin{equation*}
    \sup_m(\bI-\bK_m)^{-\top}\bm\varphi_{m,f,s}<\infty. 
\end{equation*}
Therefore, 
\begin{align*}
    &\vert \bm\varphi_{m,f,s}^\top \bG_m^{-1}[\bH(\btheta_m^{(n)})-\bH(\btheta_m)-\bG_m(\btheta_m^{(n)}-\btheta_m)]|\\
    \lesssim&\int_0^1\norm{\bD_m^{-1}\brk{\bG\left(\btheta_m+t(\btheta_m^{(n)}-\btheta_m)\right)-\bG(\btheta_m)}(\btheta_m^{(n)}-\btheta_m)}_\infty\mathrm{d}t.  
\end{align*}
We have
\begin{align*}
    &\abs{\bD(\btheta_1)-\bD(\btheta_2)}_{(s,i),(s,i)}\\
    \leq&\sum_{a\in\gA}\pi(a\mid s)\sum_{\tilde{s},k}\frac{P(\tilde{s}\mid s,a)}{m}\vert p_{s,a}(\theta_1(s,i)-\gamma\theta_1(\tilde{s},k))-p_{s,a}(\theta_2(s,i)-\gamma\theta_2(\tilde{s},k))\vert\\
    \lesssim&\vert \theta_1(s,i)-\theta_2(s,i)\vert+\frac{1}{m}\norm{\btheta_1-\btheta_2}_1, 
\end{align*}
and
\begin{align*}
    \abs{\bW(\btheta_1)-\bW(\btheta_2)}_{(s,i),(s^\prime,j)}\lesssim\frac{1}{m}[|\theta_1(s,i)-\theta_2(s,i)|+|\theta_1(s^\prime,j)-\theta_2(s^\prime,j)|].
\end{align*}
Therefore, 
\begin{align*}
    \norm{\bD(\btheta_1)-\bD(\btheta_2)}_\infty&\lesssim\norm{\btheta_1-\btheta_2}_\infty+\frac{1}{m}\norm{\btheta_1-\btheta_2}_1,\\
    \norm{\bW(\btheta_1)-\bW(\btheta_2)}_\infty&\lesssim\norm{\btheta_1-\btheta_2}_\infty+\frac{1}{m}\norm{\btheta_1-\btheta_2}_1,
\end{align*}
where $\|\bm A\|_\infty$ of a matrix $\bm A$ is defined as $\|\bm A\|_\infty=\sup_{\|x\|_\infty=1}\|\bm Ax\|_\infty=\sup_i\sum_j|A_{ij}|$. 
Moreover, according to Lemma~\ref{lem:density_lower_bound}, we know that
\begin{equation*}
    \max_{s\in\gS,i\in[m]}(\bD_m^{-1})_{(s,i),(s,i)}\lesssim\max\brc{\max_{s\in\gS,\tau_i\leq\epsilon}\frac{\theta_m(s,i)}{\tau_i},\max_{s\in\gS,\tau_i\geq1-\epsilon}\frac{(1-\gamma)^{-1}-\theta_m(s,i)}{1-\tau_i}}\lesssim m. 
\end{equation*}
Therefore, 
\begin{equation*}
    \vert \bm\varphi_{m,f,s}^\top \bG_m^{-1}[\bH(\btheta_m^{(n)})-\bH(\btheta_m)-\bG_m(\btheta_m^{(n)}-\btheta_m)]|\lesssim m\norm{\btheta_m^{(n)}-\btheta_m}_\infty^2. 
\end{equation*}
Since $\|\bH_n(\btheta_m^{(n)})\|_\infty\leq n^{-1}$, we have
\begin{equation*}
    \abs{\frac{\sqrt{n}}{\sigma_{m,s,f}}\bm\varphi_{m,f,s}^\top \bG_m^{-1}[\bH_n(\btheta_m^{(n)})-\bH(\theta_m)]}\lesssim\norm{\bD^{-1}_m}_\infty\norm{\bH_n(\btheta_m^{(n)})-\bH(\theta_m)}_\infty\lesssim \frac{m}{\sqrt{n}}. 
\end{equation*}
Finally, by the Lipschitz continuity of $f^\prime$, there exists $t_{s,i}\in[0,1]$ such that
\begin{align*}
    &\abs{\frac{\sqrt{n}}{m}\sum_{i=1}^{m}\left[f(\theta_m^{(n)}(s,i))-f(\theta_m(s,i))-f^\prime(\theta_m(s,i))(\theta_m^{(n)}(s,i)-\theta_m(s,i))\right]}\\
    \lesssim&\frac{\sqrt{n}}{m}\sum_{i=1}^m\abs{f^\prime\prn{\theta_m(s,i)+t_{s,i}(\theta_m^{(n)}(s,i)-\theta_m(s,i))}-f^\prime\prn{\theta_m(s,i)}}\abs{\theta_m^{(n)}(s,i)-\theta_m(s,i)}\\
    \lesssim&\sqrt{n}\norm{\btheta_m^{(n)}-\btheta_m}_\infty^2. 
\end{align*}
Therefore,
\begin{align*}
    |D|\ind_{O_n}&\lesssim\sup_{\norm{\btheta-\btheta_m}_\infty\leq\delta_n}\norm{\GB_n(\btheta)-\GB_n(\btheta_m)}+\left(\frac{m}{\sqrt{n}}+m\sqrt{n}\norm{\btheta_m^{(n)}-\btheta_m}_\infty^{2}\right)\ind_{O_n}\\
    &=\Delta_1+\Delta_2. 
\end{align*}
This completes the proof. 
\end{proof}

\subsection{Proof of Lemma~\ref{lem:BE_term_12}}\label{Appendix_lem_BE_term_12}

\begin{proof}
    We prove Equation~\eqref{eq:delta_1_ub} first. First, for every $(s,i)\in\gS\times[m]$, define 
    \[
    \gF_{\delta}^{(s,i)}=\{h^{\btheta}_{s,i}(R,S)-h^{\btheta_m}_{s,i}(R,S):\norm{\btheta-\btheta_m}_\infty\leq\delta\}.  
    \]
    According to Section~\ref{Appendix_lem_donsker}, 
    \begin{equation*}
        \mathbf{VC}(\gF_{\delta}^{(s,i)})\lesssim m\log (em), 
    \end{equation*}
    where $\mathbf{VC}(\cdot)$ denotes the Vapnik--Chervonenkis dimension of a function class~\citep{vershynin_2018}.
    
    Notice that for fixed $(R,S)$, 
    \begin{equation}\label{eq:envelop}
        \abs{R_{s,a}-[\theta(S_{s,a},k)-\gamma\theta(s,i)]-R_{s,a}+[\theta_m(S_{s,a},k)-\gamma\theta_m(s,i)]}\leq(1+\gamma)\norm{\btheta-\btheta_m}_\infty, 
    \end{equation}
    therefore the following function is an envelop function of $\gF_\delta^{(s,i)}$: 
    \begin{equation*}
        F_\delta^{(s,i)}(R,S)=\sum_{a\in\gA}\frac{\pi(a\mid s)}{m}\sum_{k=1}^m\ind\{\left|R_{s,a}-[\theta_m(S_{s,a},k)-\gamma\theta_m(s,i)]\right|\leq (1+\gamma)\delta\}. 
    \end{equation*}
    According to~\citet{van2023weak}, for every $(s,i)$, 
    \[
    \norm{\sup_{\norm{\btheta-\btheta_m}_\infty\leq\delta_n}\vert \sqrt{n}(\bH_n(\btheta)-\bH_n(\btheta_m))_{(s,i)}\vert}_{\psi_1}\lesssim\frac{\log (en)}{\sqrt{n}}+\sqrt{\delta_nm\log (em)},  
    \]
    where the $\psi_1$ norm of a random variable $X$ is defined as $\|X\|_{\psi_1}=\inf\{t>0:\EB\exp(|X|/t)\leq2\}$. Therefore, by Lemma~\ref{Lemma_maximal_of_sub_exponential}, 
    \[
    \norm{\sup_{\norm{\btheta-\btheta_m}_\infty\leq\delta_n}\norm{\sqrt{n}(\bH_n(\btheta)-\bH_n(\btheta_m))}_\infty}_{\psi_1}\lesssim\log (em)\left[\frac{\log (en)}{\sqrt{n}}+\sqrt{\delta_nm\log (em)}\right],  
    \]
    and
    \begin{align*}
        \EB\Delta_1\lesssim&m\norm{(\bI-\gamma \bK_m)^{-\top}\bm\varphi_{m,f,s}}_1\EB\brk{\sup_{\norm{\btheta-\btheta_m}_\infty\leq\delta_n}\norm{\sqrt{n}(\bH_n(\btheta)-\bH_n(\btheta_m))}_\infty}\\
        \lesssim&m\log (em)\left[\frac{\log (en)}{\sqrt{n}}+\sqrt{\delta_nm\log (em)}\right].
    \end{align*}

    To prove Equation~\eqref{eq:third_moment}, we figure out that
    \begin{align*}
        \EB\abs{\xi_i}^3\leq&\prn{\EB\abs{\xi_i}^2}^{\frac{1}{2}}\prn{\EB\abs{\xi_i}^4}^{\frac{1}{2}}\\
        \lesssim&\norm{\bD_m^{-1}}_\infty^2\lesssim m^2, 
    \end{align*}
    where the second inequality follows from the boundedness of $h^{\btheta_m}_{s,i}$. Finally, Equation~\eqref{eq:delta_2_ub} follows from Theorem~\ref{thm:non_asymp_improved}. 
\end{proof}

\subsection{Proof of Lemma~\ref{lem:BE_term_3}}\label{Appendix_lem_BE_term_3}
\begin{proof}
    We prove Equation~\eqref{eq:delta_1i} first. We have
    \begin{align*}
        &\vert\Delta_1-\Delta_1^{(j)}\vert\\
        \lesssim& n^{-\frac{1}{2}}\sup_{\norm{\btheta-\btheta_m}_\infty\leq\delta_n}\abs{\bm\varphi_{m,f,s}^\top\bG_m^{-1}\brk{\bm h^\btheta(X_j)-\bm h^{\btheta_m}(X_j)-(\bm h^\btheta(\wtilde{X}_j)-\bm h^{\btheta_m}(\wtilde{X}_j))}}\\
        \lesssim& mn^{-\frac{1}{2}}\sup_{\norm{\btheta-\btheta_m}_\infty\leq\delta_n}\left[\norm{\bm h^\btheta(X_j)-\bm h^{\btheta_m}(X_j)}_\infty+\norm{\bm h^\btheta(\wtilde{X}_j)-\bm h^{\btheta_m}(\wtilde{X}_j)}_\infty\right]
    \end{align*}
    Since $\wtilde{X}$ is an iid copy of the samples, 
    \begin{equation*}
        \left(\EB\vert\Delta_1-\Delta_1^{(j)}\vert^2\right)^\frac{1}{2}\lesssim mn^{-\frac{1}{2}}\left[\EB\sup_{\norm{\btheta-\btheta_m}_\infty\leq\delta_n}\norm{\bm h^\btheta(X_j)-\bm h^{\btheta_m}(X_j)}_\infty^2\right]^\frac{1}{2}.
    \end{equation*}
    By Equation~\eqref{eq:envelop}, we have
    \begin{equation*}
        \left(\EB\vert\Delta_1-\Delta_1^{(j)}\vert^2\right)^\frac{1}{2}\lesssim mn^{-\frac{1}{2}}\delta_n^{\frac{1}{2}}. 
    \end{equation*}
    Now Equation~\eqref{eq:delta_1i} follows from Cauchy's inequality. 
    
    To prove Equation~\eqref{eq:delta_2i}, we can decompose $|\xi_j|\vert\Delta_2-\Delta_2^{(j)}\vert$ as follows: 
    \begin{align*}
        &|\xi_j|\vert\Delta_2-\Delta_2^{(j)}\vert\\
        \lesssim& m\sqrt{n}\biggl[
        \begin{aligned}[t]
        &|\xi_j|\norm{\btheta_m^{(n)}-\btheta_m}_\infty^2\ind\left\{\norm{\btheta_m^{(n)}-\btheta_m}_\infty\leq\delta_n,\norm{\btheta_m^{(n,j)}-\btheta_m}_\infty>\delta_n\right\}\\
        &+|\xi_j|\norm{\btheta_m^{(n,j)}-\btheta_m}_\infty^2\ind\left\{\norm{\btheta_m^{(n)}-\btheta_m}_\infty>\delta_n,\norm{\btheta_m^{(n,j)}-\btheta_m}_\infty\leq\delta_n\right\}\\
        &+|\xi_j|\left(\norm{\btheta_m^{(n)}-\btheta_m}_\infty+\norm{\btheta_m^{(n,j)}-\btheta_m}_\infty\right)\norm{\btheta_m^{(n)}-\btheta_m^{(n,j)}}_\infty\\
        &\quad\ind\left\{\norm{\btheta_m^{(n)}-\btheta_m}_\infty\leq\delta_n,\norm{\btheta_m^{(n,j)}-\btheta_m}_\infty\leq\delta_n\right\}\biggr].
        \end{aligned}
    \end{align*}
    
    We bound the three terms respectively. First, 
    \begin{align*}
        &\EB\brk{|\xi_j|\norm{\btheta_m^{(n)}-\btheta_m}_\infty^2\ind\left\{\norm{\btheta_m^{(n)}-\btheta_m}_\infty\leq\delta_n,\norm{\btheta_m^{(n,i)}-\btheta_m}_\infty>\delta_n\right\}}\\
        \leq&\delta_n^2\EB\brk{|\xi_j|\ind\left\{\norm{\btheta_m^{(n,j)}-\btheta_m}_\infty>\delta_n\right\}}\\
        \leq &m\delta_n^2\exp\left(-\frac{n\delta_n^2}{m}\right), 
    \end{align*}
    and the second term satisfies the same upper bound. 
    
    For the third term, denote $A=\{\|\btheta_m^{(n)}-\btheta_m\|_\infty\leq\delta_n,\|\btheta_m^{(n,j)}-\btheta_m\|_\infty\leq\delta_n\}$ and we have
    \begin{align*}
        &\EB\brk{|\xi_j|\left(\norm{\btheta_m^{(n)}-\btheta_m}_\infty+\norm{\btheta_m^{(n,j)}-\btheta_m}_\infty\right)\norm{\btheta_m^{(n)}-\btheta_m^{(n,j)}}_\infty\ind_A}\\
        \lesssim&\left(\EB|\xi_j|^2\right)^\frac{1}{2}\left(\EB\norm{\btheta_m^{(n)}-\btheta_m}_\infty^4\right)^\frac{1}{4}\left(\EB\brk{\norm{\btheta_m^{(n)}-\btheta_m^{(n,j)}}_\infty^4\ind_A}\right)^\frac{1}{4}\\
        \lesssim&\sqrt{\frac{m\log (em)}{n}}\left[\EB\brk{\norm{\btheta_m^{(n)}-\btheta_m^{(n,j)}}_\infty^4\ind_A}\right]^\frac{1}{4}
    \end{align*}
    By Taylor's expansion, on the event $A$, 
    \begin{align*}
        \bH(\btheta_m^{(n)})-\bH(\btheta_m)&=\bG_m(\btheta_m^{(n)}-\btheta_m)+\bR_n;\\
        \bH(\btheta_m^{(n,j)})-\bH(\btheta_m)&=\bG_m(\btheta_m^{(n,j)}-\btheta_m)+\bR_{n,j},
    \end{align*}
    where according to Section~\ref{Appendix_lem_BE_delta}, 
    \begin{align*}
        \norm{\bR_n}_\infty&\lesssim\norm{\btheta_m^{(n)}-\btheta_m}_\infty\leq\delta_n^2, \\
        \norm{\bR_{n,j}}_\infty&\lesssim\norm{\btheta_m^{(n,j)}-\btheta_m}_\infty\leq\delta_n^2.
    \end{align*}
    Moreover, 
    \begin{align*}
        &\bH(\btheta_m^{(n)})-\bH(\btheta_m)\\
        =&-[\bH_n(\btheta_m)-\bH(\btheta_m)]-[\bH_n(\btheta_m^{(n)})-\bH(\btheta_m^{(n)})]+[\bH_n(\btheta_m)-\bH(\btheta_m)]+\bm\varepsilon_n,\\
        &\bH(\btheta_m^{(n,j)})-\bH(\btheta_m)\\
        =&-[\bH_n^{(j)}(\btheta_m)-\bH(\btheta_m)]-[\bH_n^{(j)}(\btheta_m^{(n,j)})-\bH(\btheta_m^{(n,j)})]+[\bH_n^{(j)}(\btheta_m)-\bH(\btheta_m)]+\bm\varepsilon_{n,j}. 
    \end{align*}
    where $\|\bm\varepsilon_n\|_\infty,\|\bm\varepsilon_{n,j}\|_\infty\lesssim n^{-1}$. Subtract the two equations and we have
    \begin{align*}
        &\norm{\bH(\btheta_m^{(n)})-\bH(\btheta_m^{(n,j)})}_\infty\ind_A\\
        \lesssim&\frac{1}{n}+\frac{1}{\sqrt{n}}\Biggl[
        \begin{aligned}[t]
        &\sup_{\norm{\btheta-\btheta_m}_\infty\leq\delta_n}\norm{\sqrt{n}(\bH_n(\btheta)-\bH_n(\btheta_m))}_\infty\\
        &+\sup_{\norm{\btheta-\btheta_m}_\infty\leq\delta_n}\norm{\sqrt{n}(\bH_n^{(j)}(\btheta)-\bH_n^{(j)}(\btheta_m))}_\infty\Biggr]. 
        \end{aligned}
    \end{align*}
    Moreover, 
    \begin{equation*}
        \norm{\bG_m^{-1}}_\infty\leq\norm{(\bI-\gamma\bK_m)^{-1}}_\infty\norm{\bD_m^{-1}}_\infty\lesssim m.
    \end{equation*}
    Therefore, 
    \begin{align*}
        &\EB\brk{\abs{\xi_j}\left(\norm{\btheta_m^{(n)}-\btheta_m}_\infty+\norm{\btheta_m^{(n,j)}-\btheta_m}_\infty\right)\norm{\btheta_m^{(n)}-\btheta_m^{(n,j)}}_\infty\ind_A}\\
        \lesssim &\sqrt{\frac{m^3\log (em)}{n}}\left[\frac{1}{n}+\frac{1}{\sqrt{n}}\norm{\sup_{\norm{\btheta-\btheta_m}_\infty\leq\delta_n}\norm{\sqrt{n}(\bH_n(\btheta)-\bH_n(\btheta_m))}_\infty}_4+\delta_n^2\right]\\
        \lesssim &\sqrt{\frac{m^3\log (em)}{n}}\left[\frac{1}{n}+\frac{1}{\sqrt{n}}\left(\sqrt{\delta_nm\log (em)}+\frac{\log (en)}{\sqrt{n}}\right)+\delta_n^2\right]
    \end{align*}
    Put three parts together and we conclude that Equation~\eqref{eq:delta_2i} holds. 
\end{proof}
\section{Technical Lemmas}\label{Appendix_technical_lemmas}
\begin{lemma}[Hoeffding's Lemma]\label{Lemma_Hoeffding_lemma}
    Suppose $X\in [a,b]$ is a random variable with $\EB X=0$, then for any $\lambda\in\RB$, 
    \begin{equation*}
        \EB \exp{(\lambda X)}\leq \exp\prn{\frac{\lambda^2(b-a)^2}{8}}
    \end{equation*}
\end{lemma}
\begin{proof}
    See~\cite[Lemma 2.2]{boucheron2013concentration}.
\end{proof}

\begin{lemma}[Bernstein's Inequality]
\label{Lemma_Bernstein_Inequality}
Let $X_1,\cdots,X_n$ be independent, mean zero random variables, such that $|X_i|\leq K$ for all $i$. Then for every $t>0$, we have
\begin{equation*}
    \PB\brk{\left|\sum_{i=1}^nX_i\right|\geq t}\leq2\exp\prn{-\frac{t^2/2}{\sigma^2+Kt/3}}, 
\end{equation*}
where $\sigma^2=\sum_{i=1}^n\EB X_i^2$ is the variance of the sum.
\end{lemma}
\begin{proof}
See~\cite[Theorem 2.8.4]{vershynin_2018}.
\end{proof}

\begin{lemma}[Dvoretzky-Kiefer-Wolfowitz (DKW) Inequality]
\label{Lemma_DKW_Inequality}
     Let $X_1,\dots,X_n$ be real-valued \textup{i.i.d.} random variables with cumulative distribution function $F(\cdot)$. Let $F_n$ denote the associated empirical distribution function defined by $F_n(x)=\frac{1}{n}\sum_{i=1}^n\ind\brc{X_i\leq x}$.
     Then we have for any $\epsilon>0$,
     \begin{equation*}
         \PB\brk{\sup_{x\in\RB}\abs{F(x)-F_n(x)}>\epsilon}\leq 2e^{-2n\epsilon^2}.
     \end{equation*}
\end{lemma}
\begin{proof}
    See~\cite{massart1990tight}.
\end{proof}

\begin{lemma}\label{Lemma_expected_maximal_of_sub_gaussian}
Let $X_1,...,X_N$ be any $N\geq2$ random variables such that for any $\lambda\in\RB$,
\begin{equation*}
    \EB \exp{\prn{\lambda X_i}}\leq \exp{\prn{\sigma^2\lambda^2}}.
\end{equation*}
Then
\begin{equation*}
\EB\max_{i\in\brc{1,\dots,N}} \abs{X_i}\leq3\sigma\sqrt{\log N}.
\end{equation*}
\end{lemma}

\begin{proof}
    For any $\lambda\in\RB$, 
    \begin{equation*}
        \EB\exp(\lambda|X_i|)\leq\EB[\exp(\lambda X_i)+\exp(-\lambda X_i)]\leq2\exp(\sigma^2\lambda^2). 
    \end{equation*}
    Therefore, for every $\lambda>0$, 
    \begin{align*}
        \exp\prn{\lambda\EB\max_{i\in\brc{1,\dots,N}} \abs{X_i}}&\leq\EB\exp\prn{\lambda\max_{i\in\brc{1,\dots,N}} \abs{X_i}}\\
        &=\EB\max_{i\in\brc{1,\dots,N}} \exp(\lambda|X_i|)\\
        &\leq\sum_{i=1}^N\EB\exp(\lambda|X_i|)\leq 2N\exp(\sigma^2\lambda^2), 
    \end{align*}
    where the first inequality follows from Jensen's inequality. Take $\lambda=\sigma^{-1}\sqrt{\log(2N)}$ and we have
    \begin{equation*}
        \EB\max_{i\in\brc{1,\dots,N}} \abs{X_i}\leq 2\sigma\sqrt{\log(2N)}\leq3\sigma\sqrt{\log N}. 
    \end{equation*}
\end{proof}

\begin{lemma}\label{Lemma_maximal_of_sub_exponential}
Let $X_1,...,X_N$ be any $N\geq2$ random variables such that for any $i\in[N]$,
\begin{equation*}
    \EB \exp{\prn{\frac{|X_i|}{K}}}\leq 2.
\end{equation*}
Then there exists a universal constant $C$ such that
\begin{equation*}
\EB\exp\prn{\frac{\max_{i\in\brc{1,\dots,N}}|X_i|}{CK\log N}} \leq 2. 
\end{equation*}
Therefore,
\begin{equation*}
    \norm{\max_{i\in\brc{1,\dots,N}}|X_i|}_{\psi_1}\leq CK\log N. 
\end{equation*}
\end{lemma}

\begin{proof}
    By Jensen's inequality, for $C\geq 2/\log 2$,
    \begin{align*}
        \EB\exp\prn{\frac{\max_{i\in\brc{1,\dots,N}}|X_i|}{CK\log N}}&\leq\brk{\EB\exp\prn{\frac{\max_{i\in\brc{1,\dots,N}}|X_i|}{K}}}^{\frac{1}{C\log N}}\\
        &\leq\brk{\sum_{i=1}^N\EB\exp\prn{\frac{|X_i|}{K}}}^{\frac{1}{C\log N}}\\
        &\leq(2N)^{\frac{1}{C\log N}}\leq\exp\left(\frac{2}{C}\right)\leq2. 
    \end{align*}
    This completes the proof. 
\end{proof}

\begin{lemma}\label{Lemma_bounded_density_after_convolution}
Suppose $X_1,...,X_N$ are a sequence of independent random variables. $X_i$ has density $p_i(x)$ and $\sum_{i=1}^N X_i$ has density $p(x)$.
If $\sup_{x\in \RB}p_1(x)\leq C$, then $\sup_{x\in \RB}p(x)\leq C$.
\end{lemma}
\begin{proof}
    We have
    \begin{equation*}
        p(x)=\brk{\prn{\prn{\prn{p_1\ast p_2}\ast p_3}\cdots}\ast p_N}(x),
    \end{equation*}
    where $\brk{f\ast g}(x)\coloneq \int_{\RB} f(x-y)g(y) \mathrm{d}y$.
    Therefore, if $\sup_{x\in \RB}p_1(x)\leq C$, then
    \begin{equation*}
        (p_1\ast p_2)(x)=\int p_1(x-y)p_2(y)\rd y\leq \int Cp_2(y)\rd y=C
    \end{equation*}
    and the proof is completed by deduction.
\end{proof}

\begin{lemma}\label{Lemma_modulus_of_continuity_after_convolution}
Suppose $X_1,...,X_N$ are a sequence of independent random variables. $X_i$ has density $p_i(x)$ and $\sum_{i=1}^N X_i$ has density $p(x)$.
If $p_1$ has modulus of continuity $\omega_1(\varrho)$, namely for every $\varrho>0$, 
\begin{equation*}
    \sup_{|x-y|\leq\varrho}|p_1(x)-p_1(y)|\coloneq \omega_1(\varrho)<+\infty, 
\end{equation*}
then
\begin{equation*}
    \sup_{|x-y|\leq\varrho}|p(x)-p(y)|\coloneq \omega(\varrho)\leq\omega_1(\varrho). 
\end{equation*}
\end{lemma}
\begin{proof}
    We have
    \begin{equation*}
        p(x)=\brk{\prn{\prn{\prn{p_1\ast p_2}\ast p_3}\cdots}\ast p_N}(x),
    \end{equation*}
    where $\brk{f\ast g}(x)\coloneq \int_{\RB} f(x-y)g(y) dy$.
    Therefore, if $|x-y|\leq\varrho$, 
    \begin{equation*}
        |\prn{p_1\ast p_2}(x)-\prn{p_1\ast p_2}(y)|\leq\int |p_1(x-z)-p_1(y-z)|p_2(z)\rd z\leq\omega_1(\varrho)\int p_2(z)\rd z=\omega_1(\varrho). 
    \end{equation*}
    The proof is completed by deduction.
\end{proof}

\begin{lemma}\label{lem:polya_thm}
Let $\{F_n\}_{n\geq 1}$ and $F$ be distribution functions on $\RB$. Suppose that $F_n$ converges to $F$ weakly, and that $F$ is continuous on $\RB$.
Then as $n\to\infty$, 
\begin{equation*}
    \|F_n-F\|_\infty\coloneq\sup_{x\in\RB}|F_n(x)-F(x)|\to0.
\end{equation*}
\end{lemma}

\begin{proof}
Fix $\epsilon>0$.

Since $F$ is continuous and satisfies
\begin{equation*}
    \lim_{x\to-\infty}F(x)=0,\ \lim_{x\to\infty}F(x)=1,
\end{equation*}
there exist $a<b$ such that
\begin{equation*}
    F(x)\leq \epsilon,\ x\le a;\ F(x)\ge 1-\epsilon,\ x\geq b.
\end{equation*}

Since $F$ is continuous on the compact interval $[a,b]$, there exists $\delta>0$ such that for any $|x-y|\leq\delta$, $|F(x)-F(y)|\leq\epsilon$. 

Denote $M=\lfloor(b-a)/\delta\rfloor+1$ and $(M+1)$ points
\begin{equation*}
    x_i=a+\frac{i}{M}(b-a),\ i=0,\ldots,M, 
\end{equation*}
then $x_i-x_{i-1}\leq\delta$. 
Because each $x_i$ is a continuity point of $F$ and
$F_n(x_i)\to F(x_i)$, there exists $N$ such that for all $n\geq N$ and $i=0,\ldots,M$, $|F_n(x_i)-F(x_i)|<\epsilon$. 

For any $n\ge N$ and $x\in[a,b]$, choose $i$ such that $x_{i-1}\le x\le x_i$. Using the monotonicity of $F_n$, we have
\begin{equation*}
F_n(x)-F(x)\le \prn{F_n(x_i)-F(x_i)}+\prn{F(x_i)-F(x)}\leq2\epsilon.
\end{equation*}
Similarly,
\begin{equation*}
F_n(x)-F(x)\geq \prn{F_n(x_{i-1})-F(x_{i-1})}-\prn{F(x)-F(x_{i-1})}\geq-\varepsilon-\varepsilon=-2\varepsilon.
\end{equation*}
Thus for every $x\in[a,b]$, $|F_n(x)-F(x)|\leq2\epsilon$. 

It remains to control the tails. For $x\leq a$, monotonicity gives
\begin{equation*}
    |F_n(x)-F(x)|\leq F_n(a)+F(a)\leq\epsilon+2F(a)\leq3\epsilon. 
\end{equation*}

Similarly, for $x\geq b$,
\begin{equation*}
    |F_n(x)-F(x)|\leq(1-F_n(x))+(1-F(x))\leq\epsilon+2(1-F(b))\leq3\epsilon. 
\end{equation*}

Combining the three regions,
\[
\sup_{x\in\RB}|F_n(x)-F(x)|\leq 3\epsilon
\]
for all $n\geq N$. Therefore the conclusion holds. 
\end{proof}

\end{document}